\documentclass{article}

\usepackage[final,main]{neurips_2026}

\usepackage[utf8]{inputenc}
\usepackage[T1]{fontenc}
\usepackage{microtype}
\usepackage{nicefrac}
\usepackage{natbib}
\setcitestyle{numbers,square}
\usepackage{amsmath,amssymb,amsthm,bm,mathtools}
\usepackage{hyperref}
\usepackage{thm-restate}
\usepackage{url}
\usepackage{booktabs}
\usepackage{xcolor}
\usepackage{enumitem}
\usepackage{algorithm}
\usepackage{algorithmic}
\usepackage{graphicx}
\usepackage{array}
\usepackage{multirow}
\usepackage{longtable}
\usepackage{caption}
\usepackage{bm}
\usepackage{makecell}
\usepackage{multirow}
\usepackage{rotating}

\usepackage{subcaption}

\theoremstyle{plain}
\newtheorem{theorem}{Theorem}[section]
\newtheorem{proposition}[theorem]{Proposition}
\newtheorem{lemma}[theorem]{Lemma}
\newtheorem{corollary}[theorem]{Corollary}
\theoremstyle{definition}
\newtheorem{definition}[theorem]{Definition}

\theoremstyle{remark}
\newtheorem{remark}[theorem]{Remark}

\newcommand{\HEDGE}{\textsc{Hedge}}
\newcommand{\R}{\mathbb{R}}
\newcommand{\E}{\mathbb{E}}
\newcommand{\1}{\mathbf{1}}

\newcommand{\diag}{\mathrm{diag}}

\newcommand{\vecop}{\mathrm{vec}}
\newcommand{\matop}{\mathrm{mat}}
\newcommand{\offdiag}{\mathrm{offdiag}}

\newcommand{\AH}{\mathcal A_H}
\newcommand{\Law}{\operatorname{Law}}

\newcommand{\fitwidth}[1]{\resizebox{\ifdim\width>\textwidth\textwidth\else\width\fi}{!}{#1}}

\newcommand{\N}{\mathcal{N}}
\newcommand{\mainonly}[1]{#1}
\newcommand{\restatetag}[1]{}

\title{Hypergraph Generation via Structured Stochastic Diffusion}

\author{%
  Christopher Nemeth \\
  School of Mathematical Sciences \\
  Lancaster University \\
  \texttt{c.nemeth@lancaster.ac.uk}
}

\begin{document}
\maketitle

\begin{abstract}
 Hypergraphs model higher-order interactions, but realistic hypergraph generation remains difficult because incidence, hyperedge-size heterogeneity, and overlap structure are not faithfully captured by pairwise reductions. We propose \HEDGE, a generative model defined directly on relaxed incidence matrices via a structured stochastic diffusion. The forward process combines a hypergraph-specific two-sided heat operator with an Ornstein--Uhlenbeck component, preserving structure-aware noising near the data while yielding an explicit Gaussian terminal law. Conditional on an observed hypergraph, this forward process is linear-Gaussian, so conditional means, covariances, scores, and reverse-drift targets are available in closed form. We therefore learn a permutation-equivariant state-only reverse-drift field in incidence space by regressing onto exact conditional targets, and generate samples by simulating a learned reverse-time SDE from the Gaussian base law. We establish exactness in the ideal state-only setting together with finite-horizon stability guarantees, and empirically show improved hypergraph generation quality relative to strong baselines. 
\end{abstract}

\section{Introduction}
\label{sec:introduction}

Hypergraphs extend ordinary graphs by allowing a single relation to involve any number of nodes \citep{battiston2020networks}. This makes them a natural model for systems whose interactions are inherently group-based rather than pairwise, such as coauthorship networks \citep{newman2001structure}, group-based communication and social interaction \citep{benson2018simplicial}, item co-purchase data \citep{turnbull2024latent}, and biological interaction systems \citep{klamt2009hypergraphs}. In such settings, the object of interest is not merely whether two nodes are connected, but which subsets of nodes participate together in a shared interaction. This structure is represented explicitly through the \emph{node--hyperedge incidence matrix}; Figure~\ref{fig:hypergraph_and_incidence_matrix} gives a simple actor--movie example.

This additional expressivity makes hypergraph generation substantially more challenging. A graph generator models pairwise adjacency \citep{vignac2023digress,stephenson2026g3,yimingdefog,martinkus2022spectre}, whereas a hypergraph generator must capture node participation across hyperedges, the distribution of hyperedge sizes, and the overlap structure induced by shared node memberships. These features encode the higher-order organisation of the system and are often distorted or lost under pairwise reductions such as clique expansions \citep{agarwal2006higher,zhou2006learning}.

Our starting point is that hypergraph incidence matrices admit a natural two-sided smoothing mechanism: one operator acts across nodes with similar hyperedge participation patterns, and another acts across overlapping hyperedges. Together, these define a hypergraph-aware heat flow on relaxed incidence matrices, giving an inductive bias that is intrinsic to the observed hypergraph. However, pure heat flow contracts towards a hypergraph-dependent low-complexity endpoint rather than a universal non-degenerate base law. To overcome this, we introduce Hyper Edge Diffusion and GEneration (\HEDGE), a \emph{structured stochastic diffusion} in incidence space that follows the two-sided heat operator near the data and gradually transitions to an Ornstein--Uhlenbeck (OU) regime at later times. This \emph{heat--OU process} preserves structure-aware noising while yielding an explicit Gaussian base law. Conditional on the observed hypergraph, the forward process is linear-Gaussian, so we can learn an $S_n \times S_m$-equivariant reverse-drift model directly on relaxed incidence matrices from its conditional targets and generate samples by integrating a learned reverse-time SDE.

 \paragraph{Our contributions.}
(i) We introduce \HEDGE, a structured stochastic diffusion for hypergraph generation in incidence space that combines hypergraph-specific two-sided heat smoothing with an Ornstein--Uhlenbeck regime, yielding an explicit Gaussian base law. (ii) We formulate reverse-time learning as regression onto conditional reverse-drift targets, and show that the $L^2$-optimal state-only predictor is the posterior average of these conditional reverse drifts; it coincides with the reverse drift of the Markovian projection of the marginal forward process. (iii) We prove finite-horizon stability guarantees for reverse generation in Wasserstein distance under a one-sided Lipschitz assumption on the learned field, clarifying how generation error depends on initialisation, reverse-drift approximation and numerical discretisation; the bounds apply under the stated regularity and shared-noise assumptions. (iv) We identify the natural $S_n\times S_m$ symmetry class of the reverse process, and show that the $L^2$-optimal state-only target is equivariant when the data law is invariant under these permutations, motivating an $S_n\times S_m$-equivariant neural parameterisation. (v) We empirically show improved recovery of higher-order overlap and spectral structure over statistical baselines on fixed-size real-data benchmarks.

\begin{figure}[t]
    \centering
    \begin{minipage}[c]{0.31\linewidth}
        \centering
        \includegraphics[width=\linewidth]{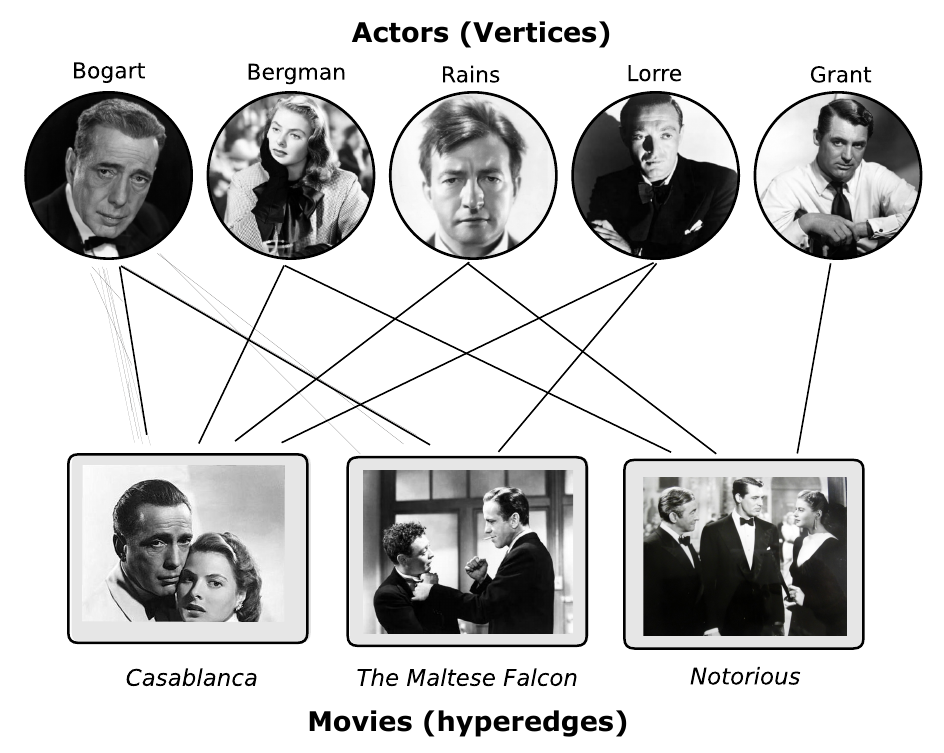}
    \end{minipage}
    \begin{minipage}[c]{0.60\linewidth}
        \centering
    \includegraphics[width=\linewidth]{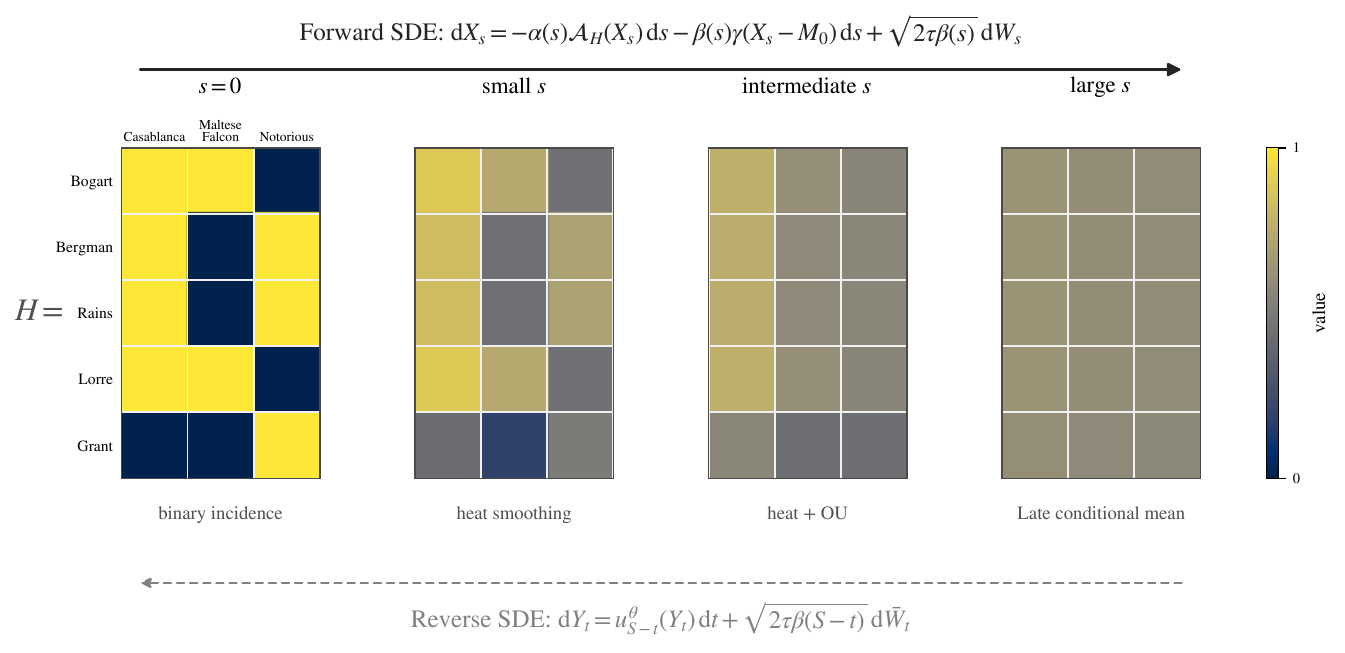}
    \end{minipage}
    \caption{Left: An actor--movie hypergraph. Right: Its incidence matrix $H$ and three snapshots of the conditional forward mean. Rows correspond to actors and columns to movies, with $H_{ij}=1$ if actor $i$ appears in movie $j$. Generation starts from the OU Gaussian base law~\eqref{eq:ou_invariant} and follows the learned reverse SDE.}
\label{fig:hypergraph_and_incidence_matrix}
\end{figure}

\section{Preliminaries}
\label{sec:preliminaries}

\paragraph{Hypergraphs and incidence matrices.}
A hypergraph is represented by the pair $\mathcal H=(V,E)$, where $V=\{1,\dots,n\}$ is a finite node set and $E=\{e_1,\dots,e_m\}$ is a multiset of non-empty hyperedges with $e_j \subseteq V$. Throughout, we consider unweighted undirected hypergraphs on a fixed vertex set $V$: hyperedges are non-empty, and vertices may be isolated. The incidence matrix of $\mathcal H$ is the binary matrix
\[
H \in \{0,1\}^{n\times m},
\qquad
H_{v,j}=1 \iff v\in e_j.
\]
Two basic marginals derived from $H$ are the node-degree and hyperedge-size vectors,
\begin{equation}
d_V := H\1_m \in \mathbb N_0^n,
\qquad
d_E := H^\top \1_n \in \mathbb N^m,
\label{eq:degree_size}
\end{equation}
where $d_V(v)$ is the number of hyperedges incident to node $v$, and $d_E(j)$ is the size of hyperedge $e_j$. We embed binary incidence matrices in a continuous state space by considering relaxed incidence states $X \in \R^{n\times m}$. This lets us define stochastic processes and differential operators directly in incidence space, which will allow for continuous-space reverse-time hypergraph generation.

\paragraph{Node-side and hyperedge-side operators.}
We view $X \in \R^{n\times m}$ as an element of a Euclidean space equipped with the Frobenius inner product $\langle X,Y\rangle_F := \mathrm{tr}(X^\top Y)$. The incidence matrix has two natural domains: rows indexed by nodes and columns indexed by hyperedges. Since our model evolves relaxed incidence matrices directly, we introduce one operator on each side. On the node side, we use the normalised hypergraph Laplacian of \citet{zhou2006learning} and \citet{chan2018spectral},
\begin{equation}
L_V(H)
=
I_n - D_V^{-1/2} H D_E^{-1} H^\top D_V^{-1/2},
\qquad
D_V := \diag(d_V),\;
D_E := \diag(d_E).
\label{eq:LV}
\end{equation}
This operator \emph{smooths node-level signals} across nodes with similar hyperedge patterns. Whenever inverse degree matrices appear, inversion is understood entrywise on strictly positive diagonal entries, with zero assigned to zero diagonal entries. Since hyperedges are non-empty, all diagonal entries of $D_E$ are strictly positive; an isolated node $v$ has $(D_V)_{vv}=0$ and, under this convention, contributes an $e_v e_v^\top$ block to $L_V$, preserving symmetric positive semidefiniteness. 

On the hyperedge side, we define the size-normalised overlap matrix
\begin{equation}
A_E(H)
:=
\offdiag\!\left(D_E^{-1/2} H^\top H D_E^{-1/2}\right),
\label{eq:AE}
\end{equation}
where $\offdiag(\cdot)$ sets diagonal entries to zero. Here, $H^\top H$ counts shared-node overlaps between hyperedges, while the normalisation divides each pairwise overlap by the geometric mean of the two hyperedge sizes. Writing $D_{\mathrm{ov}} := \diag(A_E(H)\1_m)$ for the diagonal of overlap-weighted degrees, we then define the \emph{normalised hyperedge-overlap Laplacian}
\begin{equation}
L_E(H)
:=
I_m - D_{\mathrm{ov}}^{-1/2} A_E(H) D_{\mathrm{ov}}^{-1/2}.
\label{eq:LE}
\end{equation}
This is a direct analogue of the symmetric normalised graph Laplacian \citep{chung1997spectral} applied to the weighted overlap graph on hyperedges induced by shared node memberships, closely related to the pairwise reductions used in the spectral analysis of hypergraphs \citep{agarwal2006higher,chan2018spectral} and to the overlap-based operators in \citet{macgregor2021finding}. Lemma~\ref{lem:appendix_LE_psd} establishes the positive semidefiniteness of $L_E(H)$. The operator $L_E(H)$ \emph{smooths signals across hyperedges} with similar node-membership patterns. Together, $L_V(H)$ and $L_E(H)$ provide the row-side and column-side geometry of the incidence representation.

\section{\HEDGE: Structured Stochastic Diffusion in Incidence Space}
\label{sec:method}

\HEDGE\ is a generative model defined directly on relaxed incidence matrices. The central idea is to combine a hypergraph-specific two-sided heat operator, which encodes incidence-space geometry, with a universal Ornstein--Uhlenbeck mechanism that yields a tractable Gaussian base law. This produces a forward diffusion that is hypergraph-adapted near the data, while still admitting well-defined unconditional generation through a learned state-only reverse-drift field (Figure~\ref{fig:hypergraph_and_incidence_matrix}). Full proofs for the results in this section are given in Appendix~\ref{app:theory_proofs}.

\subsection{Structured heat--OU forward diffusion}
\label{subsec:method_forward}

\paragraph{Relaxed incidence matrices.}
We work with real-valued matrices $X\in\R^{n\times m}$, \emph{relaxed incidence matrices}, as continuous surrogates for binary incidence matrices $H\in\{0,1\}^{n\times m}$. We use $s\in[0,S]$ to denote forward noising time, where $s=0$ corresponds to the observed hypergraph and larger $s$ corresponds to progressively more corrupted states; the forward state at time $s$ is denoted $X_s$. For reverse-time generation we use $t\in[0,S]$, where $t=0$ denotes the initial draw from the base law and $t=S$ the final generated sample; the reverse state at time $t$ is denoted $Y_t$. Along the reverse trajectory, the forward and reverse times are related by $s=S-t$. Given an observed hypergraph with incidence matrix $H$, we define the two-sided linear operator
\begin{equation}
\AH(X) := L_V(H)X + XL_E(H),
\qquad
X \in \R^{n\times m}.
\label{eq:AH}
\end{equation}
This operator smooths the relaxed incidence matrices simultaneously across nodes and across overlapping hyperedges. In particular, if $\frac{\mathrm{d}}{\mathrm{d}s}X_s = -\AH(X_s)$ with $X_0=H$, then $X_s = e^{-sL_V(H)}\,H\,e^{-sL_E(H)}$: the node-side and hyperedge-side heat kernels act by left and right multiplication, respectively, inducing a natural hypergraph-aware noising mechanism in incidence space.

\begin{proposition}[Dissipativity and spectral decoupling of the pure heat operator]
\label{prop:heat_operator_main}
The operator $\AH$ is self-adjoint and positive semidefinite with respect to the Frobenius inner product. Consequently, the pure heat flow $\frac{\mathrm{d}}{\mathrm{d}s}Z_s=-\AH(Z_s)$ is dissipative, $\frac{\mathrm{d}}{\mathrm{d}s}\|Z_s\|_F^2 = -2\langle Z_s,\AH(Z_s)\rangle_F\le 0$. Moreover, if $L_V(H)=U\Lambda U^\top$ and $L_E(H)=VMV^\top$ with eigenvalues $\{\lambda_i\}$ and $\{\mu_j\}$, then the transformed state $\widetilde Z_s:=U^\top Z_s V$ satisfies $\frac{\mathrm{d}}{\mathrm{d}s}\widetilde Z_s(i,j) = -(\lambda_i+\mu_j)\widetilde Z_s(i,j)$, so each joint node--hyperedge mode decays at rate $\lambda_i+\mu_j$.
\end{proposition}

Proposition~\ref{prop:heat_operator_main} shows that the heat flow is dissipative in the Frobenius geometry and suppresses high-frequency joint node--hyperedge modes first, so corruption is guided by the geometry induced by $L_V(H)$ and $L_E(H)$ rather than by isotropic noise alone; it requires only that both operators be symmetric positive semidefinite (Appendix~\ref{app-sec:heat-operator}). Pure heat flow, however, contracts toward a hypergraph-dependent low-complexity limit rather than a universal non-degenerate base law, so we embed the operator \eqref{eq:AH} in a scheduled stochastic diffusion that gradually transitions to an OU regime.

\paragraph{Forward noising schedule.}
Let $S>0$ be a terminal time horizon, let $M_0 \in \R^{n\times m}$ be a fixed mean matrix independent of the test-time hypergraph, and let $\gamma>0$ and $\tau>0$ control the late-time OU dynamics. In practice, $M_0$ specifies the centre of the Gaussian base law used for reverse-time generation (we use the constant matrix $M_0=\bar\rho\,\1_n\1_m^\top$ at the mean incidence density $\bar\rho$ of the training set); crucially, it is fixed globally and does not depend on the unknown generation-time hypergraph. We introduce continuous scheduling functions
\begin{equation}
\alpha,\beta : [0,S]\to[0,1],
\qquad
\alpha(s)+\beta(s)=1,
\qquad
\beta(0)=0,\ \beta(S)=1,
\label{eq:schedule_partition}
\end{equation}
with $\alpha$ nonincreasing and $\beta$ nondecreasing over time, and with $\beta$ strictly positive on $(0,S]$. The forward process on relaxed incidence matrices is then defined by the matrix-valued SDE
\begin{equation}
\mathrm{d}X_s
=
-\alpha(s)\AH(X_s)\,\mathrm{d}s
-\beta(s)\gamma\bigl(X_s-M_0\bigr)\,\mathrm{d}s
+
\sqrt{2\tau\beta(s)}\,\mathrm{d}W_s,
\qquad
X_0=H,
\label{eq:forward_sde}
\end{equation}
where $W_s$ is matrix Brownian motion with independent standard Brownian entries. At early times, $\alpha(s)\approx 1$ and $\beta(s)\approx 0$, so the dynamics are dominated by the hypergraph-specific two-sided heat operator. As $s \to S$, the heat contribution is gradually suppressed and the OU component becomes dominant. The coefficients approach those of the matrix-valued OU process $\mathrm{d}X_s = -\gamma(X_s-M_0)\mathrm{d}s + \sqrt{2\tau}\mathrm{d}W_s$. Writing $Z_s:=X_s-M_0$, over an interval of length $t$ this process gives $Z_{s+t}=e^{-\gamma t}Z_s+\sqrt{2\tau}\int_s^{s+t}e^{-\gamma(s+t-u)}\,\mathrm{d}W_u$; the stochastic integral, including its prefactor $\sqrt{2\tau}$, is Gaussian with vectorised covariance $\tfrac{\tau}{\gamma}(1-e^{-2\gamma t})I_{nm}$, so the influence of the initial state decays as $e^{-\gamma t}$ and the invariant law is
\begin{equation}
\pi := \N\bigl(M_0,\,\tfrac{\tau}{\gamma}I_{nm}\bigr),
\label{eq:ou_invariant}
\end{equation}
where the covariance is in vectorised coordinates. Conditional on hypergraph $H$, the finite-horizon law $p_{S\mid H}=\N(m_S(H),C_S(H))$ is exactly Gaussian (Proposition~\ref{prop:conditional_gaussian_main}). Marginalising over the data distribution, $p_S=\int p_{S\mid H}\,p_{\mathrm{data}}(\mathrm{d}H)$ is a Gaussian \emph{mixture}. The invariant law $\pi$ \eqref{eq:ou_invariant} is neither of these at finite $S$; it is the hypergraph-independent base law from which Algorithm~\ref{alg:sampling} initialises generation. Exactness (Proposition~\ref{prop:appendix_ideal_exactness}) assumes initialisation from $p_S$; initialising instead from $\pi$ contributes an initialisation error $W_2(\pi,p_S)$ to the bound of Theorem~\ref{thm:stability_reverse_main}. Proposition~\ref{prop:terminal_bounds} gives computable upper bounds for the empirical mixture $\hat p_S$, reported with the measured downstream effect in Appendix~\ref{app:terminal_law}.

\subsection{Conditional Gaussian law and the forward process}
\label{subsec:method_conditional_gaussian}

A key advantage of this forward process is that, conditional on the observed hypergraph, it remains a linear-Gaussian diffusion. This yields a Gaussian perturbation law at every time, whose moments solve linear ordinary differential equations, and hence allows us to train against conditional reverse-drift targets computed from these moments rather than against Monte Carlo approximations. Vectorising $x_s := \vecop(X_s)\in\R^{nm}$ and using $\vecop(AXB)=(B^\top\otimes A)\vecop(X)$, the SDE \eqref{eq:forward_sde} becomes
\begin{equation}
\mathrm d x_s
=
-B_s(H)\,x_s\,\mathrm ds
+\beta(s)\gamma\,\vecop(M_0)\,\mathrm ds
+\sqrt{2\tau\beta(s)}\,\mathrm d w_s,
\label{eq:forward_vec_main}
\end{equation}
with $B_s(H):=\alpha(s)\bigl(I_m\otimes L_V(H)+L_E(H)\otimes I_n\bigr)+\beta(s)\gamma I_{nm},$ and where we have used symmetry of $L_E$ to drop the transpose. This is a linear time-inhomogeneous SDE with additive Gaussian noise, so the conditional law of $x_s$ is Gaussian for every $s\in[0,S]$ \citep{oksendal2003stochastic,sarkka2019applied}.

\begin{proposition}[Conditional Gaussian law of the forward process]
\label{prop:conditional_gaussian_main}
For each $s\in[0,S]$, the conditional law of the forward state satisfies $\vecop(X_s)\mid H \sim \mathcal N(m_s(H),\,C_s(H))$, where
\begin{align}
\frac{\mathrm d}{\mathrm ds}m_s(H)
&=
-B_s(H)\,m_s(H)+\beta(s)\gamma\,\vecop(M_0),
&\quad
m_0(H)&=\vecop(H),
\label{eq:mean_ode_main}
\\
\frac{\mathrm d}{\mathrm ds}C_s(H)
&=
-B_s(H)C_s(H)-C_s(H)B_s(H)^\top+2\tau\beta(s)\,I_{nm},
&\quad
C_0(H)&=0.
\label{eq:cov_ode_main}
\end{align}
Under $\beta>0$ on $(0,S]$, $C_s(H)$ is positive definite (hence invertible) for all $s\in(0,S]$.
\end{proposition}

Given a training hypergraph $H$ and forward time $s$, one may solve \eqref{eq:mean_ode_main}--\eqref{eq:cov_ode_main} and sample $X_s$ from its conditional Gaussian law (Appendix~\ref{app-sec:conditional-gaussian}); for any $s\in(0,S]$, $C_s(H)$ is invertible by Proposition~\ref{prop:conditional_gaussian_main}, and the conditional score is
\begin{equation}
r^\star_{s\mid H}(X)
=
-\matop\bigl(C_s(H)^{-1}(\vecop(X)-m_s(H))\bigr),
\label{eq:conditional_score_main}
\end{equation}
which enters the conditional reverse drift in \eqref{eq:conditional_reverse_drift_final}. The score requires $H$, so we use the conditional moments to compute training targets for the state-only field. 

\paragraph{Efficient computation.}
A naive evaluation of \eqref{eq:mean_ode_main}--\eqref{eq:cov_ode_main} requires working with $nm\times nm$ matrices. In our experiments, we integrate the full conditional mean equation numerically in incidence coordinates (fixed-step RK4) using the dense operator $I_m\otimes L_V(H)+L_E(H)\otimes I_n$, replace the full covariance ODE with a diagonal approximation in incidence coordinates, and clip the resulting scores (Appendix~\ref{app:implementation}); this costs $O((nm)^2)$ time and memory per $(H,s)$ and is the implementation whose cost is reported in Appendix~\ref{app:compute}. The separable structure permits a cheaper alternative that retains the full conditional covariance: Proposition~\ref{prop:heat_operator_main} shows that $B_s(H)$ is diagonalisable in the Kronecker basis $V\otimes U$ (column-major $\vecop$) formed from the eigendecompositions of $L_V(H)$ and $L_E(H)$, so after a one-time $O(n^3+m^3)$ eigendecomposition per training hypergraph the mean and covariance ODEs decouple into $nm$ scalar ODEs that cost $O(nm)$ per step, plus $O(n^2m+nm^2)$ for the two basis transforms. Its implementation, validation against the dense solver, and cost are given in Appendices~\ref{app:implementation} and~\ref{app:compute}. 

\subsection{Learning the $L^2$-optimal state-only reverse drift}
\label{subsec:method_reverse}

The forward process \eqref{eq:forward_sde} is defined conditionally on a training hypergraph $H$, so for fixed $H$, the reverse-time dynamics are available in closed form. However, these dynamics depend explicitly on $H$, which is not available at generation time. The central modelling problem is therefore to construct a reverse-time drift field that can be evaluated without access to the original hypergraph. 

\paragraph{Conditional reverse drift.}
Let $p_{s\mid H}(X)$ denote the conditional density of the forward
diffusion at time $s$. Writing
\begin{equation}
b_{s\mid H}(X) := -\alpha(s)\AH(X) - \beta(s)\gamma\bigl(X-M_0\bigr)
\label{eq:conditional_forward_drift}
\end{equation}
for the conditional forward drift in \eqref{eq:forward_sde}, the forward SDE then takes the form
\[
\mathrm d X_s = b_{s\mid H}(X_s)\,\mathrm ds + \sqrt{2\tau\beta(s)}\,\mathrm d W_s,
\qquad X_0=H,
\]
with state-independent diffusion coefficient. Under standard regularity
conditions on $\{p_{s\mid H}\}_{s\in(0,S]}$, the classical time-reversal
formula for diffusions with additive noise
\citep{anderson1982reverse,haussmann1986time} then gives the conditional
reverse-time drift
\begin{equation}
u^\star_{s\mid H}(X)
= -b_{s\mid H}(X) + 2\tau\beta(s)\,r^\star_{s\mid H}(X)
= \alpha(s)\AH(X)+\beta(s)\gamma(X-M_0)+2\tau\beta(s)\,r^\star_{s\mid H}(X),
\label{eq:conditional_reverse_drift_final}
\end{equation}
where $r^\star_{s\mid H}(X)=\nabla_X \log p_{s\mid H}(X)$ is the conditional
score. Since the forward process is conditionally linear-Gaussian,
$r^\star_{s\mid H}$ is available from~\eqref{eq:conditional_score_main} given the
conditional moments, so $u^\star_{s\mid H}$ is determined by the conditional moments \eqref{eq:mean_ode_main}--\eqref{eq:cov_ode_main}. 

\paragraph{Reverse-drift regression objective.}
We learn a time-dependent neural field $u^\theta_s:\R^{n\times m}\to\R^{n\times m}$ by regressing directly onto the conditional reverse-drift targets $u^\star_{s\mid H}(X_s)$. For the exact targets and a time law $\rho$ on $(0,S]$, the population objective is
\begin{equation}
\mathcal L(\theta)
=
\E_{H\sim p_{\mathrm{data}},\, s\sim \rho,\,
X_s\sim p_{s\mid H}}
\Bigl[
\|u^\theta_s(X_s)-u^\star_{s\mid H}(X_s)\|_F^2
\Bigr].
\label{eq:reverse_drift_loss}
\end{equation}

\begin{restatable}[$L^2$-optimal state-only (Markovian-projection) reverse-drift target]{proposition}{bayesreversedrift}
\mainonly{\label{prop:bayes_reverse_drift_main}}
For each fixed $s\in(0,S]$, the unrestricted $L^2$-optimal predictor is $u_s^{\mathrm{L^2}}(X)=\mathbb E[u^\star_{s\mid H}(X)\mid X_s=X]$. Under the regularity and well-posedness conditions for Markovian projection and time reversal, this is the reverse drift of the diffusion with forward drift $\bar b_s(X)=\mathbb E[b_{s\mid H}(X)\mid X_s=X]$. This diffusion shares the forward process's marginals, without necessarily sharing its path law.\mainonly{ The proof is given in Appendix~\ref{app-sec:reverse-time-drift}.}
\end{restatable}

Proposition~\ref{prop:bayes_reverse_drift_main} identifies the $L^2$-optimal state-only approximation to the exact hypergraph-conditioned reverse drift. For the empirical training law $\hat p_{\mathrm{data}}=N^{-1}\sum_i\delta_{H^{(i)}}$, the posterior-average reverse drift in Proposition~\ref{prop:bayes_reverse_drift_main} decomposes as
\[
u_s^{L^2}(X)
= \alpha(s)\sum_{i=1}^N \omega_s(i\mid X)\,\mathcal A_{H^{(i)}}(X)
+ \beta(s)\gamma(X-M_0) + 2\tau\beta(s)\,\nabla_X\log\hat p_s(X), 
\]
where $\hat p_s=\frac1N\sum_i p_{s\mid H^{(i)}}$ is the marginal density over the $N$ training hypergraphs $H^{(i)}$ and $\omega_s(i\mid X)=p_{s\mid H^{(i)}}(X)/(N\hat p_s(X))$ the posterior over them. All three terms are functions of the current state and time. Direct evaluation of the posterior structural average and mixture score requires summing over the training hypergraphs. Since $X_s$ alone need not be Markov, $u_s^{L^2}$ is the reverse drift of the Markovian projection, which shares the one-time marginals of the latent mixture (Remark~\ref{rem:markov_projection}).

\paragraph{Learned reverse-time SDE.}
At generation time, we initialise from the Gaussian base law \eqref{eq:ou_invariant} and simulate
\begin{equation}
\mathrm d Y_t
=
u^\theta_{S-t}(Y_t)\,\mathrm dt
+\sqrt{2\tau\beta(S-t)}\,\mathrm d\bar W_t,
\qquad
t\in[0,S].
\label{eq:learned_reverse_sde_final}
\end{equation}
Under the assumptions of Proposition~\ref{prop:appendix_ideal_exactness}, if the learned reverse field matches the $L^2$-optimal
state-only target exactly and generation is initialised from the finite-horizon
marginal $p_S$, then exact integration of the reverse SDE recovers the data law
before decoding. The decoder can add incidences even to an exact binary sample with isolated vertices (Section~\ref{subsec:exp_protocol}). Algorithm~\ref{alg:sampling}
initialises from $\pi$, so the bounds below account for initialisation mismatch, reverse-drift approximation and discretisation.

\subsection{Stability and generation error}
\label{subsec:method_stability}

\begin{restatable}[Finite-horizon stability in $W_2$ via one-sided Lipschitz]{theorem}{stabilitythm}
\mainonly{\label{thm:stability_reverse_main}}
Let $Y_t^\star$ and $Y_t^\theta$ solve the ideal reverse SDE, with drift $u^\star_s:=u^{L^2}_s$, and the
learned reverse SDE, with the same diffusion coefficient $\sqrt{2\tau\beta(S-t)}$, initial laws $\nu_0^\star$ and
$\nu_0^\theta$ (in practice $p_S$ and $\pi$), synchronously coupled by the same Brownian motion. Assume both SDEs are well-posed with square-integrable solutions. Suppose $u^\theta$ admits an integrable one-sided Lipschitz
constant $\kappa:[0,S]\to\R$, i.e.\ $\langle X-Y,\,u^\theta_s(X)-u^\theta_s(Y)\rangle_F \le \kappa_s\|X-Y\|_F^2$ for all $X,Y,s$.

With $\Lambda(t):=\int_0^t (2\kappa_{S-r}+1)\,\mathrm dr$, define
\[
\mathcal E_{\mathrm{rev}}(t)^2
:=
W_2(\nu_0^\theta,\nu_0^\star)^2
+
\int_0^t e^{-\Lambda(r)}\,
\mathbb E\|u^\theta_{S-r}(Y_r^\star)-u^\star_{S-r}(Y_r^\star)\|_F^2\,\mathrm dr.
\]
Assume the weighted drift-error integral is finite. Then, for every $t\in[0,S]$,
\begin{equation}
W_2\bigl(\Law(Y_t^\theta),\Law(Y_t^\star)\bigr)
\le
e^{\Lambda(t)/2}\,\mathcal E_{\mathrm{rev}}(t).
\restatetag{eq:stability_one_sided_w2_main}\mainonly{\label{eq:stability_one_sided_w2_main}}
\end{equation}
Taking $\kappa_s\equiv L$ recovers the Gr\"onwall prefactor $e^{(2L+1)t/2}$.
If $\kappa_s\le \kappa_\star<-1/2$ uniformly on a subinterval bounded away from $s=0$, then the corresponding contribution to $e^{\Lambda(S)/2}$ is contractive on that subinterval.\mainonly{ The proof is given in Appendix~\ref{app-sec:stability-one-sided}.}
\end{restatable}

The term $\mathcal E_{\mathrm{rev}}(t)$ combines initialisation mismatch and reverse-drift error along the ideal reverse trajectory, and the one-sided formulation does not require the learned drift to be globally contractive. In the conditional reverse drift, the structural heat contribution appears with the anti-diffusive sign $+\alpha(s)\mathcal A_H$, and the term $+\beta(s)\gamma(X-M_0)$ can expand perturbations before being combined with the score. Any contractivity is a property of the full reverse drift, including the score contribution.

Together with the Euler--Maruyama bound for the learned drift (Corollary~\ref{cor:appendix_numerical_generation_error} in Appendix~\ref{app-sec:em-error}), Theorem~\ref{thm:stability_reverse_main}
yields, under the stated regularity and integrability assumptions, the following pre-decoding error bound
\[
W_2\bigl(\Law(\widehat Y_S^\theta),\hat p_{\mathrm{data}}\bigr)
\le
C\,\Delta t^{1/2}
+
e^{\Lambda(S)/2}\,\mathcal E_{\mathrm{rev}}(S).
\]
The ideal reverse process may require separate endpoint conditions at $s=0$. Corollary~\ref{cor:appendix_total_generation_error} also gives a bound for stopping at $s=\varepsilon>0$, including the residual discrepancy from the data law. The stability argument applies under the stated assumptions on the two reverse SDEs and their shared additive noise. The components which are specific to the two-sided operator are assessed in Section~\ref{subsec:exp_ablations}.

\subsection{Equivariant reverse-drift model}
\label{subsec:method_equivariance}

The learned field should transform consistently under independent permutations of rows and columns, corresponding to relabellings of nodes and hyperedges. We work with the product symmetry group $S_n\times S_m$, acting on incidence-space states as $X \mapsto P X Q^\top$, where $P$ and $Q$ are permutation matrices of sizes $n\times n$ and $m\times m$. We therefore parameterise $u^\theta_s$ using an $S_n\times S_m$-equivariant architecture \citep{maron2018invariant}, satisfying
\begin{equation}
u^\theta_s(PXQ^\top)=P\,u^\theta_s(X)\,Q^\top.
\label{eq:equivariance_reverse_drift}
\end{equation}
Under the natural symmetry assumptions on the data distribution and forward process, the $L^2$-optimal state-only reverse-drift target itself belongs to this symmetry class.

\begin{restatable}[Equivariance of the $L^2$-optimal target]{proposition}{equivariancethm}
\mainonly{\label{prop:equivariance_main}}
Assume the data distribution is invariant under independent row and column permutations, the base mean is constant, $M_0=c\,\1_n\1_m^\top$, and the driving Brownian motion is isotropic. The operators $L_V$ and $L_E$ transform equivariantly under $H\mapsto PHQ^\top$ by direct computation from their definitions. Then the forward operator satisfies $\mathcal A_{PHQ^\top}(PXQ^\top)=P\mathcal A_H(X)Q^\top$. The conditional forward law and conditional reverse drift transform equivariantly, and the $L^2$-optimal state-only reverse-drift target satisfies $u_s^{\mathrm{L^2}}(PXQ^\top)=P\,u_s^{\mathrm{L^2}}(X)\,Q^\top$.\mainonly{ The proof is given in Appendix~\ref{app-sec:equivariance}.}
\end{restatable}

\begin{algorithm}[t]
\caption{\label{alg:sampling}\HEDGE}
\begin{minipage}[t]{0.48\linewidth}
\textbf{Training}
\begin{algorithmic}[1]
\STATE {\bfseries Input:} data $\{H^{(i)}\}$, time law $\rho$ on $(0,S]$, step size $\eta$
\FOR{each training iteration}
    \STATE Sample $H \sim p_{\mathrm{data}}$ and $s \sim \rho$
    \STATE Solve \eqref{eq:mean_ode_main}--\eqref{eq:cov_ode_main} to get $m_s(H)$ and $C_s(H)$
    \vspace{-0.2cm}
    \STATE Sample $\xi \sim \mathcal N(0,I_{nm})$, form \\ 
    $X_s \leftarrow \matop(m_s(H)+C_s(H)^{1/2}\xi)$
    \STATE Evaluate target drift $u^\star \leftarrow u^\star_{s\mid H}(X_s)$ 
    \STATE Compute loss $\mathcal L(\theta) \leftarrow \|u_s^\theta(X_s)-u^\star\|_F^2$
    \STATE Update $\theta \leftarrow \theta-\eta\,\nabla_\theta \mathcal L(\theta)$
\ENDFOR
\STATE \textbf{return} trained reverse field $u_s^\theta$
\end{algorithmic}
\end{minipage}
\hfill
\begin{minipage}[t]{0.48\linewidth}
\textbf{Generation}
\begin{algorithmic}[1]

\STATE {\bfseries Input:} $u_s^\theta$, schedule $\beta$, $(M_0,\gamma,\tau)$, horizon $S$; projection $\Pi_\vartheta$ (Sec.~\ref{subsec:exp_protocol})

\STATE Choose time grid $0=t_0<\cdots<t_K=S$
\STATE Initialise $Y_{t_0} \sim \pi$ from \eqref{eq:ou_invariant}
\FOR{$k=0,\dots,K-1$}
    \STATE $\Delta t \leftarrow t_{k+1}-t_k$, $s_k \leftarrow S-t_k$
    \STATE $\widehat u_k \leftarrow u^\theta_{s_k}(Y_{t_k})$
    \STATE Sample $\varepsilon_k \sim \mathcal N(0,I_{nm})$
    \STATE $Y_{t_{k+1}} \leftarrow Y_{t_k} + \Delta t\,\widehat u_k + \sqrt{2\tau\beta(s_k)\Delta t}\,\matop(\varepsilon_k)$
\ENDFOR
\STATE \textbf{return} $\widehat H \leftarrow \Pi_\vartheta(Y_{t_K})$
\end{algorithmic}
\end{minipage}
\end{algorithm}

\section{Related Work}
\label{sec:related-work}

\paragraph{Hypergraph generation.}
Hypergraph generation has been studied through statistical and mechanistic approaches. Classical random, configuration-style, and latent-structure models prescribe higher-order marginals or impose structured dependence through rewiring, latent variables, or block structure \citep{chodrow2020configuration,chodrow2021generative,ghoshdastidar2017consistency,turnbull2024latent}. More recent learned generators include mechanistic sequence-based models and diffusion-style approaches \citep{do2020structural,gailhard2025hygene,wu2025dde}. In particular, HYGENE \citep{gailhard2025hygene} performs diffusion on a graph-like bipartite encoding of the hypergraph, while DDE \citep{wu2025dde} generates hyperlinks by diffusing latent hyperlink embeddings. \HEDGE\ differs from these approaches by operating directly on relaxed incidence matrices, so node--hyperedge participation is modelled in its native rectangular representation rather than through a graph reduction or latent hyperlink embedding.

\paragraph{Heat diffusion and spectral structure on hypergraphs.}
Our forward process is built from spectral operators on hypergraphs. On the node side, we use the normalised hypergraph Laplacian of \citet{zhou2006learning} and \citet{chan2018spectral}. More broadly, Laplacian heat flow has recently been used as a structure-aware noising mechanism in graph generation \citep{stephenson2026g3}. The limitation of pure heat flow, however, is that it contracts toward a data-dependent low-complexity limit. A common non-degenerate base law for generation requires an additional mechanism. \HEDGE\ addresses this by incorporating an Ornstein--Uhlenbeck component, thus yielding structure-aware noising together with an explicit Gaussian base law.

\paragraph{Reverse-time generative modelling.}
\HEDGE\ has natural connections to score-based diffusion models \citep{ho2020ddpm,song2021scorebased}. In standard score-based models, the reverse drift uses a known forward drift and a learned marginal score. Our forward process is conditioned on the observed hypergraph, so the state-only reverse drift combines the marginal score with the posterior-average forward drift. We learn a state-only reverse drift by regressing onto conditional reverse-drift targets, closer in spirit to vector-field regression methods such as flow matching and generator matching \citep{lipman2023flowmatching,holderrieth2024generatormatching,stephenson2026g3}.

\section{Experiments}
\label{sec:experiments}

We evaluate \HEDGE\ in two ways. First, we consider controlled synthetic ablations to test the modelling choices behind the structured heat--OU process and the two-sided incidence operator. Second, a real-data benchmark compares \HEDGE\ with three statistical null models across six datasets. Metric definitions, calibration results and further examples are given in Appendix~\ref{sec:additional_numerics}.

\subsection{Experimental setup}
\label{subsec:exp_protocol}

\paragraph{Task and representation.} Many real datasets consist of a single large observed hypergraph. For each dataset and seed we therefore sample a collection of fixed-size subhypergraphs on $64$ nodes and $m$ hyperedges (hyperedges have size at least two and vertices may be isolated), train on $80\%$ of the collection and compare generated samples with the held-out $20\%$, evaluating $\min(n_{\mathrm{test}},n_{\mathrm{gen}})$ samples per seed (five on \textsc{Actor} and \textsc{DBLP}, eight on the other four datasets). One \HEDGE\ model is trained per $(n,m)$: the method as presented is fixed-size (see Section~\ref{sec:limitations} for a discussion of the limitations of this approach).

\paragraph{Projection and post-processing.} Relaxed samples are binarised with $\Pi_\vartheta(Y)=\mathrm{Post}(\mathbf 1\{Y\ge\vartheta\};Y)$: thresholding at $\vartheta=0.55$, followed by an addition-only rule enforcing hyperedge size $\ge2$ and node degree $\ge1$ (Algorithm~\ref{alg:post}). The size constraint matches the benchmark support. The degree rule adds an incidence to each vertex left isolated after the size rule, whereas the real training sets allow isolated vertices. We therefore report both full projection and threshold-only decoding, and distinguish improvements that hold under both from those that depend on post-processing. The threshold was fixed before real-data evaluation; Appendix~\ref{app:projection} gives its tuning history and sensitivity diagnostics. All methods use the same held-out data and scoring functions, without rejecting or regenerating samples based on structure. Hardware and costs are in Appendix~\ref{app:compute}. Code: \url{https://github.com/chris-nemeth/HEDGE}.

\begin{table}[t]
\centering
\small
\setlength{\tabcolsep}{3pt}
\begin{tabular}{lrrrrr}
\toprule
Ablation & Intersect. WD & Tail gap & Node spec. WD & Edge spec. WD & Feature MMD \\
\midrule
Full \HEDGE\ & \textbf{0.105 $\pm$ 0.076} & \textbf{0.012 $\pm$ 0.004} & \textbf{0.034 $\pm$ 0.020} & \textbf{0.036 $\pm$ 0.029} & \textbf{0.331 $\pm$ 0.258} \\
OU only & 0.132 $\pm$ 0.080 & \textit{0.012 $\pm$ 0.006} & 0.044 $\pm$ 0.020 & 0.039 $\pm$ 0.030 & 0.391 $\pm$ 0.296 \\
Node only & 0.137 $\pm$ 0.086 & 0.019 $\pm$ 0.006 & 0.036 $\pm$ 0.021 & \textit{0.036 $\pm$ 0.030} & \textit{0.353 $\pm$ 0.311} \\
Edge only & \textit{0.126 $\pm$ 0.079} & 0.014 $\pm$ 0.004 & \textit{0.034 $\pm$ 0.022} & 0.037 $\pm$ 0.030 & 0.373 $\pm$ 0.279 \\
\bottomrule
\end{tabular}
\caption{Simulated-data ablations. Values are mean $\pm$ standard error across the four synthetic-regime means, each first averaged over seeds. Lower is better for all columns; best values are \textbf{bold} and second-best values are in \textit{italics}.}
\label{tab:simulated-hedge-ablations}
\end{table}

\paragraph{Datasets.}
The real-data study uses six datasets, \textsc{Cora}, \textsc{CiteSeer}, \textsc{Actor}, \textsc{House-Committees}, \textsc{DBLP}, and \textsc{Twitch}, spanning sparse citation-like structure, broader co-membership regimes, and heterogeneous committee overlap (Appendix~\ref{sec:datasets}).

\paragraph{Comparators.}
We compare \HEDGE\ with three statistical null models, \texttt{HCM-MCMC} \citep{chodrow2020configuration}, \texttt{HERGM}, and \texttt{ER-HG}, under the same fixed-size protocol, datasets and metric pipeline, over three independent seeds. \texttt{HCM-MCMC} runs one pairwise-reshuffle Markov chain per sample, started from a training subhypergraph and preserving that hypergraph's degree and size sequences exactly. The degree and size Wasserstein distances compare generated samples with held-out subhypergraphs, while signed calibration gaps use the training set as reference. Appendix~\ref{app:hcm} reports the contextual train--held-out reference discrepancy for every metric. \texttt{ER-HG} draws independent incidences at the training-set density, and \texttt{HERGM} is the maximum-entropy model of \citet{saracco2025entropy}, with independent incidences whose expected degrees and sizes match those of a randomly drawn training subhypergraph for each sample. Both receive the same post-processing rule as \HEDGE\ (Appendix~\ref{app:implementation}). 

\paragraph{Metrics.}
Figure~\ref{fig:main-realdata-comparison} (left) emphasises three metrics that probe higher-order structure. These are \emph{Overlap Tail Gap}, \emph{Intersection WD} (Wasserstein distance) and \emph{Feature MMD} (squared maximum mean discrepancy). Lower is better in all cases. Calibration and marginal-distribution results (density, mean hyperedge size and node degree, and the size and degree Wasserstein distances) are in Appendix~\ref{sec:real-data-extra-results}; definitions are in Appendix~\ref{sec:metrics}.

\subsection{Ablation study}
\label{subsec:exp_ablations}

We use simulated hypergraph distributions to isolate the main modelling choices in \HEDGE. The ablations test whether structured heat improves over unstructured OU noising, and whether using both sides of the incidence geometry is preferable to smoothing only across nodes or only across hyperedges. Full \HEDGE\ combines both heat terms with the OU component. \emph{OU only} removes the heat term; \emph{Node only} keeps only \(L_V(H)X\); and \emph{Edge only} keeps only \(XL_E(H)\). Table~\ref{tab:simulated-hedge-ablations} summarises the two ablations. Appendix~\ref{sec:ablations} reports the per-regime results and identifies the runs used for each row.

Full \HEDGE\ has the lowest displayed mean Intersection WD and Feature MMD. At the displayed precision, it ties OU only on Tail gap, Edge only on node-spectral WD, and Node only on edge-spectral WD.

\subsection{Real-data comparisons}
\label{subsec:exp_realdata_direct}

The main empirical result is a matched real-data comparison across \textsc{Cora}, \textsc{CiteSeer}, \textsc{Actor}, \textsc{House-Committees}, \textsc{DBLP}, and \textsc{Twitch}. Figure~\ref{fig:main-realdata-comparison} (left) reports the higher-order metrics most aligned with the goal of incidence-space hypergraph generation; calibration and marginal-distribution results are given in Appendix~\ref{sec:additional_numerics}.

With the full projection, \HEDGE\ has the best Intersection WD on all six datasets (tied to three decimals with \texttt{HCM-MCMC} on \textsc{Cora}), the best node- and edge-spectral distances on all six, and the best Feature MMD on four. Its Overlap Tail Gap is lower than HCM-MCMC's on all six, although ER-HG is lower on \textsc{CiteSeer} and \textsc{DBLP}, and HERGM on \textsc{CiteSeer}. Against HCM-MCMC, wins shared by both \HEDGE\ decodings are Overlap Tail Gap and node-spectral WD on six datasets, Intersection WD and edge-spectral WD on five, and Feature MMD on three. The full-projection wins on Cora Feature MMD and DBLP edge-spectral WD depend on the rule (Appendix~\ref{app:projection}). Against the post-processed samples from HERGM and ER-HG, \HEDGE\ is best on Intersection WD and both spectral distances on every dataset under either decoding. The degree rule accounts for the large calibration errors of these two baselines. HCM-MCMC remains stronger on the hyperedge-size Wasserstein distance and the mean-size and mean-degree gaps (four of six datasets each), consistent with its exact preservation of the marginals of the subhypergraph each chain starts from.

HCM-MCMC preserves each training subhypergraph's degree and size sequences by construction, yet Figure~\ref{fig:main-realdata-comparison} (left) shows that \HEDGE\ better matches the held-out overlap structure in these experiments. Figure~\ref{fig:main-realdata-comparison} (right) shows the held-out \textsc{House-Committees} subhypergraph with the smallest average distance to the other held-out examples. Beside it is the nearest generated sample from each of \HEDGE, HCM-MCMC and ER-HG, using Euclidean distances on the standardised Feature MMD features. These examples (also Appendix~\ref{sec:bipartite-realdata}) illustrate the metrics.

\begin{figure*}[t]
\centering

\begin{minipage}[t]{0.62\textwidth}
\centering
\vspace{0pt}
\scriptsize
\setlength{\tabcolsep}{3.5pt}

\begin{tabular}{@{}p{1.15cm}lrrrrrr@{}}
\toprule
Metric & Method & Cora & CiteSeer & Actor & H.-Comm. & DBLP & Twitch \\
\midrule
\multirow{5}{=}{\makecell[l]{Overlap\\Tail Gap}}
 & \HEDGE & \textbf{0.005} & 0.012 & \textbf{0.003} & \textbf{0.020} & \textit{0.013} & \textbf{0.010} \\
 & {\color{gray}\HEDGE\ (thr.)} & 0.006 & 0.012 & 0.003 & 0.021 & 0.013 & 0.010 \\
 & \texttt{HCM-MCMC} & \textit{0.007} & 0.013 & \textit{0.005} & \textit{0.040} & 0.022 & \textit{0.014} \\
 & \texttt{HERGM} & 0.010 & \textit{0.008} & 0.046 & 0.160 & 0.047 & 0.043 \\
 & \texttt{ER-HG} & \textbf{0.005} & \textbf{0.004} & 0.032 & 0.109 & \textbf{0.012} & 0.024 \\
\midrule
\multirow{5}{=}{\makecell[l]{Intersection\\WD}}
 & \HEDGE & \textbf{0.020} & \textbf{0.050} & \textbf{0.038} & \textbf{0.225} & \textbf{0.072} & \textbf{0.070} \\
 & {\color{gray}\HEDGE\ (thr.)} & 0.016 & 0.037 & 0.032 & 0.214 & 0.064 & 0.066 \\
 & \texttt{HCM-MCMC} & \textbf{0.020} & \textit{0.071} & \textit{0.051} & \textit{0.367} & \textit{0.139} & \textit{0.122} \\
 & \texttt{HERGM} & 0.163 & 0.199 & 0.353 & 0.623 & 0.340 & 0.369 \\
 & \texttt{ER-HG} & \textit{0.133} & 0.162 & 0.312 & 0.607 & 0.288 & 0.330 \\
\midrule
\multirow{5}{=}{\makecell[l]{Feature\\MMD}}
 & \HEDGE & \textbf{0.103} & \textbf{0.039} & \textbf{0.113} & \textbf{0.069} & \textit{0.251} & \textit{0.108} \\
 & {\color{gray}\HEDGE\ (thr.)} & 0.214 & 0.089 & 0.107 & 0.058 & 0.300 & 0.196 \\
 & \texttt{HCM-MCMC} & \textit{0.133} & \textit{0.095} & \textit{0.309} & \textit{0.138} & \textbf{0.120} & \textbf{0.075} \\
 & \texttt{HERGM} & 1.071 & 0.648 & 1.017 & 0.441 & 0.345 & 0.818 \\
 & \texttt{ER-HG} & 1.106 & 0.639 & 1.059 & 0.428 & 0.295 & 0.748 \\
\bottomrule
\end{tabular}

\end{minipage}
\hfill
\begin{minipage}[t]{0.34\textwidth}
\centering
\vspace{0pt}
\includegraphics[width=\linewidth]{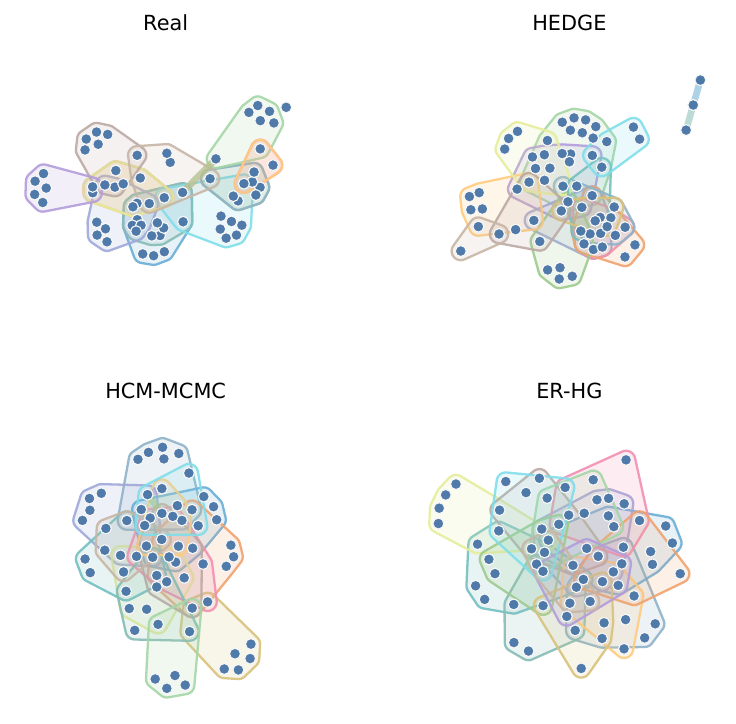}
\end{minipage}

\caption{
Real-data comparisons. Left: higher-order metrics across real hypergraph datasets, mean over three seeds (standard deviations in Appendix~\ref{sec:real-data-extra-results}). Lower is better for all metrics. Ranks use the displayed three-decimal means. Among the ranked rows (\HEDGE\ with the full projection, HCM-MCMC, HERGM, ER-HG), best values are \textbf{bold} and second-best values are in \textit{italics}. The grey rows report \HEDGE\ with thresholding only (no post-processing rule) as a diagnostic and are not ranked. H.-Comm. denotes House Committees.
Right: qualitative hypergraph comparison on House-Committees, showing a real held-out hypergraph alongside generated hypergraphs from \HEDGE, HCM-MCMC, and ER-HG. 
}
\label{fig:main-realdata-comparison}
\end{figure*}

\section{Limitations and Future Work}
\label{sec:limitations}

\HEDGE\ is a first step toward diffusion-based hypergraph generation that operates directly in the relaxed incidence space. The advantage of this approach is that it allows us to utilise linear Gaussian diffusion models and provide a more flexible modelling framework for complex higher-order relationships. However, several limitations remain which should be addressed in future work.

\paragraph{Fixed size and scalability.} As implemented, \HEDGE\ is fixed-size: one model operates on an $n\times m$ incidence space, and does not cover the setting of variable-size hypergraphs within a dataset where $n$ or $m$ may vary. Variable-size generation could be handled via masking or a set-valued architecture, but that was not considered in this work. The reported runs also build a dense $nm\times nm$ operator per training example with a cost $O((nm)^2)$ (in time and memory); the separable alternative alleviates this constraint, but much larger hypergraphs may require more computing resources or approximate spectral or local operators.

\paragraph{Relaxation, decoding and initialisation.} The relaxation of the incidence matrix buys conditional-Gaussian targets and spectral separability at the price of a projection back to the discrete space. The node-degree rule adds an incidence to each vertex left isolated after the size rule. Real datasets allow isolated vertices, and the degree rule can change individual comparisons (Appendix~\ref{app:projection}). We report both decodings and identify which comparisons depend on the rule. Generation also starts from the OU invariant law rather than the finite-horizon mixture; Appendix~\ref{app:terminal_law} bounds the gap and finds its downstream effect is minimal, except for one small Overlap Tail Gap difference.

\paragraph{Future work.} Training with the separable solver would remove the dense operator and retain the full conditional covariance, and conditioning the reverse-drift network on the incidence dimensions is a natural route to a single model across sizes. Initialisation could be brought closer to the finite-horizon mixture with learned Gaussianising transports \citep{duan2023transport,cabezas2023transport,cabezas2024markovian}. A directed extension would diffuse jointly on source and target incidence matrices \citep{preti2024higher,abuissa2026dinghy}, and degree- or size-conditioned generation could be obtained by guiding the reverse SDE \citep{makris2026tempered}.

\section*{Acknowledgments and Disclosures}
The author gratefully acknowledges the support of the EPSRC-funded ProbAI Hub (EP/Y028783/1). This work was supported by Research England under the Expanding Excellence in England (E3) funding stream, which was awarded to MARS: Mathematics for AI in Real-world Systems in the School of Mathematical Sciences at Lancaster University. 

The author declares no competing interests.

\bibliography{hedge_revised}
\bibliographystyle{plainnat}

\appendix
\renewcommand{\mainonly}[1]{}
\renewcommand{\restatetag}[1]{\tag{\ref*{#1}}}

\section{Proofs of Theoretical Results}
\label{app:theory_proofs}

This appendix contains the main theoretical results underpinning the structured heat--OU diffusion used by \HEDGE. The development proceeds in five steps. First, we analyse the pure heat component as a two-sided dissipative semigroup. Second, we establish the conditional linear-Gaussian structure of the forward process. Third, we characterise the exact conditional reverse drift and prove that the marginal drift of the forward process coincides with the $L^2$-optimal state-only target. Fourth, we prove the finite-horizon stability results under a one-sided Lipschitz $W_2$ bound  and the Euler--Maruyama discretisation bound. Finally, we show that the $L^2$-optimal state-only target is $S_n\times S_m$-equivariant.

Throughout, we equip $\R^{n\times m}$ with the Frobenius inner product $\langle X,Y\rangle_F := \operatorname{tr}(X^\top Y)$ and norm $\|X\|_F^2 := \langle X,X\rangle_F$. Inverse degree matrices are interpreted entrywise on strictly positive diagonal entries, with zero on zero diagonal entries.

\subsection{The pure heat operator as a dissipative semigroup}
\label{app-sec:heat-operator}

\begin{lemma}[Positive semidefiniteness of $L_E$]
\label{lem:appendix_LE_psd}
Fix $H\in\{0,1\}^{n\times m}$. Then $A_E(H)$ is symmetric and entrywise nonnegative, and $L_E(H)=I_m-D_{\mathrm{ov}}^{-1/2}A_E(H)D_{\mathrm{ov}}^{-1/2}$ is symmetric positive semidefinite.
\end{lemma}

\begin{proof}
Let $M:=D_E^{-1/2}H^\top H D_E^{-1/2}$; since $H^\top H$ is symmetric and
$D_E^{-1/2}$ is diagonal, $M$ is symmetric, so
$A_E(H)=\operatorname{offdiag}(M)$ is symmetric. For $1\le i,j\le m$,
$(H^\top H)_{ij}=\sum_v H_{v,i}H_{v,j}=|e_i\cap e_j|\ge 0$, so $M$ is
entrywise nonnegative; $\operatorname{offdiag}$ preserves entrywise
nonnegativity, hence so is $A_E(H)$.
 
For positive semidefiniteness, fix $x\in\R^m$ and partition the indices as
$\mathcal I_0:=\{i:\delta_i=0\}$ and $\mathcal I_+:=\{i:\delta_i>0\}$, where
$\delta_i:=(D_{\mathrm{ov}})_{ii}=\sum_j A_E(H)_{ij}$. For $i\in\mathcal I_0$,
nonnegativity of $A_E(H)$ entrywise forces
$A_E(H)_{ij}=A_E(H)_{ji}=0$ for all $j$, so row and column $i$ contribute
nothing to any quadratic form involving $A_E(H)$. Moreover, under our
convention $(D_{\mathrm{ov}}^{-1/2})_{ii}=0$ whenever $\delta_i=0$, so the
$i$th row and column of $D_{\mathrm{ov}}^{-1/2}A_E(H)D_{\mathrm{ov}}^{-1/2}$
are identically zero as well. Writing
\[
x^\top L_E(H)x
=
x^\top x
-
x^\top D_{\mathrm{ov}}^{-1/2}A_E(H)D_{\mathrm{ov}}^{-1/2}x,
\]
the second term therefore reduces to a sum over $\mathcal I_+\times\mathcal I_+$.
Setting $y_i:=x_i/\sqrt{\delta_i}$ for $i\in\mathcal I_+$,
\[
x^\top D_{\mathrm{ov}}^{-1/2}A_E(H)D_{\mathrm{ov}}^{-1/2}x
=
\sum_{i,j\in\mathcal I_+} A_E(H)_{ij}\,y_i y_j.
\]
For the first term, splitting the sum over $\mathcal I_0\cup\mathcal I_+$
and using $x_i^2=\delta_i y_i^2$ on $\mathcal I_+$ gives
\[
x^\top x
=
\sum_{i\in\mathcal I_0} x_i^2
+
\sum_{i\in\mathcal I_+}\delta_i y_i^2
=
\sum_{i\in\mathcal I_0} x_i^2
+
\sum_{i\in\mathcal I_+}\Bigl(\sum_{j\in\mathcal I_+} A_E(H)_{ij}\Bigr) y_i^2,
\]
where the second equality uses $\delta_i=\sum_{j\in\mathcal I_+} A_E(H)_{ij}$
for $i\in\mathcal I_+$ (the terms with $j\in\mathcal I_0$ vanish by the
argument above). Combining,
\[
x^\top L_E(H)x
=
\sum_{i\in\mathcal I_0} x_i^2
+
\sum_{i,j\in\mathcal I_+} A_E(H)_{ij}(y_i^2 - y_i y_j).
\]
By symmetry of $A_E(H)$, the second sum equals
$\tfrac12\sum_{i,j\in\mathcal I_+} A_E(H)_{ij}(y_i^2+y_j^2-2y_iy_j)
=\tfrac12\sum_{i,j\in\mathcal I_+} A_E(H)_{ij}(y_i-y_j)^2$. Therefore
\[
x^\top L_E(H)x
=
\sum_{i\in\mathcal I_0} x_i^2
+
\tfrac12\sum_{i,j\in\mathcal I_+} A_E(H)_{ij}(y_i-y_j)^2,
\]
and both terms are nonnegative because $A_E(H)$ is entrywise nonnegative,
proving $L_E(H)\succeq 0$.
\end{proof}

\begin{proof}[Proof of Proposition~\ref{prop:heat_operator_main}]
Write $L_V:=L_V(H)$ and $L_E:=L_E(H)$. Both are symmetric. For $L_V$, write $L_V=I_n-BB^\top$ with $B:=D_V^{-1/2}HD_E^{-1/2}$; by Cauchy--Schwarz, $|(B^\top x)_e|^2\le\sum_{v\in e}x_v^2/d_V(v)$, and summing over $e$ gives $\|B^\top x\|^2\le\sum_{v:\,d_V(v)>0}x_v^2\le\|x\|^2$, so $0\preceq L_V\preceq I_n$ under the zero-inverse convention (an isolated vertex contributes a unit diagonal entry). $L_E\succeq0$ by Lemma~\ref{lem:appendix_LE_psd}, and its eigenvalues lie in $[0,2]$ as for any symmetric normalised Laplacian.

\emph{Self-adjointness:} For $X,Y\in\R^{n\times m}$,
$\langle X,\AH(Y)\rangle_F=\operatorname{tr}(X^\top L_V Y)+\operatorname{tr}(X^\top Y L_E)$. Using symmetry of $L_V, L_E$, $\operatorname{tr}(X^\top L_V Y)=\operatorname{tr}((L_V X)^\top Y)$ and $\operatorname{tr}(X^\top Y L_E)=\operatorname{tr}((XL_E)^\top Y)$, giving $\langle X,\AH(Y)\rangle_F=\langle \AH(X),Y\rangle_F$.

\emph{PSD:} $\langle X,\AH(X)\rangle_F = \operatorname{tr}(X^\top L_V X) + \operatorname{tr}(X L_E X^\top)\ge 0$ since $L_V, L_E\succeq 0$.

\emph{Dissipativity:} If $\dot Z_s = -\AH(Z_s)$ then $\frac{\mathrm{d}}{\mathrm{d}s}\|Z_s\|_F^2=-2\langle Z_s,\AH(Z_s)\rangle_F\le 0$.

\emph{Spectral decoupling:} Write $L_V=U\Lambda U^\top$, $L_E=VMV^\top$ with orthogonal $U,V$ and diagonal $\Lambda,M$. Set $\widetilde Z_s := U^\top Z_s V$. Then
\[
\frac{\mathrm{d}}{\mathrm{d}s}\widetilde Z_s = -U^\top(L_V Z_s + Z_s L_E)V = -\Lambda \widetilde Z_s - \widetilde Z_s M,
\]
so entrywise $\frac{\mathrm{d}}{\mathrm{d}s}\widetilde Z_s(i,j) = -(\lambda_i+\mu_j)\widetilde Z_s(i,j)$, giving $\widetilde Z_s(i,j) = e^{-s(\lambda_i+\mu_j)}\widetilde Z_0(i,j)$ and $Z_s = e^{-sL_V}Z_0 e^{-sL_E}$.
\end{proof}

\begin{corollary}[Pure-heat limit]
\label{cor:pure_heat_limit}
Let $Z_s$ satisfy $\tfrac{\mathrm d}{\mathrm ds}Z_s = -\AH(Z_s)$ with
$Z_0\in\R^{n\times m}$. Then $Z_s\to \Pi_V Z_0\,\Pi_E$ as $s\to\infty$, where
$\Pi_V$ and $\Pi_E$ are the orthogonal projections onto $\ker L_V(H)$ and
$\ker L_E(H)$, respectively. Moreover, letting
\[
\eta:=\min\{\lambda_i+\mu_j:\lambda_i+\mu_j>0\}>0,
\]
the convergence is exponential at rate $\eta$:
\[
\|Z_s-\Pi_V Z_0\,\Pi_E\|_F\le e^{-\eta s}\,\|Z_0\|_F\qquad \text{for all }s\ge 0.
\]
\end{corollary}

\begin{proof}
Write $L_V(H)=U\Lambda U^\top$ and $L_E(H)=VMV^\top$ with orthogonal
$U,V$ and diagonal
$\Lambda=\diag(\lambda_1,\dots,\lambda_n)$,
$M=\diag(\mu_1,\dots,\mu_m)$; by Proposition~\ref{prop:heat_operator_main},
$\lambda_i,\mu_j\ge 0$. Let $\Pi_\Lambda\in\R^{n\times n}$ be the diagonal
matrix with $(\Pi_\Lambda)_{ii}=1$ if $\lambda_i=0$ and $0$ otherwise, and
similarly $\Pi_M\in\R^{m\times m}$ with $(\Pi_M)_{jj}=\mathbf 1[\mu_j=0]$.
Then $\Pi_V=U\Pi_\Lambda U^\top$ and $\Pi_E=V\Pi_M V^\top$ are the
orthogonal projections onto $\ker L_V(H)$ and $\ker L_E(H)$.

Setting $\widetilde Z_s:=U^\top Z_s V$, Proposition~\ref{prop:heat_operator_main}
gives $\widetilde Z_s(i,j)=e^{-s(\lambda_i+\mu_j)}\widetilde Z_0(i,j)$
entrywise. Split the indices into $\mathcal K:=\{(i,j):\lambda_i+\mu_j=0\}$
(equivalently, $\lambda_i=\mu_j=0$ since both are nonnegative) and its
complement $\mathcal K^c$. For $(i,j)\in\mathcal K$,
$\widetilde Z_s(i,j)=\widetilde Z_0(i,j)$ for all $s$; for
$(i,j)\in\mathcal K^c$, $|\widetilde Z_s(i,j)|\le e^{-\eta s}|\widetilde Z_0(i,j)|$
by definition of $\eta$. Hence
\[
\|\widetilde Z_s-\Pi_\Lambda\widetilde Z_0\Pi_M\|_F^2
=\sum_{(i,j)\in\mathcal K^c}|\widetilde Z_s(i,j)|^2
\le e^{-2\eta s}\sum_{(i,j)\in\mathcal K^c}|\widetilde Z_0(i,j)|^2
\le e^{-2\eta s}\|\widetilde Z_0\|_F^2.
\]
Since $U,V$ are orthogonal, $\|\widetilde Z_s-\Pi_\Lambda\widetilde Z_0\Pi_M\|_F=\|Z_s-\Pi_V Z_0\Pi_E\|_F$ and $\|\widetilde Z_0\|_F=\|Z_0\|_F$,
so $\|Z_s-\Pi_V Z_0\Pi_E\|_F\le e^{-\eta s}\|Z_0\|_F$. Taking $s\to\infty$
gives $Z_s\to\Pi_V Z_0\Pi_E$.
\end{proof}

\subsection{Conditional Gaussian structure of the forward process}
\label{app-sec:conditional-gaussian}

\begin{proof}[Proof of Proposition~\ref{prop:conditional_gaussian_main}]
Applying the definition $\operatorname{vec}(AXB)=(B^\top\otimes A)\operatorname{vec}(X)$ to $\AH(X)=L_V X + X L_E$ with symmetric $L_E$ gives $\operatorname{vec}(\AH(X))=(I_m\otimes L_V + L_E\otimes I_n)\operatorname{vec}(X)$. The forward SDE therefore becomes
\[
\mathrm dx_s = -B_s(H) x_s\,\mathrm ds + b_s\,\mathrm ds + \sqrt{2\tau\beta(s)}\,\mathrm dw_s,\quad x_0=\operatorname{vec}(H),
\]
with $b_s = \beta(s)\gamma\operatorname{vec}(M_0)$, a linear time-inhomogeneous SDE with additive Gaussian noise. Let $\Phi(s,u;H)$ be the transition matrix of $\partial_s z_s = -B_s(H) z_s$, i.e.\ the unique invertible solution of $\partial_s \Phi(s,u;H) = -B_s(H)\Phi(s,u;H)$, with $\Phi(u,u;H)=I_{nm}$. 

Setting $y_s:=\Phi(0,s;H) x_s$ and applying the product rule,
\[
\mathrm dy_s = \Phi(0,s;H)b_s\,\mathrm ds + \sqrt{2\tau\beta(s)}\Phi(0,s;H)\,\mathrm dw_s,
\]
so integrating and inverting gives

\[x_s = \Phi(s,0;H)\operatorname{vec}(H) + \int_0^s \Phi(s,u;H) b_u\,\mathrm du + \sqrt{2\tau}\int_0^s \Phi(s,u;H)\sqrt{\beta(u)}\,\mathrm dw_u.
\]

The first two terms are deterministic given $H$, and the Wiener integral has deterministic integrand, and hence is a centred Gaussian \citep[Sec.~4.3, 5.5, 6.1]{sarkka2019applied}. The mean and covariance ODEs \eqref{eq:mean_ode_main}--\eqref{eq:cov_ode_main} then follow by differentiating the explicit expressions \citep[Sec.~5.5, 6.1]{sarkka2019applied}.

For invertibility of $C_s(H)$: $C_s(H) = 2\tau \int_0^s \beta(u)\Phi(s,u;H)\Phi(s,u;H)^\top\,\mathrm du$. Since $\beta>0$ on $(0,s]$ and $\Phi$ is invertible, the integrand is positive definite on a set of positive Lebesgue measure, so $C_s(H)\succ 0$ for all $s\in(0,S]$.
\end{proof}

\subsection{Reverse-time drift and $L^2$-optimal target}
\label{app-sec:reverse-time-drift}

We first give the form of the marginal forward drift under the empirical data measure, which is later used to identify the $L^2$-optimal state-only target with a state-only marginal reverse drift.

\begin{lemma}[Marginal forward drift]
\label{lem:marginal_drift}
Let $\hat p_{\mathrm{data}} = \frac1N\sum_i \delta_{H^{(i)}}$ be the empirical data measure. For each $s\in(0,S]$, the conditional law $p_{s\mid H^{(i)}}$ is a non-degenerate Gaussian by Proposition~\ref{prop:conditional_gaussian_main}, hence strictly positive and $C^\infty$ in $X$, with locally bounded first derivatives uniformly in $i$. The marginal densities satisfy the Fokker--Planck equation with diffusion coefficient $\sqrt{2\tau\beta(s)}$ and state-only drift
\[
\bar b_s(X) = \mathbb E[b_{s\mid H}(X)\mid X_s=X] = -\alpha(s)\sum_i \omega_s(i\mid X) \mathcal A_{H^{(i)}}(X) - \beta(s)\gamma(X - M_0),
\]
where $\omega_s(i\mid X) = p_{s\mid H^{(i)}}(X)/\sum_j p_{s\mid H^{(j)}}(X)$. Under well-posedness and uniqueness of the associated Fokker--Planck equation, the SDE with this drift is a diffusion with these one-time marginals.
\end{lemma}

\begin{proof}
Throughout, all densities are with respect to Lebesgue measure on
$\R^{n\times m}$, and we use $\nabla\cdot$ and $\nabla^2$ to denote the
divergence and Laplacian operators on $\R^{n\times m}$ identified with
$\R^{nm}$ via vectorisation. We assume sufficient regularity for the
derivatives to be exchanged with the finite sums and conditional
expectations below; this holds, in particular, under the smoothness
assumptions used elsewhere in the appendix (strictly positive
$C^1$ conditional densities, square-integrable solutions to the forward
SDE).

For each $H\in\mathcal D$, the conditional density $p_{s\mid H}$ of the
forward SDE \eqref{eq:forward_sde} satisfies the Fokker--Planck (Kolmogorov
forward) equation \citep{fearnhead2025scalable}
\begin{equation}
\partial_s p_{s\mid H}(X)
= -\nabla\!\cdot\!\bigl(b_{s\mid H}(X)\, p_{s\mid H}(X)\bigr)
+ \tau\beta(s)\,\nabla^2 p_{s\mid H}(X),
\label{eq:cond_fp}
\end{equation}
where $b_{s\mid H}(X) = -\alpha(s)\AH(X) - \beta(s)\gamma(X-M_0)$ is the
conditional drift in \eqref{eq:conditional_forward_drift} and the diffusion
contribution is $\tfrac12\,(2\tau\beta(s))\,\nabla^2 = \tau\beta(s)\,\nabla^2$
because the diffusion coefficient is state-independent and isotropic.

The marginal density under the empirical data measure is
$\hat p_s(X) = \tfrac1N\sum_i p_{s\mid H^{(i)}}(X)$. Averaging
\eqref{eq:cond_fp} over $i$ gives
\begin{equation}
\partial_s \hat p_s(X)
= -\frac1N\sum_i \nabla\!\cdot\!\bigl(b_{s\mid H^{(i)}}(X)\, p_{s\mid H^{(i)}}(X)\bigr)
+ \tau\beta(s)\,\nabla^2 \hat p_s(X).
\label{eq:marginal_fp_intermediate}
\end{equation}
The diffusion term is the same as for any individual conditional density
because the diffusion coefficient does not depend on $H$.

We show that the drift sum in \eqref{eq:marginal_fp_intermediate} equals
$\nabla\cdot(\bar b_s(X)\,\hat p_s(X))$ for the state-only drift
$\bar b_s$ of the lemma. Define the posterior weights
\[
\omega_s(i\mid X) := \frac{p_{s\mid H^{(i)}}(X)}{\sum_j p_{s\mid H^{(j)}}(X)}
= \frac{p_{s\mid H^{(i)}}(X)}{N\,\hat p_s(X)},
\]
so that $\omega_s(\cdot\mid X)$ is a probability distribution over
$\{1,\dots,N\}$ for every $X$ in the support of $\hat p_s$, and
$p_{s\mid H^{(i)}}(X) = N\,\omega_s(i\mid X)\,\hat p_s(X)$. Substituting into
the drift sum,
\begin{align}
\frac1N\sum_i b_{s\mid H^{(i)}}(X)\, p_{s\mid H^{(i)}}(X) 
&= \frac1N\sum_i b_{s\mid H^{(i)}}(X)\cdot N\,\omega_s(i\mid X)\,\hat p_s(X) \\
&= \hat p_s(X)\sum_i \omega_s(i\mid X)\, b_{s\mid H^{(i)}}(X).
\end{align}
Defining
\begin{equation}
\bar b_s(X) := \sum_i \omega_s(i\mid X)\, b_{s\mid H^{(i)}}(X)
= \mathbb E_{\hat p_{\mathrm{data}}}\!\bigl[b_{s\mid H}(X)\,\big|\, X_s=X\bigr],
\label{eq:bar_b_def}
\end{equation}
the drift sum equals $\bar b_s(X)\,\hat p_s(X)$, and
\eqref{eq:marginal_fp_intermediate} becomes
\begin{equation}
\partial_s \hat p_s(X)
= -\nabla\!\cdot\!\bigl(\bar b_s(X)\, \hat p_s(X)\bigr)
+ \tau\beta(s)\,\nabla^2 \hat p_s(X).
\label{eq:marginal_fp_final}
\end{equation}
The second equality in \eqref{eq:bar_b_def} follows from the Bayes-rule
identity for the posterior over the latent index $I$ given $X_s=X$ under
the empirical data measure: $(H,X_s) = (H^{(I)},X_s)$ with
$I\sim\mathrm{Uniform}\{1,\dots,N\}$ and
$X_s\mid I=i\sim p_{s\mid H^{(i)}}$, so
$\mathrm{Pr}(I=i\mid X_s=X) = \omega_s(i\mid X)$.

Equation~\eqref{eq:marginal_fp_final} is the Fokker--Planck equation of the
SDE
\[
\mathrm d \widetilde X_s = \bar b_s(\widetilde X_s)\,\mathrm ds + \sqrt{2\tau\beta(s)}\,\mathrm dW_s,
\]
which has state-independent diffusion coefficient
$\sqrt{2\tau\beta(s)}$ and state-only drift $\bar b_s$. Substituting the
explicit form $b_{s\mid H}(X) = -\alpha(s)\AH(X) - \beta(s)\gamma(X-M_0)$
into \eqref{eq:bar_b_def} and using that the OU term is $H$-independent,
\[
\bar b_s(X)
= -\alpha(s)\sum_i \omega_s(i\mid X)\, \mathcal A_{H^{(i)}}(X) - \beta(s)\gamma(X-M_0),
\]
which gives the expression in the statement of the lemma.
\end{proof}

\begin{remark}
\label{rem:marginal_drift_cattiaux}

The Fokker--Planck calculation identifies a diffusion with the same one-time marginals as the latent mixture \citep{gyongy1986mimicking,brigo2008general}. Time reversal is applied to this diffusion under the relevant regularity and well-posedness assumptions \citep{haussmann1986time,cattiaux2023time}.
\end{remark}

\begin{remark}[Markovian projection]
\label{rem:markov_projection}
$(X_s,H)$ is Markov but $X_s$ alone need not be. Proposition~\ref{prop:bayes_reverse_drift_main} identifies the drift of the Markovian projection diffusion of the latent mixture \citep{gyongy1986mimicking,brigo2008general}, which shares its diffusion coefficient and one-time marginals; it does not assert equality of path laws.
\end{remark}

\begin{lemma}[Conditional reverse drift and $L^2$-optimal state-only regression]
\label{thm:appendix_reverse_drift_empirical}
Fix $\mathcal D=\{H^{(1)},\dots,H^{(N)}\}\subset\{0,1\}^{n\times m}$ and let $\hat p_{\mathrm{data}}=\frac1N\sum_i \delta_{H^{(i)}}$. Fix $s\in(0,S]$, and for each $H\in\mathcal D$ let $b_{s\mid H}(X)=-\alpha(s)\AH(X)-\beta(s)\gamma(X-M_0)$. Assume the conditional laws $p_{s\mid H^{(i)}}$ are strictly positive $C^1$ densities and that the conditional and marginal diffusions satisfy the time-reversal regularity and well-posedness assumptions on intervals bounded away from $s=0$. Then:
\begin{enumerate}
    \item[(i)] the conditional reverse-time drift is
    $u^\star_{s\mid H}(X) = -b_{s\mid H}(X) + 2\tau\beta(s)r^\star_{s\mid H}(X) = \alpha(s)\AH(X)+\beta(s)\gamma(X-M_0)+2\tau\beta(s)r^\star_{s\mid H}(X)$;
    \item[(ii)] whenever $C_s(H)\succ 0$, $r^\star_{s\mid H}(X)=-\operatorname{mat}(C_s(H)^{-1}(\operatorname{vec}(X)-m_s(H)))$;
    \item[(iii)] the posterior-averaged reverse drift $\bar u^\star_s(X):=\sum_i \omega_s(i\mid X) u^\star_{s\mid H^{(i)}}(X)$ satisfies
    \[
    \bar u^\star_s(X) = \alpha(s)\sum_i \omega_s(i\mid X) \mathcal A_{H^{(i)}}(X) + \beta(s)\gamma(X-M_0) + 2\tau\beta(s)\nabla_X\log\hat p_s(X);
    \]
    \item[(iv)] the $L^2$-optimal state-only predictor for the population objective \eqref{eq:reverse_drift_loss} is $u_s^{\mathrm{L^2}}(X)=\bar u^\star_s(X)$.
\end{enumerate}
Moreover, the reverse drift of the diffusion coincides with this predictor: by Lemma~\ref{lem:marginal_drift} and the Anderson/Haussmann--Pardoux time-reversal formula \citep{anderson1982reverse,haussmann1986time}, the reversed SDE has state-only drift $u^\star_s(X) = -\bar b_s(X) + 2\tau\beta(s)\nabla_X\log\hat p_s(X) = \bar u^\star_s(X)$.
\end{lemma}

\begin{proof}
\emph{(i)} For state-independent diffusion coefficient
$\sigma_s = \sqrt{2\tau\beta(s)}$, the general
Anderson--Haussmann--Pardoux time-reversal formula
$u^\star_{s\mid H}(X) = -b_{s\mid H}(X) + \nabla_X\!\cdot\!\bigl(p_{s\mid H}(X)\,\sigma_s\sigma_s^\top\bigr)/p_{s\mid H}(X)$
\citep{anderson1982reverse,haussmann1986time} simplifies to
\[
u^\star_{s\mid H}(X)
= -b_{s\mid H}(X) + \sigma_s^2\,\nabla_X\log p_{s\mid H}(X),
\]
because $\sigma_s$ does not depend on the state and so the divergence term
collapses to $\sigma_s^2\,\nabla_X p_{s\mid H}/p_{s\mid H} = \sigma_s^2\,\nabla_X\log p_{s\mid H}$. Substituting the explicit form of $b_{s\mid H}$ from
\eqref{eq:conditional_forward_drift} gives
\[
u^\star_{s\mid H}(X)
= \alpha(s)\AH(X) + \beta(s)\gamma(X-M_0) + 2\tau\beta(s)\,r^\star_{s\mid H}(X).
\]

\emph{(ii)} Direct from $\nabla_x \log\mathcal N(x;m,C) = -C^{-1}(x-m)$
applied to the vectorised conditional law
$\vecop(X_s)\mid H \sim \mathcal N(m_s(H), C_s(H))$ established in
Proposition~\ref{prop:conditional_gaussian_main}, then reshaped to matrix
form via $\matop$. Invertibility of $C_s(H)$ for $s\in(0,S]$ also follows
from Proposition~\ref{prop:conditional_gaussian_main}.

\emph{(iii)} Expand $\bar u^\star_s$ using (i):
\[
\bar u^\star_s(X)
= \alpha(s)\sum_i \omega_s(i\mid X)\,\mathcal A_{H^{(i)}}(X)
+ \beta(s)\gamma(X-M_0)
+ 2\tau\beta(s)\sum_i \omega_s(i\mid X)\,r^\star_{s\mid H^{(i)}}(X).
\]
The structural term is the posterior-weighted sum, the OU term is
$H$-independent (and so its average over $\omega_s(\cdot\mid X)$ is unchanged),
and the score term is handled by the mixture-score identity
\begin{align*}
\sum_i \omega_s(i\mid X)\,r^\star_{s\mid H^{(i)}}(X)
&= \sum_i \frac{p_{s\mid H^{(i)}}(X)}{\sum_j p_{s\mid H^{(j)}}(X)}\,
   \frac{\nabla_X p_{s\mid H^{(i)}}(X)}{p_{s\mid H^{(i)}}(X)} \\
&= \frac{\sum_i \nabla_X p_{s\mid H^{(i)}}(X)}{\sum_j p_{s\mid H^{(j)}}(X)}
= \frac{\tfrac1N\sum_i \nabla_X p_{s\mid H^{(i)}}(X)}{\tfrac1N\sum_j p_{s\mid H^{(j)}}(X)} \\
&= \frac{\nabla_X \hat p_s(X)}{\hat p_s(X)}
= \nabla_X\log\hat p_s(X),
\end{align*}
where $\hat p_s=N^{-1}\sum_i p_{s\mid H^{(i)}}$.

\emph{(iv)} Conditioning on $X_s$, the population objective
\eqref{eq:reverse_drift_loss} decomposes by the Tower property as
\[
\mathcal L(\theta)
= \E_{s\sim\rho}\,\E_{X_s\sim\hat p_s}\!\left[
   \E\!\bigl[\|u^\theta_s(X_s) - u^\star_{s\mid H}(X_s)\|_F^2 \,\big|\, X_s\bigr]
   \right].
\]
For each $X$, the inner conditional expectation is minimised pointwise
over $u^\theta_s(X) \in \R^{n\times m}$ by the conditional expectation
$\E[u^\star_{s\mid H}(X_s) \mid X_s = X]$, since $L^2$-projection onto
the $\sigma(X_s)$-measurable random variables coincides with conditional
expectation \citep[Sec.~34]{billingsley1995probability}. Under the empirical
data measure, the joint law of $(H, X_s)$ is obtained by sampling
$I \sim \mathrm{Uniform}\{1,\dots,N\}$ and $X_s \sim p_{s\mid H^{(I)}}$,
and the posterior law of $I$ given $X_s = X$ is
$\mathrm{Pr}(I = i \mid X_s = X) = \omega_s(i\mid X)$ by Bayes' rule. Hence
\[
u_s^{L^2}(X)
= \E[u^\star_{s\mid H}(X_s) \mid X_s = X]
= \sum_i \omega_s(i\mid X)\, u^\star_{s\mid H^{(i)}}(X)
= \bar u^\star_s(X),
\]
where the last equality is the definition of $\bar u^\star_s$ in (iii).

By Lemma~\ref{lem:marginal_drift}, $\hat p_s$ satisfies the Fokker--Planck equation with drift $\bar b_s$ and diffusion coefficient $\sigma_s=\sqrt{2\tau\beta(s)}$. On intervals bounded away from $s=0$, the finite Gaussian mixture is smooth and positive. Under the regularity and well-posedness assumptions for the diffusion and its time reversal, the reversed marginal density path has drift
\[
u^\star_s(X)=-\bar b_s(X)+\sigma_s^2\nabla_X\log\hat p_s(X).
\]

Substituting $\bar b_s(X) = -\alpha(s)\sum_i \omega_s(i\mid X)\,\mathcal A_{H^{(i)}}(X)
- \beta(s)\gamma(X-M_0)$ from Lemma~\ref{lem:marginal_drift} matches the
expression for $\bar u^\star_s(X)$ given in (iii) term-by-term, so
$u^\star_s = \bar u^\star_s = u_s^{L^2}$.
\end{proof}

\bayesreversedrift*

\begin{proof}[Proof of Proposition~\ref{prop:bayes_reverse_drift_main}]
Part~(iv) of Lemma~\ref{thm:appendix_reverse_drift_empirical} shows that the unrestricted $L^2$-optimal predictor is the conditional expectation $\mathbb E[u^\star_{s\mid H}(X)\mid X_s=X]$, and the final statement of that lemma identifies it with the reverse drift of the diffusion with forward drift $\bar b_s$. By Lemma~\ref{lem:marginal_drift}, this diffusion has the same one-time marginals as the forward process; Remark~\ref{rem:markov_projection} explains why the path laws need not agree. Lemma~\ref{thm:appendix_reverse_drift_empirical} is stated for the empirical data measure $\hat p_{\mathrm{data}}$. The same argument applies to a general data law $p_{\mathrm{data}}$, with the sums over training hypergraphs replaced by integrals against $p_{\mathrm{data}}$, under regularity that allows differentiation under the integral sign.
\end{proof}

\begin{proposition}[Exactness under exact initialisation and a state-only reverse representation]
\label{prop:appendix_ideal_exactness}
Let $\{\hat p_s\}_{s\in[0,S]}$ be the marginal forward family under $\hat p_{\mathrm{data}}$, with $\hat p_0=\hat p_{\mathrm{data}}$ and $\hat p_S$ the terminal law under \eqref{eq:forward_sde}. If the reverse SDE $\mathrm dY_t = \bar u^\star_{S-t}(Y_t)\,\mathrm dt + \sqrt{2\tau\beta(S-t)}\,\mathrm d\bar W_t$ is well-posed on each $[0,S-\varepsilon]$, its associated Fokker--Planck equation has a unique probability evolution there, and it admits a continuous extension to $t=S$, then, when initialised from $Y_0\sim\hat p_S$ and solved exactly, $\Law(Y_t)=\hat p_{S-t}$ for all $t$, and in particular $\Law(Y_S)=\hat p_{\mathrm{data}}$.
\end{proposition}

\begin{proof}
Lemma~\ref{lem:marginal_drift} gives, for $s>0$,
\begin{equation}
\partial_s\hat p_s=-\nabla_X\!\cdot(\bar b_s\hat p_s)+\tau\beta(s)\nabla_X^2\hat p_s.
\label{eq:fwd_fp_marginal}
\end{equation}
Define
\begin{equation}
u^\star_s=-\bar b_s+2\tau\beta(s)\nabla_X\log\hat p_s.
\label{eq:reverse_drift_marginal}
\end{equation}
For $T_\varepsilon=S-\varepsilon$, $\varepsilon>0$, direct substitution shows that $r_t=\hat p_{S-t}$ solves
\[
\partial_t r_t=-\nabla_X\!\cdot(u^\star_{S-t}r_t)+\tau\beta(S-t)\nabla_X^2r_t,
\qquad r_0=\hat p_S,
\quad 0\le t\le T_\varepsilon.
\]
By uniqueness of this Fokker--Planck evolution, these are the one-time marginals of
\begin{equation}
\mathrm dY_t=u^\star_{S-t}(Y_t)\,\mathrm dt+\sqrt{2\tau\beta(S-t)}\,\mathrm d\bar W_t,
\qquad Y_0\sim\hat p_S.
\label{eq:reverse_sde_marginal}
\end{equation}
The mixture-score identity in Lemma~\ref{thm:appendix_reverse_drift_empirical} gives $u^\star_s=\bar u^\star_s$. Since $\varepsilon$ is arbitrary, $\Law(Y_t)=\hat p_{S-t}$ for $t<S$. The continuous endpoint extension and $X_s\to H$ as $s\downarrow0$ give $\Law(Y_S)=\hat p_{\mathrm{data}}$ by weak convergence.
\end{proof}

In practice, \HEDGE\ initialises generation from $\pi=\mathcal N(M_0,(\tau/\gamma)I)$ rather than $\hat p_S$. The discrepancy $\nu_0^\theta\ne\nu_0^\star:=\hat p_S$, together with the training error and solver error, is controlled by the finite-horizon stability results below, and its size can be bounded explicitly.

\begin{proposition}[Computable bounds on the initialisation mismatch]
\label{prop:terminal_bounds}
Let $\hat p_S=\frac1N\sum_{i=1}^N \mathcal N(m_i,C_i)$ with $m_i:=m_S(H^{(i)})$ and $C_i:=C_S(H^{(i)})$ be the empirical finite-horizon mixture, and let $\pi=\mathcal N(\vecop(M_0),\sigma^2 I_{nm})$ with $\sigma^2=\tau/\gamma$. Denote by $c^{(i)}_{jk}$ the eigenvalues of $C_i$, indexed by the joint modes $(j,k)$ of the Kronecker eigenbasis of Proposition~\ref{prop:heat_operator_main}. Then
\[
W_2^2(\hat p_S,\pi)\le \frac1N\sum_{i=1}^N\Bigl[\|m_i-\vecop(M_0)\|^2+\sum_{j,k}\bigl(\sqrt{c^{(i)}_{jk}}-\sigma\bigr)^2\Bigr],
\]
\[
\mathrm{KL}(\hat p_S\,\|\,\pi)\le \frac1{2N}\sum_{i=1}^N\Bigl[\frac{\|m_i-\vecop(M_0)\|^2}{\sigma^{2}}+\sum_{j,k}\Bigl(\frac{c^{(i)}_{jk}}{\sigma^2}-1-\log\frac{c^{(i)}_{jk}}{\sigma^2}\Bigr)\Bigr].
\]
\end{proposition}

\begin{proof}
Both $W_2^2(\cdot,\pi)$ and $\mathrm{KL}(\cdot\,\|\,\pi)$ are convex in their first argument, so each is bounded by its average over the mixture components. For a Gaussian component, $W_2^2(\mathcal N(m_i,C_i),\pi)=\|m_i-\vecop(M_0)\|^2+\operatorname{tr}\bigl(C_i+\sigma^2 I-2\sigma C_i^{1/2}\bigr)$ \citep{chewi2025statistical}, and because $\pi$ is isotropic the trace equals $\sum_{j,k}(\sqrt{c^{(i)}_{jk}}-\sigma)^2$ in any orthonormal eigenbasis of $C_i$. The Gaussian Kullback--Leibler formula gives the second bound, with $\log\det C_i=\sum_{j,k}\log c^{(i)}_{jk}$.
\end{proof}

The exact mixture-to-Gaussian $W_2$ distance has no closed form; the bounds are computable from the per-hypergraph moments, which are themselves obtained by numerically integrating the scalar modal ODEs (Appendix~\ref{app:implementation}). Values for every dataset, together with the measured downstream effect of initialising from $\hat p_S$ instead of $\pi$, are reported in Appendix~\ref{app:terminal_law}. More elaborate initialisation schemes, for example learned Gaussianising transports of the kind used in transport-based Monte Carlo \citep{duan2023transport,cabezas2023transport} or the flow-based transports of Markovian flow matching \citep{cabezas2024markovian}, could also be analysed through the same initialisation term, but we leave such refinements to future work.

\subsection{Stability via a one-sided Lipschitz $W_2$ bound}
\label{app-sec:stability-one-sided}

\begin{definition}[One-sided Lipschitz constant]
\label{def:one_sided_lipschitz}
A measurable function $\kappa:[0,S]\to\R$ is a one-sided Lipschitz constant of $v:[0,S]\times\R^{n\times m}\to\R^{n\times m}$ if $\langle X-Y, v_s(X)-v_s(Y)\rangle_F\le\kappa_s\|X-Y\|_F^2$ for all $X,Y,s$.
\end{definition}

\begin{lemma}[$W_2$ stability via one-sided Lipschitz]
\label{thm:appendix_stability_reverse_one_sided}
Let $u^\star_s,u^\theta_s$ be measurable time-dependent drift fields on $[0,S]$, and consider the reverse-time SDEs
\begin{align}
\mathrm d Y_t^\star
&=
u^\star_{S-t}(Y_t^\star)\,\mathrm dt
+
\sqrt{2\tau\beta(S-t)}\,\mathrm d\bar W_t,
\qquad
Y_0^\star \sim \nu_0^\star,
\label{eq:ideal_reverse_sde_stability}
\\
\mathrm d Y_t^\theta
&=
u^\theta_{S-t}(Y_t^\theta)\,\mathrm dt
+
\sqrt{2\tau\beta(S-t)}\,\mathrm d\bar W_t,
\qquad
Y_0^\theta \sim \nu_0^\theta,
\label{eq:learned_reverse_sde_stability}
\end{align}
driven by the same Brownian motion. Assume both are well-posed with square-integrable solutions, and that $u^\theta$ has an integrable one-sided Lipschitz constant $\kappa:[0,S]\to\R$. Define the error as $e_t := u^\theta_{S-t}(Y_t^\star)-u^\star_{S-t}(Y_t^\star)$ and let $\Lambda(t):=\int_0^t(2\kappa_{S-r}+1)\,\mathrm dr$. Assume $\int_0^S e^{-\Lambda(r)}\mathbb E\|e_r\|_F^2\,\mathrm dr<\infty$. Then
\begin{equation}
\mathbb E\|Y_t^\theta-Y_t^\star\|_F^2
\le
e^{\Lambda(t)}
\left(
\mathbb E\|Y_0^\theta-Y_0^\star\|_F^2
+
\int_0^t e^{-\Lambda(r)}\mathbb E\|e_r\|_F^2\,\mathrm dr
\right),
\label{eq:stability_one_sided_ms}
\end{equation}
and consequently
\begin{equation}
W_2(\Law(Y_t^\theta),\Law(Y_t^\star))
\le
e^{\Lambda(t)/2}
\left(
W_2(\nu_0^\theta,\nu_0^\star)^2
+
\int_0^t e^{-\Lambda(r)}\mathbb E\|e_r\|_F^2\,\mathrm dr
\right)^{\!1/2}.
\label{eq:stability_one_sided_w2}
\end{equation}
If $\kappa_s\le L$ uniformly (global Lipschitz), $\Lambda(t)\le (2L+1)t$ recovers the standard Gr\"onwall estimate. If $\kappa_s\le\kappa_\star<-1/2$ uniformly, then $e^{\Lambda(S)/2}<1$.
\end{lemma}

\begin{proof}
Let $\Delta_t:=Y_t^\theta-Y_t^\star$. Since both SDEs share the same Brownian motion and diffusion coefficient, the stochastic terms cancel:
$\mathrm d\Delta_t = [u^\theta_{S-t}(Y_t^\theta)-u^\star_{S-t}(Y_t^\star)]\,\mathrm dt$. 

By adding and subtracting $u^\theta_{S-t}(Y_t^\star)$ we have,
\[
\mathrm d\Delta_t = a_t\,\mathrm dt + e_t\,\mathrm dt,\quad a_t:=u^\theta_{S-t}(Y_t^\theta)-u^\theta_{S-t}(Y_t^\star).
\]
$\Delta_t$ has absolutely continuous paths, so for a.e.\ $t$ the chain rule gives
$\tfrac{\mathrm{d}}{\mathrm{d}t}\|\Delta_t\|_F^2 = 2\langle\Delta_t,a_t\rangle_F + 2\langle\Delta_t,e_t\rangle_F$.

By the one-sided Lipschitz definition (Definition~\ref{def:one_sided_lipschitz}), $2\langle\Delta_t,a_t\rangle_F\le 2\kappa_{S-t}\|\Delta_t\|_F^2$. By Young's inequality, $2\langle\Delta_t,e_t\rangle_F\le\|\Delta_t\|_F^2+\|e_t\|_F^2$. Hence
\[
\tfrac{\mathrm{d}}{\mathrm{d}t}\|\Delta_t\|_F^2 \le (2\kappa_{S-t}+1)\|\Delta_t\|_F^2+\|e_t\|_F^2.
\]

Taking expectations of both sides of the pathwise inequality and using
Fubini's theorem to interchange the time derivative with the expectation
(justified by square-integrability of the solutions),
\[
\frac{\mathrm d}{\mathrm dt}\,\mathbb E\|\Delta_t\|_F^2
\;\le\;
(2\kappa_{S-t}+1)\,\mathbb E\|\Delta_t\|_F^2 + \mathbb E\|e_t\|_F^2.
\]
Define the shorthand
\[
f(t):=\mathbb E\|\Delta_t\|_F^2,
\quad
g(t):=\mathbb E\|e_t\|_F^2,
\quad
\mu(t):=2\kappa_{S-t}+1,
\quad
\Lambda(t):=\int_0^t \mu(r)\,\mathrm dr,
\]
so the inequality becomes $f'(t)\le \mu(t)\,f(t)+g(t)$. Applying the
standard integrating-factor identity for linear differential
inequalities, multiplying both sides by $e^{-\Lambda(t)}$ and using
$\Lambda'(t)=\mu(t)$,
\[
\frac{\mathrm d}{\mathrm dt}\!\left(e^{-\Lambda(t)} f(t)\right)
= e^{-\Lambda(t)}\bigl(f'(t)-\mu(t)\,f(t)\bigr)
\;\le\;
e^{-\Lambda(t)}\,g(t).
\]
Integrating this inequality from $0$ to $t$, and using $\Lambda(0)=0$,
\[
e^{-\Lambda(t)} f(t) - f(0)
\;\le\;
\int_0^t e^{-\Lambda(r)}\,g(r)\,\mathrm dr.
\]
Multiplying through by $e^{\Lambda(t)}$ and substituting back the
definitions of $f$ and $g$ gives the mean-square bound
\eqref{eq:stability_one_sided_ms}.

The coupling $(Y_0^\theta, Y_0^\star)$ at $t=0$ can be chosen as the
optimal $W_2$-coupling of $(\nu_0^\theta, \nu_0^\star)$, so that
$\mathbb E\|Y_0^\theta - Y_0^\star\|_F^2 = W_2(\nu_0^\theta, \nu_0^\star)^2$.
The synchronously-coupled processes $(Y_t^\theta, Y_t^\star)$ obtained
by solving \eqref{eq:ideal_reverse_sde_stability}--\eqref{eq:learned_reverse_sde_stability}
under the same Brownian motion then form a (not-necessarily-optimal)
coupling of $(\Law(Y_t^\theta), \Law(Y_t^\star))$ at every later
$t\in(0,S]$, hence by the definition of $W_2$ as an infimum over
couplings,
\[
W_2\bigl(\Law(Y_t^\theta), \Law(Y_t^\star)\bigr)^2
\;\le\;
\mathbb E\|Y_t^\theta - Y_t^\star\|_F^2
= f(t).
\]
Substituting \eqref{eq:stability_one_sided_ms} and taking square roots
yields the $W_2$ bound \eqref{eq:stability_one_sided_w2}.
\end{proof}

\stabilitythm*

\begin{proof}[Proof of Theorem~\ref{thm:stability_reverse_main}]
Apply Lemma~\ref{thm:appendix_stability_reverse_one_sided} with $u^\star_s=u^{L^2}_s$. Its assumptions are those of Theorem~\ref{thm:stability_reverse_main}, and its bound \eqref{eq:stability_one_sided_w2} is \eqref{eq:stability_one_sided_w2_main} with $\mathcal E_{\mathrm{rev}}(t)$ as defined there. If $\kappa_s\equiv L$, then $\Lambda(t)=(2L+1)t$, which gives the Gr\"onwall prefactor. If $\kappa_s\le\kappa_\star<-1/2$ for all $s$ in a subinterval $I$, then the integrand $2\kappa_{S-r}+1$ of $\Lambda$ is at most $2\kappa_\star+1<0$ whenever $S-r\in I$, so the factor that $I$ contributes to $e^{\Lambda(S)/2}$ is at most $e^{(2\kappa_\star+1)|I|/2}<1$.
\end{proof}

\begin{remark}[One-sided constant of the ideal reverse drift]
\label{rem:one_sided_hedge}
For each training hypergraph $H$, the conditional reverse drift
$u^\star_{s\mid H}(X)=\alpha(s)\AH(X)+\beta(s)\gamma(X-M_0)-2\tau\beta(s)\matop\bigl(C_s(H)^{-1}(\vecop(X)-m_s(H))\bigr)$
is affine in $X$ with symmetric Jacobian $J_s(H)=\alpha(s)K_H+\beta(s)\gamma I_{nm}-2\tau\beta(s)C_s(H)^{-1}$, where
$K_H:=I_m\otimes L_V(H)+L_E(H)\otimes I_n$. Since $K_H$ and $C_s(H)$ are diagonal in the same Kronecker eigenbasis,
with eigenvalues $\lambda_j+\mu_k$ and $c_{jk}(s)$, the one-sided constant of $u^\star_{s\mid H}$ is
\[
\kappa_s(H)=\max_{j,k}\Bigl[\alpha(s)(\lambda_j+\mu_k)+\beta(s)\gamma-\frac{2\tau\beta(s)}{c_{jk}(s)}\Bigr].
\]
The heat and OU terms are expansive, with $0\le\lambda_j+\mu_k\le 3$ (proof of Proposition~\ref{prop:heat_operator_main}),
and only the score term contracts; when $c_{jk}(s)\approx\tau/\gamma$ the bracket is approximately
$\alpha(s)(\lambda_j+\mu_k)-\beta(s)\gamma$. For the marginal drift $u^\star_s=\sum_i\omega_s(i\mid X)\,u^\star_{s\mid H^{(i)}}$,
the identity $\nabla_X\omega_s(i\mid X)=\omega_s(i\mid X)\bigl(r^\star_{s\mid H^{(i)}}(X)-\nabla_X\log\hat p_s(X)\bigr)$ gives
\[
\nabla u^\star_s(X)=\sum_i\omega_s(i\mid X)\,J_s(H^{(i)})+\operatorname{Cov}_{\omega_s(\cdot\mid X)}\bigl(\vecop u^\star_{s\mid H}(X),\,\vecop r^\star_{s\mid H}(X)\bigr),
\]
so $\kappa_s\le\max_i\kappa_s(H^{(i)})+\sup_X\lambda_{\max}\bigl(\operatorname{sym}\operatorname{Cov}_{\omega_s(\cdot\mid X)}(\cdot,\cdot)\bigr)$.
The covariance term measures posterior uncertainty about which training hypergraph generated $X$; it contains the positive semidefinite part
$2\tau\beta(s)\operatorname{Cov}_{\omega_s(\cdot\mid X)}(\vecop r^\star_{s\mid H},\vecop r^\star_{s\mid H})$ and can grow as $s\downarrow 0$.
Lemma~\ref{thm:appendix_stability_reverse_one_sided} assumes the one-sided bound on $u^\theta$; the same argument with the
roles of $u^\theta$ and $u^\star$ exchanged gives the bound with the one-sided constant of $u^\star$ and the drift error
evaluated along $Y^\theta$.
\end{remark}

\subsection{Euler--Maruyama discretisation error}
\label{app-sec:em-error}

\begin{corollary}[Euler--Maruyama numerical error]
\label{cor:appendix_numerical_generation_error}
Assume $u^\theta_s$ is globally Lipschitz in state uniformly in $s$, has
at most linear growth, and satisfies
\[
\|u_s^\theta(X)-u_r^\theta(X)\|_F\le L_t(1+\|X\|_F)|s-r|^{1/2}
\quad\text{for all }r,s\in[0,S],\ X\in\R^{n\times m}.
\]
Assume $\mathbb E\|Y_0^\theta\|_F^2<\infty$ and $\sqrt{\beta}$ is bounded and Lipschitz on $[0,S]$. Let
$0=t_0<t_1<\cdots<t_K=S$ be a uniform grid with $\Delta t = S/K$, and
let $\widehat Y^\theta$ be the Euler--Maruyama approximation of the
learned reverse SDE \eqref{eq:learned_reverse_sde_stability}. Then there
is $C>0$, independent of $\Delta t$, such that
\[
\mathbb E\|Y_S^\theta - \widehat Y_S^\theta\|_F^2 \le C\,\Delta t,
\qquad
W_2\bigl(\Law(Y_S^\theta),\,\Law(\widehat Y_S^\theta)\bigr)\le C^{1/2}\,\Delta t^{1/2}.
\]
\end{corollary}

\begin{proof}
Vectorisation is an isometry from Frobenius to Euclidean norm and preserves the stated Lipschitz, growth and time-regularity bounds. The diffusion coefficient $\sqrt{2\tau\beta(S-t)}$ is state-independent and Lipschitz in time; the assumed regularity holds for the cosine schedule $\beta(s)=\sin^2(\pi s/2S)$. The standard strong-order-$1/2$ Euler--Maruyama estimate \citep[Ch.~10]{kloeden1992numerical} therefore gives
\[
\mathbb E\|Y_S^\theta-\widehat Y_S^\theta\|_F^2\le C\Delta t.
\]
Here $C$ depends on the coefficient bounds, $L_t$, the initial second moment and $S$, but not on $\Delta t$. The synchronous pair $(Y_S^\theta,\widehat Y_S^\theta)$ is a coupling of their laws, so the definition of $W_2$ gives the second inequality.
\end{proof}

\subsection{Total generation-error decomposition}

\begin{corollary}[Total generation error]
\label{cor:appendix_total_generation_error}
Under the assumptions of Proposition~\ref{prop:appendix_ideal_exactness},
Lemma~\ref{thm:appendix_stability_reverse_one_sided}, and
Corollary~\ref{cor:appendix_numerical_generation_error}, with $\nu_0^\star=\hat p_S$,
\begin{equation}
W_2\bigl(\Law(\widehat Y_S^\theta),\,\hat p_{\mathrm{data}}\bigr)
\;\le\;
C^{1/2}\,\Delta t^{1/2}
\;+\;
e^{\Lambda(S)/2}\,\mathcal E_{\mathrm{rev}},
\label{eq:total_generation_error}
\end{equation}
where $\Lambda(S) = \int_0^S(2\kappa_{S-r}+1)\,\mathrm dr$ as in
Lemma~\ref{thm:appendix_stability_reverse_one_sided}, and
\[
\mathcal E_{\mathrm{rev}}^2
:= W_2(\nu_0^\theta,\nu_0^\star)^2
+ \int_0^S e^{-\Lambda(r)}\,\mathbb E\|u^\theta_{S-r}(Y_r^\star)-u^\star_{S-r}(Y_r^\star)\|_F^2\,\mathrm dr.
\]
The full-horizon statement requires all these assumptions, including the numerical regularity of the learned drift, on $[0,S]$. For $T_\varepsilon=S-\varepsilon$, if the corresponding assumptions hold on the truncated interval, then
\begin{equation}
\begin{split}
W_2\bigl(\Law(\widehat Y_{T_\varepsilon}^\theta),\hat p_{\mathrm{data}}\bigr)
\le{}& C_\varepsilon\Delta t^{1/2}
+e^{\Lambda(T_\varepsilon)/2}\mathcal E_{\mathrm{rev}}(T_\varepsilon)
+W_2(\hat p_\varepsilon,\hat p_{\mathrm{data}}),
\end{split}
\label{eq:truncated_generation_error}
\end{equation}
where $\mathcal E_{\mathrm{rev}}(t)$ is defined in Theorem~\ref{thm:stability_reverse_main} and the grid ends at $T_\varepsilon$. If $u^\theta=u^\star$ and $\nu_0^\theta=\nu_0^\star=\hat p_S$, the drift and initialisation errors vanish, but the endpoint term remains. No uniform bound on $C_\varepsilon$ as $\varepsilon\downarrow0$ is asserted.
\end{corollary}

\begin{proof}
Exactness gives $\Law(Y_S^\star)=\hat p_{\mathrm{data}}$. Insert $\Law(Y_S^\theta)$ in the $W_2$ triangle inequality and apply Corollary~\ref{cor:appendix_numerical_generation_error} and Lemma~\ref{thm:appendix_stability_reverse_one_sided} to obtain \eqref{eq:total_generation_error}. At $T_\varepsilon$, the same argument bounds the distance to $\Law(Y_{T_\varepsilon}^\star)=\hat p_\varepsilon$; a further triangle inequality gives \eqref{eq:truncated_generation_error}. Under the finite empirical data law, the forward coupling yields
\[
W_2(\hat p_\varepsilon,\hat p_{\mathrm{data}})^2
\le \mathbb E\|X_\varepsilon-H\|_F^2\longrightarrow0.
\]
The ideal drift may be singular at $s=0$, so setting $u^\theta=u^\star$ does not by itself justify the full-horizon numerical assumptions.
\end{proof}

\subsection{Equivariance of the $L^2$-optimal target}
\label{app-sec:equivariance}

Let $P\in\mathbb R^{n\times n}$ and $Q\in\mathbb R^{m\times m}$ be permutation
matrices. The action of $S_n\times S_m$ on incidence-space states is
\[
    X \longmapsto PXQ^\top ,
    \qquad X\in\mathbb R^{n\times m}.
\]
We write $\mathcal T_{P,Q}(X):=PXQ^\top$. This map is orthogonal with respect
to the Frobenius inner product and has unit absolute Jacobian determinant.

\begin{lemma}[Equivariance of the row- and column-side Laplacians]
\label{lem:lv_le_equivariance}
For any $H\in\{0,1\}^{n\times m}$,
\[
    L_V(PHQ^\top)=P L_V(H)P^\top,
    \qquad
    L_E(PHQ^\top)=Q L_E(H)Q^\top .
\]
Consequently, for the two-sided heat operator
\[
    \mathcal A_H(X):=L_V(H)X+XL_E(H),
\]
one has
\[
    \mathcal A_{PHQ^\top}(PXQ^\top)
    =
    P\mathcal A_H(X)Q^\top .
\]
\end{lemma}

\begin{proof}
The node-degree and hyperedge-size diagonal matrices transform as
\[
    D_V(PHQ^\top)=P D_V(H)P^\top,
    \qquad
    D_E(PHQ^\top)=Q D_E(H)Q^\top .
\]
Therefore,
\[
\begin{aligned}
&(PHQ^\top)D_E(PHQ^\top)^{-1}(PHQ^\top)^\top  \\
&\qquad =
PHQ^\top \bigl(QD_E(H)^{-1}Q^\top\bigr)QH^\top P^\top
=
PHD_E(H)^{-1}H^\top P^\top .
\end{aligned}
\]
Since diagonal inverse square roots commute with permutation conjugation,
\[
    D_V(PHQ^\top)^{-1/2}
    =
    P D_V(H)^{-1/2}P^\top ,
\]
which gives $L_V(PHQ^\top)=PL_V(H)P^\top$.

For the hyperedge-side operator,
\[
    (PHQ^\top)^\top(PHQ^\top)=QH^\top H Q^\top ,
\]
and the same permutation-conjugation identity applies to $D_E^{-1/2}$.
Moreover, $\operatorname{offdiag}(\cdot)$ commutes with conjugation by a
permutation matrix. Hence
\[
    A_E(PHQ^\top)=Q A_E(H)Q^\top .
\]
It follows that
\[
    D_{\mathrm{ov}}(PHQ^\top)
    =
    QD_{\mathrm{ov}}(H)Q^\top ,
\]
and therefore
\[
    L_E(PHQ^\top)=QL_E(H)Q^\top .
\]
The final claim follows by direct substitution:
\[
\begin{aligned}
\mathcal A_{PHQ^\top}(PXQ^\top)
&=
PL_V(H)P^\top PXQ^\top
+
PXQ^\top QL_E(H)Q^\top  \\
&=
P\bigl(L_V(H)X+XL_E(H)\bigr)Q^\top
=
P\mathcal A_H(X)Q^\top .
\end{aligned}
\]
\end{proof}

\begin{lemma}[Equivariance of the conditional forward and reverse laws]
\label{lem:conditional_law_equivariance}
Assume $PM_0Q^\top=M_0$, and let the matrix Brownian motion be isotropic, so
that
\[
    PW_sQ^\top \overset{d}{=} W_s .
\]
Let $X_s^{H}$ denote the solution of the forward SDE initialised at
$X_0=H$. Then
\[
    X_s^{PHQ^\top}
    \overset{d}{=}
    P X_s^{H} Q^\top .
\]
Consequently, for every $s>0$,
\[
    p_{s\mid PHQ^\top}(PXQ^\top)=p_{s\mid H}(X),
\]
and the conditional score and conditional reverse drift satisfy
\[
    r^\star_{s\mid PHQ^\top}(PXQ^\top)
    =
    P r^\star_{s\mid H}(X)Q^\top ,
\]
and
\[
    u^\star_{s\mid PHQ^\top}(PXQ^\top)
    =
    P u^\star_{s\mid H}(X)Q^\top .
\]
\end{lemma}

\begin{proof}
Define $Y_s:=PX_s^H Q^\top$. Since $P$ and $Q$ are constant matrices,
\[
    \mathrm dY_s=P\,\mathrm dX_s^H\,Q^\top .
\]
Substituting the forward SDE gives
\[
\begin{aligned}
\mathrm dY_s
&=
-\alpha(s)P\mathcal A_H(X_s^H)Q^\top\,\mathrm ds
-\beta(s)\gamma P(X_s^H-M_0)Q^\top\,\mathrm ds
+
\sqrt{2\tau\beta(s)}\,P\,\mathrm dW_s\,Q^\top .
\end{aligned}
\]
By Lemma~\ref{lem:lv_le_equivariance},
\[
    P\mathcal A_H(X_s^H)Q^\top
    =
    \mathcal A_{PHQ^\top}(PX_s^HQ^\top)
    =
    \mathcal A_{PHQ^\top}(Y_s),
\]
and by the assumed invariance of $M_0$,
\[
    P(X_s^H-M_0)Q^\top
    =
    Y_s-M_0 .
\]
Finally, isotropy of the Brownian motion gives
\[
    P\,\mathrm dW_s\,Q^\top \overset{d}{=} \mathrm dW_s .
\]
Thus $Y_s$ satisfies the same SDE in law as the process initialised at
$PHQ^\top$. Since $Y_0=PHQ^\top$, uniqueness in law yields
\[
    X_s^{PHQ^\top}
    \overset{d}{=}
    PX_s^HQ^\top .
\]

Because $\mathcal T_{P,Q}$ is an orthogonal transformation with unit absolute
Jacobian determinant, the density identity follows:
\[
    p_{s\mid PHQ^\top}(PXQ^\top)=p_{s\mid H}(X).
\]
Differentiating with respect to $X$ under the orthogonal change of variables
gives the score transformation
\[
    r^\star_{s\mid PHQ^\top}(PXQ^\top)
    =
    P r^\star_{s\mid H}(X)Q^\top .
\]
The conditional reverse drift is
\[
    u^\star_{s\mid H}(X)
    =
    \alpha(s)\mathcal A_H(X)
    +
    \beta(s)\gamma(X-M_0)
    +
    2\tau\beta(s)r^\star_{s\mid H}(X).
\]
Combining the equivariance of $\mathcal A_H$, the invariance of $M_0$, and
the score transformation gives
\[
    u^\star_{s\mid PHQ^\top}(PXQ^\top)
    =
    P u^\star_{s\mid H}(X)Q^\top .
\]
\end{proof}

\begin{lemma}[Equivariance of the \texorpdfstring{$L^2$}{L2}-optimal reverse-drift target]
\label{thm:appendix_equivariance_empirical}
Let $p_{\mathrm{data}}$ be a data law on incidence matrices that is invariant
under the action $H\mapsto PHQ^\top$ and assume that $PM_0Q^\top=M_0$, and that the
driving Brownian motion is isotropic. Let
\[
    p_s(X)
    :=
    \int p_{s\mid H}(X)\,p_{\mathrm{data}}(\mathrm dH)
\]
be the marginal forward density induced by the data law, and let
\[
    u_s^{L^2}(X)
    :=
    \mathbb E_{p_{\mathrm{data}}}
    \left[
        u^\star_{s\mid H}(X)
        \,\middle|\,
        X_s=X
    \right]
\]
be the $L^2$-optimal state-only reverse-drift target. Then
\[
    p_s(PXQ^\top)=p_s(X)
\]
and
\[
    u_s^{L^2}(PXQ^\top)
    =
    P u_s^{L^2}(X)Q^\top .
\]
\end{lemma}

\begin{proof}
First consider the marginal density. Using Lemma~\ref{lem:conditional_law_equivariance},
\[
    p_{s\mid H}(PXQ^\top)
    =
    p_{s\mid P^\top H Q}(X).
\]
Therefore,
\[
\begin{aligned}
p_s(PXQ^\top)
&=
\int p_{s\mid H}(PXQ^\top)\,p_{\mathrm{data}}(\mathrm dH) \\
&=
\int p_{s\mid P^\top H Q}(X)\,p_{\mathrm{data}}(\mathrm dH).
\end{aligned}
\]
By invariance of $p_{\mathrm{data}}$ under $H\mapsto P^\top H Q$, the
last integral equals
\[
    \int p_{s\mid H}(X)\,p_{\mathrm{data}}(\mathrm dH)
    =
    p_s(X).
\]

We now prove the equivariance of the posterior-averaged reverse drift. Define
the posterior measure over the latent training hypergraph by
\[
    \omega_s(\mathrm dH\mid X)
    :=
    \frac{
        p_{s\mid H}(X)\,p_{\mathrm{data}}(\mathrm dH)
    }{
        p_s(X)
    }.
\]
Then
\[
    u_s^{L^2}(X)
    =
    \int u^\star_{s\mid H}(X)\,\omega_s(\mathrm dH\mid X).
\]
At the transformed state $PXQ^\top$,
\[
\begin{aligned}
u_s^{L^2}(PXQ^\top)
&=
\int u^\star_{s\mid H}(PXQ^\top)\,
    \frac{p_{s\mid H}(PXQ^\top)}
         {p_s(PXQ^\top)}
    \,p_{\mathrm{data}}(\mathrm dH).
\end{aligned}
\]
Using Lemma~\ref{lem:conditional_law_equivariance} with
$\widetilde H=P^\top H Q$, we have
\[
    u^\star_{s\mid H}(PXQ^\top)
    =
    P u^\star_{s\mid P^\top H Q}(X)Q^\top ,
\]
and
\[
    p_{s\mid H}(PXQ^\top)
    =
    p_{s\mid P^\top H Q}(X).
\]
Together with $p_s(PXQ^\top)=p_s(X)$, this gives
\[
\begin{aligned}
u_s^{L^2}(PXQ^\top)
&=
P
\left[
\int
    u^\star_{s\mid P^\top H Q}(X)
    \frac{p_{s\mid P^\top H Q}(X)}
         {p_s(X)}
    \,p_{\mathrm{data}}(\mathrm dH)
\right]
Q^\top .
\end{aligned}
\]
Finally, by invariance of $p_{\mathrm{data}}$, the change of variable
$\widetilde H=P^\top H Q$ leaves the data law unchanged, so
\[
\begin{aligned}
u_s^{L^2}(PXQ^\top)
&=
P
\left[
\int
    u^\star_{s\mid \widetilde H}(X)
    \frac{p_{s\mid \widetilde H}(X)}
         {p_s(X)}
    \,p_{\mathrm{data}}(\mathrm d\widetilde H)
\right]
Q^\top  \\
&=
P u_s^{L^2}(X)Q^\top .
\end{aligned}
\]
This proves the claim.
\end{proof}

\equivariancethm*

\begin{proof}[Proof of Proposition~\ref{prop:equivariance_main}]
Lemma~\ref{lem:lv_le_equivariance} gives the transformation rules for $L_V$ and $L_E$ and the equivariance of $\AH$. A constant base mean $M_0=c\,\1_n\1_m^\top$ satisfies $PM_0Q^\top=M_0$, so Lemma~\ref{lem:conditional_law_equivariance} gives the equivariance of the conditional forward law and of the conditional reverse drift. Lemma~\ref{thm:appendix_equivariance_empirical} then gives $u_s^{L^2}(PXQ^\top)=P\,u_s^{L^2}(X)\,Q^\top$.
\end{proof}

\section{Additional Numerical Results}
\label{sec:additional_numerics}

\subsection{Real-data benchmarks}
\label{sec:datasets}

We evaluate \HEDGE\ on real hypergraph datasets drawn primarily from the
DHG-Bench/AllSet benchmark family \citep{li2026dhgbench,chien2022you},
together with congressional committee hypergraph data. These datasets were
originally introduced for hypergraph learning tasks, especially node
classification and hyperedge prediction
\citep{yadati2019hypergcn,pinder2021gaussian}, but in this
paper we use only their incidence structure. Node features, labels, and
task-specific train/validation/test splits are not used by \HEDGE.

A central difficulty in the real-data setting is that each benchmark provides
essentially a \emph{single large observed hypergraph}, rather than a collection
of many hypergraphs that could be used directly for generative training and
evaluation. We therefore construct the learning problem by sampling fixed-size
subhypergraphs from this single observed hypergraph. These sampled
subhypergraphs provide the train/test objects used by the generative benchmark.

For a hypergraph incidence matrix $H\in\{0,1\}^{n\times m}$, rows correspond
to nodes and columns correspond to hyperedges. Table~\ref{tab:real-dataset-stats}
reports basic structural properties of the full benchmark hypergraphs before
subhypergraph sampling. We include the number of nodes $n$, number of
hyperedges $m$, average hyperedge size
\[
\bar d_E
=
\frac{1}{m}\sum_{j=1}^m \sum_{i=1}^n H_{ij},
\]
average node degree
\[
\bar d_V
=
\frac{1}{n}\sum_{i=1}^n \sum_{j=1}^m H_{ij},
\]
and incidence density
\[
\rho
=
\frac{1}{nm}\sum_{i=1}^n\sum_{j=1}^m H_{ij}
=
\frac{\bar d_E}{n}
=
\frac{\bar d_V}{m}.
\]
The feature and class columns are included only to identify the original
benchmark datasets; these quantities are not used by the generative model.

\begin{table}[t]
\centering
\caption{Summary statistics of the real hypergraph datasets used in the paper.
The statistics describe the full benchmark hypergraphs before fixed-size
subhypergraph sampling. Features and classes are part of the original learning
benchmarks but are not used by HEDGE.}
\label{tab:real-dataset-stats}
\resizebox{\textwidth}{!}{
\begin{tabular}{llrrrrrrr}
\toprule
Dataset
& Source family
& Nodes $n$
& Hyperedges $m$
& Avg. edge size $\bar d_E$
& Avg. node degree $\bar d_V$
& Density $\rho$
& Features
& Classes \\
\midrule
Cora
& Cocitation
& 2,708
& 1,579
& 3.03
& 1.77
& $1.12\times 10^{-3}$
& 1,433
& 7 \\
CiteSeer
& Cocitation
& 3,312
& 1,079
& 3.20
& 1.04
& $9.66\times 10^{-4}$
& 3,703
& 6 \\
DBLP
& Coauthorship
& 41,302
& 22,363
& 4.45
& 2.41
& $1.08\times 10^{-4}$
& 1,425
& 6 \\
House-Committees
& Congressional committees
& 1,290
& 341
& 34.73
& 9.18
& $2.69\times 10^{-2}$
& 100
& 2 \\
Actor
& Heterogeneous
& 16,255
& 10,164
& 5.43
& 3.40
& $3.34\times 10^{-4}$
& 50
& 3 \\
Twitch
& Heterogeneous
& 16,812
& 2,627
& 6.23
& 0.97
& $3.71\times 10^{-4}$
& 7
& 2 \\
\bottomrule
\end{tabular}
}
\end{table}

\paragraph{Cora.}
Cora is used in its cocitation hypergraph form
\citep{yadati2019hypergcn,chien2022you}. Nodes correspond to papers, and hyperedges encode citation-induced contexts among papers. Thus, $H_{ij}=1$ indicates that paper $i$ participates in citation context $j$. This dataset is sparse, with small average hyperedge size and low average node degree, making it a useful benchmark for testing whether a generator can reproduce academic incidence patterns without over-densifying the generated hypergraphs. We emphasise that this is the cocitation Cora dataset, not the coauthorship Cora dataset.

\paragraph{CiteSeer.}
CiteSeer is another cocitation-style academic hypergraph \citep{yadati2019hypergcn,chien2022you}. Nodes are papers and hyperedges
correspond to citation-induced groups. It has a similar semantic structure to
Cora but differs in size, feature dimension, class structure, sparsity, and
hyperedge incidence profile. It therefore tests whether performance on
citation-derived hypergraphs is robust across related but distinct
bibliographic datasets. 

\paragraph{DBLP.}
DBLP is used as a coauthorship hypergraph
\citep{yadati2019hypergcn,chien2022you}. Nodes correspond to authors, and
each hyperedge corresponds to the author set of a publication. Thus,
$H_{ij}=1$ means that author $i$ appears on publication $j$. This
dataset differs from the cocitation benchmarks because hyperedges represent
explicit group interactions rather than citation-derived contexts. It is a
natural benchmark for evaluating whether a generator captures collaborative
group-size distributions, repeated author participation, and overlap between
author teams.

\paragraph{House-Committees.}
House-Committees is a congressional committee hypergraph
\citep{chodrow2021generative}. Nodes correspond to members of the United States House
of Representatives, and hyperedges correspond to committees. An incidence entry
$H_{ij}=1$ indicates that representative $i$ serves on committee $j$.
Compared with the academic datasets, House-Committees is much denser and has
substantially larger hyperedges. It therefore provides a useful stress test
for models of group membership, because the generator must reproduce both large
committee sizes and structured overlap between committees. 

\paragraph{Actor.}
Actor is a heterogeneous benchmark derived from the actor-only induced subgraph
of a larger film-director-actor-writer network \citep{tang2009social}. In the
version used here, nodes correspond to actors and hyperedges encode
co-occurrence contexts derived from the benchmark construction. In incidence
form, $H_{ij}=1$ records that actor $i$ participates in context $j$.
This dataset has a different semantic origin from citation, coauthorship, and
committee data: hyperedges arise from media co-appearance. It is therefore
useful for evaluating whether \HEDGE\ captures incidence structure beyond
academic and institutional domains. 

\paragraph{Twitch.}
Twitch is a heterogeneous benchmark derived from social-network data on the
Twitch streaming platform \citep{li2025hyper}. In the version
used here, nodes correspond to users or accounts, while hyperedges encode
shared Twitch-related contexts from the benchmark construction. In this paper,
all attributes and labels are discarded and only the binary incidence relation
is used. Twitch provides an online social-platform domain and tests whether
\HEDGE\ can capture sparse but non-academic incidence patterns with moderate
hyperedge sizes. 

\paragraph{Use in the generative benchmark.}
The original datasets are not used as independent samples from a population of
hypergraphs. Instead, each real dataset is a single observed hypergraph. To
obtain a controlled generative evaluation, we sample a set of fixed-size
subhypergraphs from the observed hypergraph, train on one subset, and compare
generated samples against held-out subhypergraphs of the same size. This
protocol controls the node and hyperedge dimensions of the incidence matrices,
so the evaluation focuses on whether a method can reproduce incidence-space
structure at a fixed scale.

\subsection{Evaluation metrics}
\label{sec:metrics}

We evaluate generated hypergraphs by comparing a batch of held-out
\emph{real} incidence matrices with a batch of \emph{generated} incidence
matrices. This is a comparison between two populations of network-valued objects; distance-based models for populations of interaction networks and multisets \citep{bolt2025modelling,bolt2022distances} give a model-based treatment of the same problem, whereas here we compare pooled structural summaries directly. Let
\[
\mathcal H_{\mathrm{real}}
=
\{H^{(r)}_1,\dots,H^{(r)}_{N_r}\},
\qquad
\mathcal H_{\mathrm{gen}}
=
\{H^{(g)}_1,\dots,H^{(g)}_{N_g}\},
\]
where each incidence matrix has fixed size $H\in\{0,1\}^{n\times m}$.
For a hypergraph $H$, write
\[
d_V(H)=H\mathbf 1_m\in\mathbb R^n,
\qquad
d_E(H)=H^\top \mathbf 1_n\in\mathbb R^m,
\]
for the node-degree and hyperedge-size vectors, respectively. Lower is better
for every reported metric except the signed gaps $\Delta\rho$, $\Delta k$ and $\Delta e$, for which values
closer to zero are better (tables rank them by absolute value).

In the main real-data comparison tables we report a compact set of
\emph{primary metrics} that target the structural properties most central to
incidence-space hypergraph generation: density calibration, mean node degree,
mean hyperedge size, distributional agreement for node degrees and hyperedge
sizes, and higher-order overlap structure. These are defined below.

\paragraph{Calibration-style summary gaps.}
For a batch $\mathcal H$, define the mean incidence density,
mean node degree, and mean hyperedge size by
\[
\rho(\mathcal H)
=
\frac{1}{|\mathcal H|\,nm}\sum_{H\in\mathcal H}\sum_{i=1}^n\sum_{j=1}^m H_{ij},
\]
\[
\bar k(\mathcal H)
=
\frac{1}{|\mathcal H|\,n}\sum_{H\in\mathcal H}\sum_{i=1}^n d_V(H)_i,
\qquad
\bar e(\mathcal H)
=
\frac{1}{|\mathcal H|\,m}\sum_{H\in\mathcal H}\sum_{j=1}^m d_E(H)_j.
\]
With $\mathcal H_{\mathrm{train}}$ the training set, we report the signed gaps
\[
\Delta\rho
=
\rho(\mathcal H_{\mathrm{gen}})-\rho(\mathcal H_{\mathrm{train}}),
\qquad
\Delta k
=
\bar k(\mathcal H_{\mathrm{gen}})-\bar k(\mathcal H_{\mathrm{train}}),
\qquad
\Delta e
=
\bar e(\mathcal H_{\mathrm{gen}})-\bar e(\mathcal H_{\mathrm{train}}),
\]
computed over all generated samples; every table uses this convention.
For fixed $n,m$, these gaps satisfy $\Delta k=m\Delta\rho$ and $\Delta e=n\Delta\rho$: they express the same density discrepancy in different units.

\paragraph{Node-degree distribution metric.}
Let
\[
D_V(\mathcal H)
=
\{d_V(H)_i : H\in\mathcal H,\ i=1,\dots,n\}
\]
denote the pooled node-degree sample across the batch. We compare real and
generated batches using the one-dimensional Wasserstein distance
\[
W_1(k)
:=
W_1\!\bigl(D_V(\mathcal H_{\mathrm{real}}),\,D_V(\mathcal H_{\mathrm{gen}})\bigr).
\]
This measures whether the generator reproduces the full distribution of node
participation counts, not only its mean.

\paragraph{Hyperedge-size distribution metric.}
Similarly, let
\[
D_E(\mathcal H)
=
\{d_E(H)_j : H\in\mathcal H,\ j=1,\dots,m\}
\]
denote the pooled hyperedge-size sample. We report
\[
W_1(e)
:=
W_1\!\bigl(D_E(\mathcal H_{\mathrm{real}}),\,D_E(\mathcal H_{\mathrm{gen}})\bigr).
\]
This measures whether the generator reproduces the distribution of group sizes,
which is one of the most basic structural signatures of a hypergraph.

\paragraph{Spectral Wasserstein distances.}
For each hypergraph $H$, let $L_V(H)$ be the node-side hypergraph
Laplacian and $L_E(H)$ the hyperedge-overlap Laplacian used by HEDGE.
Writing their eigenvalues in nondecreasing order as
\[
0\leq \lambda_1^V(H)\leq\cdots\leq \lambda_n^V(H),
\qquad
0\leq \lambda_1^E(H)\leq\cdots\leq \lambda_m^E(H),
\]
we retain the first $K$ eigenvalues from each side and pool them across the batch:
\[
\Lambda_V^K(\mathcal H)
=
\{\lambda_i^V(H):H\in\mathcal H,\ i=1,\ldots,\min(K,n)\},
\]
\[
\Lambda_E^K(\mathcal H)
=
\{\lambda_i^E(H):H\in\mathcal H,\ i=1,\ldots,\min(K,m)\}.
\]
The node spectral Wasserstein distance and edge spectral Wasserstein distance are
\[
\mathrm{Node\ spec.\ WD}
=
W_1\!\left(
\Lambda_V^K(\mathcal H_{\mathrm{real}}),
\Lambda_V^K(\mathcal H_{\mathrm{gen}})
\right),
\]
\[
\mathrm{Edge\ spec.\ WD}
=
W_1\!\left(
\Lambda_E^K(\mathcal H_{\mathrm{real}}),
\Lambda_E^K(\mathcal H_{\mathrm{gen}})
\right).
\]
The node-side spectrum probes global connectivity and diffusion geometry induced
by node--hyperedge incidence, while the edge-side spectrum probes the geometry of
hyperedge overlap. These metrics therefore assess whether generated samples match
operator-level structure, not only degree, size, or overlap marginals.

\paragraph{Overlap-tail mass.}
To probe genuinely higher-order structure, we measure how often distinct
hyperedges overlap in more than one node. For a hypergraph $H$, define the
pairwise hyperedge-overlap matrix
\[
O(H)=H^\top H,
\]
and let
\[
\mathcal I(H)=\{O_{jk}(H):1\le j<k\le m\}
\]
be the multiset of off-diagonal pairwise hyperedge intersections. We then
define the overlap-tail mass above threshold $2$ by
\[
T_2(H)
=
\frac{1}{|\mathcal I(H)|}
\sum_{z\in\mathcal I(H)} \mathbf 1\{z\ge 2\}.
\]
The reported batch-level tail-gap metric is
\[
\mathrm{Tail\ gap}
=
\left|
\frac{1}{N_r}\sum_{a=1}^{N_r} T_2(H^{(r)}_a)
-
\frac{1}{N_g}\sum_{b=1}^{N_g} T_2(H^{(g)}_b)
\right|.
\]
This statistic is sensitive to repeated multi-node overlap between distinct
hyperedges, and therefore helps distinguish generators that match only marginal
degree/size behaviour from those that also reproduce higher-order shared
membership structure.

\paragraph{Intersection-distribution Wasserstein distance.}
To compare the full distribution of pairwise hyperedge overlaps, not only its
tail mass, we again use the multiset
\[
\mathcal I(H)=\{(H^\top H)_{jk}:1\le j<k\le m\}.
\]
Pooling these values across a batch gives
\[
\mathcal I(\mathcal H)
=
\bigcup_{H\in\mathcal H}\mathcal I(H).
\]
We then define the intersection Wasserstein distance by
\[
\mathrm{Intersection\ WD}
:=
W_1\!\bigl(\mathcal I(\mathcal H_{\mathrm{real}}),\,
\mathcal I(\mathcal H_{\mathrm{gen}})\bigr).
\]
This metric is more stringent than Tail gap because it compares the full
distribution of pairwise hyperedge intersections rather than only the mass in
its upper tail.

\paragraph{Feature MMD.}
To compare broader structural signatures, we compute a feature vector
$\phi(H)\in\mathbb R^d$ for each hypergraph and then measure discrepancy
between the real and generated feature samples using maximum mean discrepancy
(MMD). Here $\phi(H)\in\mathbb R^{15}$ collects the mean, standard deviation, $90\%$ quantile and Gini
coefficient of the node degrees and of the hyperedge sizes, the density, the mean, standard deviation and
$90\%$ quantile of the pairwise hyperedge intersections, the number of triangles and the global clustering
coefficient of the clique projection, and the overlap-tail mass $T_2(H)$. Given feature sets
\[
\Phi_{\mathrm{real}}
=
\{\phi(H):H\in\mathcal H_{\mathrm{real}}\},
\qquad
\Phi_{\mathrm{gen}}
=
\{\phi(H):H\in\mathcal H_{\mathrm{gen}}\},
\]
the reported Feature MMD is the biased (V-statistic) estimate of the squared MMD with the Gaussian kernel
$k(x,y)=\exp(-\|x-y\|^2/2\sigma^2)$, where $\sigma^2$ is the median of the non-zero pairwise squared
distances in $\Phi_{\mathrm{real}}\cup\Phi_{\mathrm{gen}}$:
\[
\mathrm{Feature\ MMD}
=
\widehat{\operatorname{MMD}}{}^2(\Phi_{\mathrm{real}},\Phi_{\mathrm{gen}}),
\]
truncated at zero.
Lower values indicate closer agreement between real and generated hypergraphs
in the chosen structural feature space.

\subsection{Ablation results}
\label{sec:ablations}

Controlled synthetic ablations compare the structured heat--OU process and two-sided operator with simpler variants under matched training and sampling budgets. Per-regime tables report mean $\pm$ standard error across seeds; summary tables report mean $\pm$ standard error across the four regime means after averaging within each regime. Lower is better for all metrics (Appendix~\ref{sec:metrics}).

The simulated datasets cover four qualitatively different regimes. The \textsc{Configuration} regime emphasises heterogeneous degree and hyperedge-size marginals. The \textsc{Overlapping blocks} regime contains latent group structure
with repeated higher-order overlap. The \textsc{Committee} regime is a denser group-membership setting with broad hyperedge sizes and structured overlap. The \textsc{Sparse tail overlap} regime is sparse overall but contains a controlled minority of hyperedge pairs with nontrivial intersections. Together these settings
separate marginal calibration from the higher-order overlap and spectral structure that \HEDGE\ is designed to model.

Table~\ref{tab:simulated-hedge-ablations} averages across the four regimes. Its Full \HEDGE\ and OU only rows use the forward-process ablation (Table~\ref{tab:sim-hedge-forward-process}), while the Node only and Edge only rows use the geometry ablation (Table~\ref{tab:sim-hedge-two-sided-geometry}). For \textsc{Configuration}, the forward-process ablation and the geometry ablation used separate fits of the full and pure OU models with the same settings. The other three regimes reuse those fits across both comparisons.

\paragraph{Forward-process ablation.}
Table~\ref{tab:sim-hedge-forward-process} compares the full structured heat--OU model with a pure OU variant. The full model uses the proposed forward process: early-time noising is guided by the hypergraph heat operator, while late-time dynamics transition to the OU mechanism that yields a Gaussian base
law. The \textsc{Pure OU} variant removes the hypergraph heat term and therefore tests whether generic OU corruption, together with the same learned reverse model, is sufficient.

The heat--OU model improves Intersection WD, Feature MMD and node-spectral WD in all four regimes. It improves edge-spectral WD and hyperedge-size WD in three; pure OU is better on edge-spectral WD in the Configuration regime and on degree WD in the Configuration, Overlapping blocks and Sparse tail overlap regimes.

Tail gap measures only the mass of hyperedge pairs intersecting in at least two nodes. Pure OU has a smaller gap on Overlapping blocks and Sparse tail overlap, while the full model is better on Committee and Configuration. These exceptions motivate assessing the metrics together.

\begin{table}[t]
\centering
\scriptsize
\setlength{\tabcolsep}{2pt}
\resizebox{\textwidth}{!}{\begin{tabular}{llrrrrrrr}
\toprule
Dataset & Variant & Deg. WD & Size WD & Intersect. WD & Tail gap & Node spec. WD & Edge spec. WD & Feature MMD \\
\midrule
Committee & Full heat OU & 0.243 $\pm$ 0.033 & 0.305 $\pm$ 0.009 & 0.044 $\pm$ 0.005 & 0.013 $\pm$ 0.000 & 0.005 $\pm$ 0.001 & 0.002 $\pm$ 0.000 & 0.073 $\pm$ 0.031 \\
 & Pure OU & 0.245 $\pm$ 0.074 & 0.310 $\pm$ 0.111 & 0.104 $\pm$ 0.057 & 0.028 $\pm$ 0.012 & 0.009 $\pm$ 0.002 & 0.004 $\pm$ 0.002 & 0.100 $\pm$ 0.027 \\
\midrule
Configuration & Full heat OU & 0.267 $\pm$ 0.013 & 0.282 $\pm$ 0.010 & 0.025 $\pm$ 0.001 & 0.005 $\pm$ 0.000 & 0.022 $\pm$ 0.003 & 0.008 $\pm$ 0.000 & 0.069 $\pm$ 0.022 \\
 & Pure OU & 0.203 $\pm$ 0.005 & 0.301 $\pm$ 0.015 & 0.032 $\pm$ 0.013 & 0.011 $\pm$ 0.005 & 0.035 $\pm$ 0.002 & 0.005 $\pm$ 0.001 & 0.089 $\pm$ 0.019 \\
\midrule
Overlapping blocks & Full heat OU & 0.131 $\pm$ 0.024 & 0.168 $\pm$ 0.032 & 0.333 $\pm$ 0.013 & 0.021 $\pm$ 0.002 & 0.095 $\pm$ 0.003 & 0.124 $\pm$ 0.002 & 1.106 $\pm$ 0.110 \\
 & Pure OU & 0.127 $\pm$ 0.018 & 0.177 $\pm$ 0.024 & 0.366 $\pm$ 0.003 & 0.005 $\pm$ 0.001 & 0.102 $\pm$ 0.001 & 0.129 $\pm$ 0.001 & 1.280 $\pm$ 0.027 \\
\midrule
Sparse tail overlap & Full heat OU & 0.128 $\pm$ 0.019 & 0.164 $\pm$ 0.013 & 0.018 $\pm$ 0.002 & 0.007 $\pm$ 0.001 & 0.016 $\pm$ 0.006 & 0.012 $\pm$ 0.001 & 0.078 $\pm$ 0.020 \\
  & Pure OU & 0.112 $\pm$ 0.017 & 0.213 $\pm$ 0.019 & 0.025 $\pm$ 0.004 & 0.004 $\pm$ 0.001 & 0.031 $\pm$ 0.004 & 0.018 $\pm$ 0.004 & 0.095 $\pm$ 0.030 \\
\bottomrule
\end{tabular}}
\vspace{1em}
\caption{Simulated-data HEDGE forward process ablation. Values are mean $\pm$ standard error across seeds; lower is better.}
\label{tab:sim-hedge-forward-process}
\end{table}

\paragraph{Two-sided incidence geometry.}
Table~\ref{tab:sim-hedge-two-sided-geometry} ablates the heat operator itself.
The \textsc{Two-sided} variant uses
\[
\mathcal A_H(X)=L_V(H)X + X L_E(H),
\]
where the first term smooths across nodes and the second smooths across
overlapping hyperedges. The \textsc{Node only} variant removes the hyperedge-side
overlap term, giving $\mathcal A_H(X)=L_V(H)X$. The \textsc{Edge only} variant
removes the node-side term, giving $\mathcal A_H(X)=XL_E(H)$. The
\textsc{Pure OU} row is included as an unstructured reference.

The two-sided operator is most clearly beneficial in regimes where higher-order
overlap is central. On \textsc{Committee}, the two-sided model is best on
Intersection WD, Tail gap, node-spectrum WD, edge-spectrum WD, and hyperedge-size
WD, while remaining competitive on degree WD and Feature MMD. On
\textsc{Overlapping blocks}, it is best on hyperedge-size WD, Intersection WD,
node-spectrum WD, edge-spectrum WD, and Feature MMD. These two regimes are the
closest simulated analogues of the structural phenomena targeted by \HEDGE:
co-membership, block overlap, and repeated shared participation. In these cases,
using both node-side and hyperedge-side geometry gives the best overall
structural fidelity.

The remaining regimes show a more nuanced pattern. On \textsc{Configuration},
simpler variants can match degree, size, and some overlap statistics more
closely, which is expected because the data-generating mechanism is dominated by
marginal degree and size constraints rather than rich overlap geometry. On
\textsc{Sparse tail overlap}, the node-only and edge-only variants perform very
well on some marginal and feature metrics, while the two-sided model remains
competitive on the overlap and spectral quantities. This indicates that when
overlap structure is deliberately sparse and localised, one side of the incidence
geometry can sometimes be sufficient. The important point is that the two-sided
operator is not merely a marginal-matching device: its advantage appears most
clearly in the regimes where node participation and hyperedge overlap jointly
determine the data distribution.

\begin{table}[t]
\centering
\scriptsize
\setlength{\tabcolsep}{2pt}
\resizebox{\textwidth}{!}{\begin{tabular}{llrrrrrrr}
\toprule
Dataset & Variant & Deg. WD & Size WD & Intersect. WD & Tail gap & Node spec. WD & Edge spec. WD & Feature MMD \\
\midrule
Committee & Edge only & 0.364 $\pm$ 0.067 & 0.435 $\pm$ 0.108 & 0.107 $\pm$ 0.046 & 0.024 $\pm$ 0.012 & 0.010 $\pm$ 0.002 & 0.003 $\pm$ 0.001 & 0.201 $\pm$ 0.053 \\
 & Node only & 0.212 $\pm$ 0.052 & 0.351 $\pm$ 0.060 & 0.117 $\pm$ 0.022 & 0.036 $\pm$ 0.007 & 0.008 $\pm$ 0.002 & 0.005 $\pm$ 0.001 & 0.053 $\pm$ 0.017 \\
 & Pure OU & 0.245 $\pm$ 0.074 & 0.310 $\pm$ 0.111 & 0.104 $\pm$ 0.057 & 0.028 $\pm$ 0.012 & 0.009 $\pm$ 0.002 & 0.004 $\pm$ 0.002 & 0.100 $\pm$ 0.027 \\
 & Two-sided & 0.243 $\pm$ 0.033 & 0.305 $\pm$ 0.009 & 0.044 $\pm$ 0.005 & 0.013 $\pm$ 0.000 & 0.005 $\pm$ 0.001 & 0.002 $\pm$ 0.000 & 0.073 $\pm$ 0.031 \\
\midrule
Configuration & Edge only & 0.173 $\pm$ 0.012 & 0.210 $\pm$ 0.008 & 0.024 $\pm$ 0.008 & 0.006 $\pm$ 0.002 & 0.016 $\pm$ 0.003 & 0.007 $\pm$ 0.001 & 0.029 $\pm$ 0.006 \\
 & Node only & 0.166 $\pm$ 0.016 & 0.246 $\pm$ 0.011 & 0.031 $\pm$ 0.006 & 0.010 $\pm$ 0.003 & 0.025 $\pm$ 0.002 & 0.005 $\pm$ 0.001 & 0.053 $\pm$ 0.026 \\
 & Pure OU & 0.149 $\pm$ 0.020 & 0.226 $\pm$ 0.023 & 0.023 $\pm$ 0.002 & 0.005 $\pm$ 0.002 & 0.026 $\pm$ 0.002 & 0.003 $\pm$ 0.000 & 0.044 $\pm$ 0.019 \\
 & Two-sided & 0.194 $\pm$ 0.009 & 0.248 $\pm$ 0.016 & 0.043 $\pm$ 0.015 & 0.014 $\pm$ 0.005 & 0.018 $\pm$ 0.006 & 0.008 $\pm$ 0.003 & 0.028 $\pm$ 0.008 \\
\midrule
Overlapping blocks & Edge only & 0.208 $\pm$ 0.034 & 0.237 $\pm$ 0.054 & 0.355 $\pm$ 0.005 & 0.018 $\pm$ 0.004 & 0.098 $\pm$ 0.001 & 0.128 $\pm$ 0.000 & 1.201 $\pm$ 0.076 \\
 & Node only & 0.193 $\pm$ 0.061 & 0.209 $\pm$ 0.061 & 0.385 $\pm$ 0.013 & 0.022 $\pm$ 0.007 & 0.097 $\pm$ 0.000 & 0.127 $\pm$ 0.001 & 1.285 $\pm$ 0.046 \\
 & Pure OU & 0.127 $\pm$ 0.018 & 0.177 $\pm$ 0.024 & 0.366 $\pm$ 0.003 & 0.005 $\pm$ 0.001 & 0.102 $\pm$ 0.001 & 0.129 $\pm$ 0.001 & 1.280 $\pm$ 0.027 \\
 & Two-sided & 0.131 $\pm$ 0.024 & 0.168 $\pm$ 0.032 & 0.333 $\pm$ 0.013 & 0.021 $\pm$ 0.002 & 0.095 $\pm$ 0.003 & 0.124 $\pm$ 0.002 & 1.106 $\pm$ 0.110 \\
\midrule
Sparse tail overlap & Edge only & 0.075 $\pm$ 0.004 & 0.144 $\pm$ 0.015 & 0.019 $\pm$ 0.002 & 0.007 $\pm$ 0.002 & 0.011 $\pm$ 0.007 & 0.012 $\pm$ 0.002 & 0.062 $\pm$ 0.029 \\
 & Node only & 0.067 $\pm$ 0.011 & 0.143 $\pm$ 0.007 & 0.014 $\pm$ 0.002 & 0.007 $\pm$ 0.001 & 0.013 $\pm$ 0.003 & 0.009 $\pm$ 0.001 & 0.023 $\pm$ 0.007 \\
 & Pure OU & 0.112 $\pm$ 0.017 & 0.213 $\pm$ 0.019 & 0.025 $\pm$ 0.004 & 0.004 $\pm$ 0.001 & 0.031 $\pm$ 0.004 & 0.018 $\pm$ 0.004 & 0.095 $\pm$ 0.030 \\
 & Two-sided & 0.128 $\pm$ 0.019 & 0.164 $\pm$ 0.013 & 0.018 $\pm$ 0.002 & 0.007 $\pm$ 0.001 & 0.016 $\pm$ 0.006 & 0.012 $\pm$ 0.001 & 0.078 $\pm$ 0.020 \\
\bottomrule
\end{tabular}}
\vspace{1em}
\caption{Simulated-data HEDGE two-sided geometry ablation. Values are mean $\pm$ standard error across seeds; lower is better.}
\label{tab:sim-hedge-two-sided-geometry}
\end{table}

\begin{table}[t]
\centering
\scriptsize
\setlength{\tabcolsep}{3pt}
\begin{tabular}{llrrrrr}
\toprule
Ablation & Variant & Intersect. WD & Tail gap & Node spec. WD & Edge spec. WD & Feature MMD \\
\midrule
Equivariance
& equivariant
& \textbf{0.212 $\pm$ 0.071}
& \textbf{0.037 $\pm$ 0.015}
& \textbf{0.046 $\pm$ 0.019}
& \textit{0.047 $\pm$ 0.028}
& \textbf{0.534 $\pm$ 0.277} \\
& equivariant + relabel
& \textit{0.229 $\pm$ 0.091}
& \textit{0.046 $\pm$ 0.015}
& \textit{0.047 $\pm$ 0.019}
& \textbf{0.045 $\pm$ 0.029}
& \textit{0.610 $\pm$ 0.252} \\
& vectorized MLP
& 5.964 $\pm$ 0.082
& 0.715 $\pm$ 0.128
& 0.301 $\pm$ 0.069
& 0.159 $\pm$ 0.049
& 0.902 $\pm$ 0.060 \\
& vectorized MLP + relabel
& 5.977 $\pm$ 0.068
& 0.715 $\pm$ 0.128
& 0.302 $\pm$ 0.069
& 0.159 $\pm$ 0.049
& 0.901 $\pm$ 0.060 \\
\midrule
Projection
& $\vartheta=0.3$
& \textit{0.105 $\pm$ 0.076}
& \textit{0.012 $\pm$ 0.004}
& \textbf{0.034 $\pm$ 0.020}
& \textbf{0.036 $\pm$ 0.029}
& \textbf{0.331 $\pm$ 0.258} \\
& $\vartheta=0.4$
& \textit{0.105 $\pm$ 0.076}
& \textit{0.012 $\pm$ 0.004}
& \textbf{0.034 $\pm$ 0.020}
& \textbf{0.036 $\pm$ 0.029}
& \textbf{0.331 $\pm$ 0.258} \\
& $\vartheta=0.5$
& \textit{0.105 $\pm$ 0.076}
& \textit{0.012 $\pm$ 0.004}
& \textbf{0.034 $\pm$ 0.020}
& \textbf{0.036 $\pm$ 0.029}
& \textbf{0.331 $\pm$ 0.258} \\
& $\vartheta=0.6$
& \textit{0.105 $\pm$ 0.076}
& \textit{0.012 $\pm$ 0.004}
& \textbf{0.034 $\pm$ 0.020}
& \textbf{0.036 $\pm$ 0.029}
& \textbf{0.331 $\pm$ 0.258} \\
& $\vartheta=0.7$
& \textit{0.105 $\pm$ 0.076}
& \textit{0.012 $\pm$ 0.004}
& \textbf{0.034 $\pm$ 0.020}
& \textbf{0.036 $\pm$ 0.029}
& \textbf{0.331 $\pm$ 0.258} \\
& density-matched
& \textbf{0.103 $\pm$ 0.073}
& \textbf{0.007 $\pm$ 0.001}
& \textit{0.035 $\pm$ 0.020}
& \textbf{0.036 $\pm$ 0.029}
& \textit{0.411 $\pm$ 0.305} \\
\midrule
$\tau$
& $\tau=0.05$
& 1.804 $\pm$ 1.491
& 0.101 $\pm$ 0.036
& 0.065 $\pm$ 0.014
& 0.061 $\pm$ 0.025
& 0.621 $\pm$ 0.085 \\
& $\tau=0.1$
& 1.607 $\pm$ 1.365
& 0.110 $\pm$ 0.074
& 0.079 $\pm$ 0.027
& 0.062 $\pm$ 0.025
& 0.805 $\pm$ 0.146 \\
& $\tau=0.2$
& 0.225 $\pm$ 0.074
& 0.045 $\pm$ 0.016
& 0.047 $\pm$ 0.019
& \textit{0.044 $\pm$ 0.027}
& \textit{0.461 $\pm$ 0.259} \\
& $\tau=0.25$
& 0.230 $\pm$ 0.076
& 0.048 $\pm$ 0.016
& \textbf{0.043 $\pm$ 0.021}
& 0.045 $\pm$ 0.029
& \textbf{0.425 $\pm$ 0.204} \\
& $\tau=0.5$
& \textit{0.210 $\pm$ 0.075}
& \textit{0.042 $\pm$ 0.014}
& 0.048 $\pm$ 0.021
& \textit{0.044 $\pm$ 0.029}
& 0.570 $\pm$ 0.262 \\
& $\tau=1.0$
& \textbf{0.156 $\pm$ 0.074}
& \textbf{0.027 $\pm$ 0.006}
& \textit{0.044 $\pm$ 0.021}
& \textbf{0.042 $\pm$ 0.030}
& 0.533 $\pm$ 0.235 \\
\bottomrule
\end{tabular}
\caption{Stage-2 simulated-data \HEDGE\ ablations. Values are mean $\pm$ standard error after first averaging each variant within each of the four synthetic datasets, with five seeds per dataset. Lower is better for all columns. Bold and italics mark the best and second-best variant within each ablation block; ties after rounding are marked accordingly.}
\label{tab:simulated-hedge-stage2-ablations}
\end{table}

\paragraph{Additional implementation ablations.}
Table~\ref{tab:simulated-hedge-stage2-ablations} reports additional simulated-data
ablations testing three implementation and model-design choices: the equivariant
reverse-drift architecture, the final projection rule, and the diffusion
temperature \(\tau\). These experiments use the same four synthetic regimes as the
main ablation study and report mean \(\pm\) standard error after first averaging
within each synthetic dataset; the standard error is across the four regime means.

\paragraph{Equivariance.}
The equivariance ablation compares the \(S_n\times S_m\)-equivariant architecture
used by \HEDGE\ with a non-equivariant vectorized MLP baseline. The difference is
large across all reported metrics. The equivariant model is best on Intersection
WD, Tail gap, Node spectral WD, and Feature MMD, while the relabelled equivariant
variant is best on Edge spectral WD and second-best on the remaining structural
metrics. By contrast, the vectorized MLP variants perform substantially worse,
especially on Intersection WD and Tail gap. This supports the use of a
permutation-equivariant reverse field: treating the incidence matrix as an
ordered vector loses the row--column relabelling symmetry that is intrinsic to
hypergraphs. The relabel rows should be interpreted as a robustness check rather
than a strong stress test, since most reported metrics are themselves
permutation-invariant.

\paragraph{Projection threshold.}
The projection ablation tests sensitivity to the final binarisation map
\(\Pi_\vartheta\) (each variant applies the post-processing rule of Algorithm~\ref{alg:post} after thresholding; the threshold-only decoding on real data is studied in Appendix~\ref{app:projection}). Fixed thresholds between \(0.3\) and \(0.7\) give nearly identical
structural metrics, which is consistent with the relaxed samples being close to
binary at the end of reverse sampling. Density-matched projection improves
Intersection WD and Tail gap, but gives a worse Feature MMD. Thus thresholding is
not driving the main conclusions: the structural metrics are stable across a
wide threshold range, while density matching mainly trades marginal calibration
against broader feature-space agreement.

\paragraph{Diffusion temperature.}
The \(\tau\)-ablation shows that very small temperatures are harmful. The
variants \(\tau=0.05\) and \(\tau=0.1\) have much worse intersection, tail, and
spectral discrepancies, indicating that insufficient stochasticity makes the
reverse-generation problem poorly calibrated. Intermediate and larger values are
substantially better. In this sweep, \(\tau=1.0\) gives the best Intersection WD,
Tail gap, and Edge spectral WD, while \(\tau=0.25\) gives the best Node spectral
WD and Feature MMD. The overall conclusion is that \HEDGE\ is not overly
sensitive within a reasonable temperature range, but too little noise degrades
higher-order structural recovery.

\subsection{Real-data table of results}
\label{sec:real-data-extra-results}

Table~\ref{tab:realdata_full_main} reports every real-data metric with uncertainty: the signed gaps $\Delta\rho$, $\Delta e$, and $\Delta k$ of generated incidence density, mean hyperedge size, and mean node degree relative to the \emph{training} set, the hyperedge-size and node-degree Wasserstein distances $W_1(e)$ and $W_1(k)$, the higher-order metrics of the main text, and the node- and edge-side spectral Wasserstein distances. Entries are mean $\pm$ sample standard deviation over three seeds. Rows are \HEDGE\ with the full projection, \HEDGE\ with thresholding only (grey; a diagnostic, not ranked), HCM-MCMC, HERGM (the maximum-entropy model of \citet{saracco2025entropy}), ER-HG (with the corrected decoder; Appendix~\ref{app:projection}), and the contextual reference discrepancy between the training and held-out sets (grey; not ranked; Appendix~\ref{app:hcm}). Table~\ref{tab:main-higher-order-uncertainty} repeats the main-text metrics with uncertainty. Lower is better for distances; signed gaps are ranked by absolute magnitude.

The overall pattern is as follows. HCM-MCMC is the strongest method on hyperedge-size distribution matching $W_1(e)$ (four of six datasets) and closer than \HEDGE\ on the mean-size and mean-degree gaps (four of six), which is consistent with its exact preservation of the marginals of the subhypergraph each chain starts from, whereas \HEDGE\ is better on node-degree matching $W_1(k)$ on four of six. \HEDGE\ is the most consistently strong method on the metrics that probe higher-order incidence structure: it is best on all six datasets for Intersection WD (tied with HCM-MCMC on \textsc{Cora}) and for the node- and edge-spectral WDs, and on four of six for Feature MMD and Overlap Tail Gap, where HCM-MCMC and ER-HG respectively are best on the other two. The threshold-only row shows which of these conclusions depend on the post-processing rule (Appendix~\ref{app:projection}). HERGM and ER-HG are weaker than \HEDGE\ on Intersection WD and the spectral metrics on every dataset, and their large calibration gaps are produced by the degree rule rather than by the models (Table~\ref{tab:null_model_stages}); the reference row shows that the degree- and size-distribution distances $W_1(k)$ and $W_1(e)$ of the generative methods are of the same order as those between the training and held-out sets themselves.

\begin{table*}[t]
\centering
\scriptsize
\setlength{\tabcolsep}{3pt}
\caption{Real-data comparison across the six main datasets, all metrics. Entries are mean {\tiny$\pm$} sample standard deviation over three seeds. Lower is better for distances; signed gaps are ranked by absolute magnitude. Among the ranked rows (\HEDGE, HCM-MCMC, HERGM, ER-HG) the best value is in \textbf{bold} and the second-best in \textit{italics}; grey rows are unranked diagnostics. $\Delta$ gaps are generated minus training set; the reference row's $\Delta$ gaps are zero by construction (they compare the training set with itself) and only its distance columns are informative. H.-Comm.\ = House Committees.}
\label{tab:realdata_full_main}
\fitwidth{\begin{tabular}{@{}p{1.35cm}lcccccc@{}}
\toprule
Metric & Method & Cora & CiteSeer & Actor & H.-Comm. & DBLP & Twitch \\
\midrule
\multirow{6}{=}{\makecell[l]{$\Delta\rho$}}
 & \HEDGE & \textit{0.001 {\tiny$\pm$} 0.001} & \textbf{0.001 {\tiny$\pm$} 0.001} & \textbf{-0.000 {\tiny$\pm$} 0.001} & \textit{-0.007 {\tiny$\pm$} 0.005} & \textbf{-0.001 {\tiny$\pm$} 0.003} & \textbf{-0.001 {\tiny$\pm$} 0.001} \\
 & {\color{gray}\HEDGE\ (thr.)} & -0.001 {\tiny$\pm$} 0.002 & -0.003 {\tiny$\pm$} 0.001 & -0.002 {\tiny$\pm$} 0.001 & -0.009 {\tiny$\pm$} 0.005 & -0.006 {\tiny$\pm$} 0.003 & -0.003 {\tiny$\pm$} 0.001 \\
 & \texttt{HCM-MCMC} & \textbf{-0.000 {\tiny$\pm$} 0.001} & \textbf{0.001 {\tiny$\pm$} 0.003} & \textbf{-0.000 {\tiny$\pm$} 0.002} & \textbf{-0.003 {\tiny$\pm$} 0.007} & \textit{0.002 {\tiny$\pm$} 0.006} & \textbf{0.001 {\tiny$\pm$} 0.002} \\
 & \texttt{HERGM} & 0.018 {\tiny$\pm$} 0.001 & 0.018 {\tiny$\pm$} 0.002 & 0.025 {\tiny$\pm$} 0.002 & 0.013 {\tiny$\pm$} 0.005 & 0.024 {\tiny$\pm$} 0.003 & 0.025 {\tiny$\pm$} 0.001 \\
 & \texttt{ER-HG} & 0.017 {\tiny$\pm$} 0.001 & \textit{0.016 {\tiny$\pm$} 0.002} & \textit{0.023 {\tiny$\pm$} 0.002} & 0.012 {\tiny$\pm$} 0.004 & 0.021 {\tiny$\pm$} 0.002 & \textit{0.021 {\tiny$\pm$} 0.002} \\
 & {\color{gray}Reference (train vs.\ held-out)} & 0.000 {\tiny$\pm$} 0.000 & 0.000 {\tiny$\pm$} 0.000 & 0.000 {\tiny$\pm$} 0.000 & 0.000 {\tiny$\pm$} 0.000 & 0.000 {\tiny$\pm$} 0.000 & 0.000 {\tiny$\pm$} 0.000 \\
\midrule
\multirow{6}{=}{\makecell[l]{$\Delta e$}}
 & \HEDGE & \textit{0.073 {\tiny$\pm$} 0.059} & \textbf{0.053 {\tiny$\pm$} 0.077} & \textit{-0.022 {\tiny$\pm$} 0.048} & \textit{-0.459 {\tiny$\pm$} 0.297} & \textbf{-0.044 {\tiny$\pm$} 0.187} & \textit{-0.054 {\tiny$\pm$} 0.081} \\
 & {\color{gray}\HEDGE\ (thr.)} & -0.088 {\tiny$\pm$} 0.116 & -0.197 {\tiny$\pm$} 0.092 & -0.126 {\tiny$\pm$} 0.075 & -0.563 {\tiny$\pm$} 0.319 & -0.400 {\tiny$\pm$} 0.220 & -0.171 {\tiny$\pm$} 0.036 \\
 & \texttt{HCM-MCMC} & \textbf{-0.013 {\tiny$\pm$} 0.038} & \textit{0.066 {\tiny$\pm$} 0.173} & \textbf{-0.006 {\tiny$\pm$} 0.113} & \textbf{-0.206 {\tiny$\pm$} 0.478} & \textit{0.146 {\tiny$\pm$} 0.382} & \textbf{0.035 {\tiny$\pm$} 0.143} \\
 & \texttt{HERGM} & 1.160 {\tiny$\pm$} 0.076 & 1.148 {\tiny$\pm$} 0.108 & 1.624 {\tiny$\pm$} 0.145 & 0.856 {\tiny$\pm$} 0.296 & 1.542 {\tiny$\pm$} 0.165 & 1.579 {\tiny$\pm$} 0.069 \\
 & \texttt{ER-HG} & 1.060 {\tiny$\pm$} 0.080 & 0.996 {\tiny$\pm$} 0.097 & 1.442 {\tiny$\pm$} 0.156 & 0.775 {\tiny$\pm$} 0.245 & 1.313 {\tiny$\pm$} 0.131 & 1.368 {\tiny$\pm$} 0.147 \\
 & {\color{gray}Reference (train vs.\ held-out)} & 0.000 {\tiny$\pm$} 0.000 & 0.000 {\tiny$\pm$} 0.000 & 0.000 {\tiny$\pm$} 0.000 & 0.000 {\tiny$\pm$} 0.000 & 0.000 {\tiny$\pm$} 0.000 & 0.000 {\tiny$\pm$} 0.000 \\
\midrule
\multirow{6}{=}{\makecell[l]{$\Delta k$}}
 & \HEDGE & \textit{0.029 {\tiny$\pm$} 0.023} & \textbf{0.021 {\tiny$\pm$} 0.030} & \textit{-0.005 {\tiny$\pm$} 0.012} & \textit{-0.115 {\tiny$\pm$} 0.074} & \textbf{-0.012 {\tiny$\pm$} 0.052} & \textit{-0.013 {\tiny$\pm$} 0.020} \\
 & {\color{gray}\HEDGE\ (thr.)} & -0.035 {\tiny$\pm$} 0.045 & -0.077 {\tiny$\pm$} 0.036 & -0.031 {\tiny$\pm$} 0.019 & -0.141 {\tiny$\pm$} 0.080 & -0.113 {\tiny$\pm$} 0.062 & -0.043 {\tiny$\pm$} 0.009 \\
 & \texttt{HCM-MCMC} & \textbf{-0.005 {\tiny$\pm$} 0.015} & \textit{0.026 {\tiny$\pm$} 0.068} & \textbf{-0.002 {\tiny$\pm$} 0.028} & \textbf{-0.052 {\tiny$\pm$} 0.120} & \textit{0.041 {\tiny$\pm$} 0.108} & \textbf{0.009 {\tiny$\pm$} 0.036} \\
 & \texttt{HERGM} & 0.453 {\tiny$\pm$} 0.030 & 0.448 {\tiny$\pm$} 0.042 & 0.406 {\tiny$\pm$} 0.036 & 0.214 {\tiny$\pm$} 0.074 & 0.434 {\tiny$\pm$} 0.046 & 0.395 {\tiny$\pm$} 0.017 \\
 & \texttt{ER-HG} & 0.414 {\tiny$\pm$} 0.031 & 0.389 {\tiny$\pm$} 0.038 & 0.360 {\tiny$\pm$} 0.039 & 0.194 {\tiny$\pm$} 0.061 & 0.369 {\tiny$\pm$} 0.037 & 0.342 {\tiny$\pm$} 0.037 \\
 & {\color{gray}Reference (train vs.\ held-out)} & 0.000 {\tiny$\pm$} 0.000 & 0.000 {\tiny$\pm$} 0.000 & 0.000 {\tiny$\pm$} 0.000 & 0.000 {\tiny$\pm$} 0.000 & 0.000 {\tiny$\pm$} 0.000 & 0.000 {\tiny$\pm$} 0.000 \\
\midrule
\multirow{6}{=}{\makecell[l]{$W_1(e)$}}
 & \HEDGE & \textit{0.143 {\tiny$\pm$} 0.055} & \textit{0.352 {\tiny$\pm$} 0.042} & \textbf{0.304 {\tiny$\pm$} 0.038} & \textbf{0.638 {\tiny$\pm$} 0.222} & \textit{0.967 {\tiny$\pm$} 0.161} & \textit{0.362 {\tiny$\pm$} 0.051} \\
 & {\color{gray}\HEDGE\ (thr.)} & 0.225 {\tiny$\pm$} 0.022 & 0.428 {\tiny$\pm$} 0.015 & 0.358 {\tiny$\pm$} 0.073 & 0.721 {\tiny$\pm$} 0.253 & 1.015 {\tiny$\pm$} 0.122 & 0.406 {\tiny$\pm$} 0.036 \\
 & \texttt{HCM-MCMC} & \textbf{0.073 {\tiny$\pm$} 0.063} & \textbf{0.262 {\tiny$\pm$} 0.166} & \textit{0.317 {\tiny$\pm$} 0.154} & \textit{0.771 {\tiny$\pm$} 0.166} & \textbf{0.744 {\tiny$\pm$} 0.462} & \textbf{0.284 {\tiny$\pm$} 0.039} \\
 & \texttt{HERGM} & 1.157 {\tiny$\pm$} 0.103 & 1.243 {\tiny$\pm$} 0.080 & 1.637 {\tiny$\pm$} 0.217 & 1.031 {\tiny$\pm$} 0.488 & 1.759 {\tiny$\pm$} 0.093 & 1.453 {\tiny$\pm$} 0.051 \\
 & \texttt{ER-HG} & 1.057 {\tiny$\pm$} 0.112 & 1.135 {\tiny$\pm$} 0.095 & 1.421 {\tiny$\pm$} 0.420 & 1.638 {\tiny$\pm$} 0.172 & 1.870 {\tiny$\pm$} 0.338 & 1.294 {\tiny$\pm$} 0.077 \\
 & {\color{gray}Reference (train vs.\ held-out)} & 0.052 {\tiny$\pm$} 0.025 & 0.225 {\tiny$\pm$} 0.113 & 0.308 {\tiny$\pm$} 0.083 & 0.539 {\tiny$\pm$} 0.056 & 0.489 {\tiny$\pm$} 0.069 & 0.302 {\tiny$\pm$} 0.063 \\
\midrule
\multirow{6}{=}{\makecell[l]{$W_1(k)$}}
 & \HEDGE & \textit{0.057 {\tiny$\pm$} 0.023} & \textbf{0.094 {\tiny$\pm$} 0.031} & \textbf{0.036 {\tiny$\pm$} 0.022} & \textbf{0.128 {\tiny$\pm$} 0.033} & \textbf{0.143 {\tiny$\pm$} 0.096} & \textit{0.102 {\tiny$\pm$} 0.034} \\
 & {\color{gray}\HEDGE\ (thr.)} & 0.072 {\tiny$\pm$} 0.011 & 0.079 {\tiny$\pm$} 0.019 & 0.040 {\tiny$\pm$} 0.012 & 0.141 {\tiny$\pm$} 0.048 & 0.117 {\tiny$\pm$} 0.108 & 0.092 {\tiny$\pm$} 0.038 \\
 & \texttt{HCM-MCMC} & \textbf{0.055 {\tiny$\pm$} 0.014} & \textit{0.113 {\tiny$\pm$} 0.071} & \textit{0.038 {\tiny$\pm$} 0.027} & \textit{0.148 {\tiny$\pm$} 0.039} & \textit{0.201 {\tiny$\pm$} 0.162} & \textbf{0.061 {\tiny$\pm$} 0.027} \\
 & \texttt{HERGM} & 0.452 {\tiny$\pm$} 0.040 & 0.486 {\tiny$\pm$} 0.031 & 0.409 {\tiny$\pm$} 0.054 & 0.246 {\tiny$\pm$} 0.097 & 0.495 {\tiny$\pm$} 0.026 & 0.363 {\tiny$\pm$} 0.013 \\
 & \texttt{ER-HG} & 0.417 {\tiny$\pm$} 0.045 & 0.430 {\tiny$\pm$} 0.030 & 0.355 {\tiny$\pm$} 0.105 & 0.231 {\tiny$\pm$} 0.086 & 0.401 {\tiny$\pm$} 0.118 & 0.311 {\tiny$\pm$} 0.027 \\
 & {\color{gray}Reference (train vs.\ held-out)} & 0.046 {\tiny$\pm$} 0.023 & 0.093 {\tiny$\pm$} 0.023 & 0.048 {\tiny$\pm$} 0.008 & 0.111 {\tiny$\pm$} 0.026 & 0.121 {\tiny$\pm$} 0.066 & 0.077 {\tiny$\pm$} 0.016 \\
\midrule
\multirow{6}{=}{\makecell[l]{Overlap\\Tail Gap}}
 & \HEDGE & \textbf{0.005 {\tiny$\pm$} 0.001} & 0.012 {\tiny$\pm$} 0.004 & \textbf{0.003 {\tiny$\pm$} 0.002} & \textbf{0.020 {\tiny$\pm$} 0.017} & \textit{0.013 {\tiny$\pm$} 0.014} & \textbf{0.010 {\tiny$\pm$} 0.001} \\
 & {\color{gray}\HEDGE\ (thr.)} & 0.006 {\tiny$\pm$} 0.001 & 0.012 {\tiny$\pm$} 0.004 & 0.003 {\tiny$\pm$} 0.002 & 0.021 {\tiny$\pm$} 0.016 & 0.013 {\tiny$\pm$} 0.014 & 0.010 {\tiny$\pm$} 0.001 \\
 & \texttt{HCM-MCMC} & \textit{0.007 {\tiny$\pm$} 0.002} & 0.013 {\tiny$\pm$} 0.004 & \textit{0.005 {\tiny$\pm$} 0.003} & \textit{0.040 {\tiny$\pm$} 0.040} & 0.022 {\tiny$\pm$} 0.004 & \textit{0.014 {\tiny$\pm$} 0.007} \\
 & \texttt{HERGM} & 0.010 {\tiny$\pm$} 0.005 & \textit{0.008 {\tiny$\pm$} 0.005} & 0.046 {\tiny$\pm$} 0.012 & 0.160 {\tiny$\pm$} 0.027 & 0.047 {\tiny$\pm$} 0.013 & 0.043 {\tiny$\pm$} 0.002 \\
 & \texttt{ER-HG} & \textbf{0.005 {\tiny$\pm$} 0.005} & \textbf{0.004 {\tiny$\pm$} 0.004} & 0.032 {\tiny$\pm$} 0.018 & 0.109 {\tiny$\pm$} 0.043 & \textbf{0.012 {\tiny$\pm$} 0.018} & 0.024 {\tiny$\pm$} 0.004 \\
 & {\color{gray}Reference (train vs.\ held-out)} & 0.001 {\tiny$\pm$} 0.001 & 0.006 {\tiny$\pm$} 0.002 & 0.004 {\tiny$\pm$} 0.001 & 0.008 {\tiny$\pm$} 0.007 & 0.008 {\tiny$\pm$} 0.002 & 0.005 {\tiny$\pm$} 0.005 \\
\midrule
\multirow{6}{=}{\makecell[l]{Intersection\\WD}}
 & \HEDGE & \textbf{0.020 {\tiny$\pm$} 0.013} & \textbf{0.050 {\tiny$\pm$} 0.023} & \textbf{0.038 {\tiny$\pm$} 0.016} & \textbf{0.225 {\tiny$\pm$} 0.056} & \textbf{0.072 {\tiny$\pm$} 0.029} & \textbf{0.070 {\tiny$\pm$} 0.013} \\
 & {\color{gray}\HEDGE\ (thr.)} & 0.016 {\tiny$\pm$} 0.009 & 0.037 {\tiny$\pm$} 0.025 & 0.032 {\tiny$\pm$} 0.019 & 0.214 {\tiny$\pm$} 0.056 & 0.064 {\tiny$\pm$} 0.039 & 0.066 {\tiny$\pm$} 0.014 \\
 & \texttt{HCM-MCMC} & \textbf{0.020 {\tiny$\pm$} 0.005} & \textit{0.071 {\tiny$\pm$} 0.018} & \textit{0.051 {\tiny$\pm$} 0.034} & \textit{0.367 {\tiny$\pm$} 0.020} & \textit{0.139 {\tiny$\pm$} 0.067} & \textit{0.122 {\tiny$\pm$} 0.028} \\
 & \texttt{HERGM} & 0.163 {\tiny$\pm$} 0.011 & 0.199 {\tiny$\pm$} 0.023 & 0.353 {\tiny$\pm$} 0.028 & 0.623 {\tiny$\pm$} 0.085 & 0.340 {\tiny$\pm$} 0.053 & 0.369 {\tiny$\pm$} 0.009 \\
 & \texttt{ER-HG} & \textit{0.133 {\tiny$\pm$} 0.014} & 0.162 {\tiny$\pm$} 0.013 & 0.312 {\tiny$\pm$} 0.049 & 0.607 {\tiny$\pm$} 0.074 & 0.288 {\tiny$\pm$} 0.049 & 0.330 {\tiny$\pm$} 0.031 \\
 & {\color{gray}Reference (train vs.\ held-out)} & 0.010 {\tiny$\pm$} 0.003 & 0.024 {\tiny$\pm$} 0.010 & 0.023 {\tiny$\pm$} 0.007 & 0.136 {\tiny$\pm$} 0.058 & 0.059 {\tiny$\pm$} 0.013 & 0.033 {\tiny$\pm$} 0.014 \\
\midrule
\multirow{6}{=}{\makecell[l]{Node\\spec.\ WD}}
 & \HEDGE & \textbf{0.004 {\tiny$\pm$} 0.004} & \textbf{0.021 {\tiny$\pm$} 0.010} & \textbf{0.039 {\tiny$\pm$} 0.011} & \textbf{0.068 {\tiny$\pm$} 0.027} & \textbf{0.071 {\tiny$\pm$} 0.029} & \textbf{0.068 {\tiny$\pm$} 0.009} \\
 & {\color{gray}\HEDGE\ (thr.)} & 0.001 {\tiny$\pm$} 0.001 & 0.012 {\tiny$\pm$} 0.003 & 0.025 {\tiny$\pm$} 0.015 & 0.064 {\tiny$\pm$} 0.028 & 0.063 {\tiny$\pm$} 0.038 & 0.063 {\tiny$\pm$} 0.015 \\
 & \texttt{HCM-MCMC} & \textit{0.015 {\tiny$\pm$} 0.012} & \textit{0.070 {\tiny$\pm$} 0.052} & \textit{0.056 {\tiny$\pm$} 0.036} & \textit{0.115 {\tiny$\pm$} 0.018} & \textit{0.121 {\tiny$\pm$} 0.044} & \textit{0.142 {\tiny$\pm$} 0.032} \\
 & \texttt{HERGM} & 0.215 {\tiny$\pm$} 0.007 & 0.241 {\tiny$\pm$} 0.015 & 0.297 {\tiny$\pm$} 0.022 & 0.145 {\tiny$\pm$} 0.031 & 0.292 {\tiny$\pm$} 0.009 & 0.280 {\tiny$\pm$} 0.005 \\
 & \texttt{ER-HG} & 0.202 {\tiny$\pm$} 0.014 & 0.219 {\tiny$\pm$} 0.008 & 0.267 {\tiny$\pm$} 0.042 & 0.158 {\tiny$\pm$} 0.036 & 0.279 {\tiny$\pm$} 0.030 & 0.265 {\tiny$\pm$} 0.014 \\
 & {\color{gray}Reference (train vs.\ held-out)} & 0.002 {\tiny$\pm$} 0.001 & 0.008 {\tiny$\pm$} 0.004 & 0.031 {\tiny$\pm$} 0.002 & 0.028 {\tiny$\pm$} 0.007 & 0.028 {\tiny$\pm$} 0.014 & 0.036 {\tiny$\pm$} 0.024 \\
\midrule
\multirow{6}{=}{\makecell[l]{Edge\\spec.\ WD}}
 & \HEDGE & \textbf{0.038 {\tiny$\pm$} 0.015} & \textbf{0.058 {\tiny$\pm$} 0.008} & \textbf{0.068 {\tiny$\pm$} 0.050} & \textbf{0.088 {\tiny$\pm$} 0.018} & \textbf{0.065 {\tiny$\pm$} 0.030} & \textbf{0.085 {\tiny$\pm$} 0.054} \\
 & {\color{gray}\HEDGE\ (thr.)} & 0.068 {\tiny$\pm$} 0.004 & 0.067 {\tiny$\pm$} 0.017 & 0.034 {\tiny$\pm$} 0.011 & 0.089 {\tiny$\pm$} 0.011 & 0.126 {\tiny$\pm$} 0.086 & 0.122 {\tiny$\pm$} 0.066 \\
 & \texttt{HCM-MCMC} & \textit{0.085 {\tiny$\pm$} 0.025} & \textit{0.090 {\tiny$\pm$} 0.044} & \textit{0.101 {\tiny$\pm$} 0.104} & \textit{0.113 {\tiny$\pm$} 0.004} & \textit{0.100 {\tiny$\pm$} 0.038} & \textit{0.190 {\tiny$\pm$} 0.069} \\
 & \texttt{HERGM} & 0.196 {\tiny$\pm$} 0.008 & 0.192 {\tiny$\pm$} 0.009 & 0.315 {\tiny$\pm$} 0.054 & 0.183 {\tiny$\pm$} 0.031 & 0.234 {\tiny$\pm$} 0.027 & 0.258 {\tiny$\pm$} 0.019 \\
 & \texttt{ER-HG} & 0.210 {\tiny$\pm$} 0.010 & 0.201 {\tiny$\pm$} 0.008 & 0.329 {\tiny$\pm$} 0.056 & 0.185 {\tiny$\pm$} 0.024 & 0.239 {\tiny$\pm$} 0.014 & 0.271 {\tiny$\pm$} 0.017 \\
 & {\color{gray}Reference (train vs.\ held-out)} & 0.039 {\tiny$\pm$} 0.020 & 0.054 {\tiny$\pm$} 0.050 & 0.075 {\tiny$\pm$} 0.043 & 0.047 {\tiny$\pm$} 0.009 & 0.079 {\tiny$\pm$} 0.038 & 0.078 {\tiny$\pm$} 0.016 \\
\midrule
\multirow{6}{=}{\makecell[l]{Feature\\MMD}}
 & \HEDGE & \textbf{0.103 {\tiny$\pm$} 0.022} & \textbf{0.039 {\tiny$\pm$} 0.027} & \textbf{0.113 {\tiny$\pm$} 0.095} & \textbf{0.069 {\tiny$\pm$} 0.055} & \textit{0.251 {\tiny$\pm$} 0.101} & \textit{0.108 {\tiny$\pm$} 0.064} \\
 & {\color{gray}\HEDGE\ (thr.)} & 0.214 {\tiny$\pm$} 0.133 & 0.089 {\tiny$\pm$} 0.086 & 0.107 {\tiny$\pm$} 0.096 & 0.058 {\tiny$\pm$} 0.040 & 0.300 {\tiny$\pm$} 0.140 & 0.196 {\tiny$\pm$} 0.207 \\
 & \texttt{HCM-MCMC} & \textit{0.133 {\tiny$\pm$} 0.068} & \textit{0.095 {\tiny$\pm$} 0.063} & \textit{0.309 {\tiny$\pm$} 0.318} & \textit{0.138 {\tiny$\pm$} 0.133} & \textbf{0.120 {\tiny$\pm$} 0.168} & \textbf{0.075 {\tiny$\pm$} 0.024} \\
 & \texttt{HERGM} & 1.071 {\tiny$\pm$} 0.076 & 0.648 {\tiny$\pm$} 0.193 & 1.017 {\tiny$\pm$} 0.198 & 0.441 {\tiny$\pm$} 0.233 & 0.345 {\tiny$\pm$} 0.180 & 0.818 {\tiny$\pm$} 0.260 \\
 & \texttt{ER-HG} & 1.106 {\tiny$\pm$} 0.193 & 0.639 {\tiny$\pm$} 0.193 & 1.059 {\tiny$\pm$} 0.135 & 0.428 {\tiny$\pm$} 0.407 & 0.295 {\tiny$\pm$} 0.139 & 0.748 {\tiny$\pm$} 0.326 \\
 & {\color{gray}Reference (train vs.\ held-out)} & 0.054 {\tiny$\pm$} 0.010 & 0.072 {\tiny$\pm$} 0.053 & 0.117 {\tiny$\pm$} 0.064 & 0.117 {\tiny$\pm$} 0.077 & 0.151 {\tiny$\pm$} 0.022 & 0.119 {\tiny$\pm$} 0.115 \\
\bottomrule
\end{tabular}
}
\end{table*}

\begin{table*}[t]
\centering
\scriptsize
\setlength{\tabcolsep}{3.5pt}
\caption{Higher-order real-data comparison with uncertainty (the metrics of the main text). Entries are mean $\pm$ sample standard deviation over three seeds; lower is better; ranking marks as in Table~\ref{tab:realdata_full_main}. H.-Comm. = House Committees.}
\label{tab:main-higher-order-uncertainty}
\begin{tabular}{@{}p{1.35cm}lrrrrrr@{}}
\toprule
Metric & Method & Cora & CiteSeer & Actor & H.-Comm. & DBLP & Twitch \\
\midrule
\multirow{5}{=}{\makecell[l]{Overlap\\Tail Gap}}
 & \HEDGE & \textbf{0.005 {\tiny$\pm$} 0.001} & 0.012 {\tiny$\pm$} 0.004 & \textbf{0.003 {\tiny$\pm$} 0.002} & \textbf{0.020 {\tiny$\pm$} 0.017} & \textit{0.013 {\tiny$\pm$} 0.014} & \textbf{0.010 {\tiny$\pm$} 0.001} \\
 & {\color{gray}\HEDGE\ (thr.)} & 0.006 {\tiny$\pm$} 0.001 & 0.012 {\tiny$\pm$} 0.004 & 0.003 {\tiny$\pm$} 0.002 & 0.021 {\tiny$\pm$} 0.016 & 0.013 {\tiny$\pm$} 0.014 & 0.010 {\tiny$\pm$} 0.001 \\
 & \texttt{HCM-MCMC} & \textit{0.007 {\tiny$\pm$} 0.002} & 0.013 {\tiny$\pm$} 0.004 & \textit{0.005 {\tiny$\pm$} 0.003} & \textit{0.040 {\tiny$\pm$} 0.040} & 0.022 {\tiny$\pm$} 0.004 & \textit{0.014 {\tiny$\pm$} 0.007} \\
 & \texttt{HERGM} & 0.010 {\tiny$\pm$} 0.005 & \textit{0.008 {\tiny$\pm$} 0.005} & 0.046 {\tiny$\pm$} 0.012 & 0.160 {\tiny$\pm$} 0.027 & 0.047 {\tiny$\pm$} 0.013 & 0.043 {\tiny$\pm$} 0.002 \\
 & \texttt{ER-HG} & \textbf{0.005 {\tiny$\pm$} 0.005} & \textbf{0.004 {\tiny$\pm$} 0.004} & 0.032 {\tiny$\pm$} 0.018 & 0.109 {\tiny$\pm$} 0.043 & \textbf{0.012 {\tiny$\pm$} 0.018} & 0.024 {\tiny$\pm$} 0.004 \\
\midrule
\multirow{5}{=}{\makecell[l]{Intersection\\WD}}
 & \HEDGE & \textbf{0.020 {\tiny$\pm$} 0.013} & \textbf{0.050 {\tiny$\pm$} 0.023} & \textbf{0.038 {\tiny$\pm$} 0.016} & \textbf{0.225 {\tiny$\pm$} 0.056} & \textbf{0.072 {\tiny$\pm$} 0.029} & \textbf{0.070 {\tiny$\pm$} 0.013} \\
 & {\color{gray}\HEDGE\ (thr.)} & 0.016 {\tiny$\pm$} 0.009 & 0.037 {\tiny$\pm$} 0.025 & 0.032 {\tiny$\pm$} 0.019 & 0.214 {\tiny$\pm$} 0.056 & 0.064 {\tiny$\pm$} 0.039 & 0.066 {\tiny$\pm$} 0.014 \\
 & \texttt{HCM-MCMC} & \textbf{0.020 {\tiny$\pm$} 0.005} & \textit{0.071 {\tiny$\pm$} 0.018} & \textit{0.051 {\tiny$\pm$} 0.034} & \textit{0.367 {\tiny$\pm$} 0.020} & \textit{0.139 {\tiny$\pm$} 0.067} & \textit{0.122 {\tiny$\pm$} 0.028} \\
 & \texttt{HERGM} & 0.163 {\tiny$\pm$} 0.011 & 0.199 {\tiny$\pm$} 0.023 & 0.353 {\tiny$\pm$} 0.028 & 0.623 {\tiny$\pm$} 0.085 & 0.340 {\tiny$\pm$} 0.053 & 0.369 {\tiny$\pm$} 0.009 \\
 & \texttt{ER-HG} & \textit{0.133 {\tiny$\pm$} 0.014} & 0.162 {\tiny$\pm$} 0.013 & 0.312 {\tiny$\pm$} 0.049 & 0.607 {\tiny$\pm$} 0.074 & 0.288 {\tiny$\pm$} 0.049 & 0.330 {\tiny$\pm$} 0.031 \\
\midrule
\multirow{5}{=}{\makecell[l]{Feature\\MMD}}
 & \HEDGE & \textbf{0.103 {\tiny$\pm$} 0.022} & \textbf{0.039 {\tiny$\pm$} 0.027} & \textbf{0.113 {\tiny$\pm$} 0.095} & \textbf{0.069 {\tiny$\pm$} 0.055} & \textit{0.251 {\tiny$\pm$} 0.101} & \textit{0.108 {\tiny$\pm$} 0.064} \\
 & {\color{gray}\HEDGE\ (thr.)} & 0.214 {\tiny$\pm$} 0.133 & 0.089 {\tiny$\pm$} 0.086 & 0.107 {\tiny$\pm$} 0.096 & 0.058 {\tiny$\pm$} 0.040 & 0.300 {\tiny$\pm$} 0.140 & 0.196 {\tiny$\pm$} 0.207 \\
 & \texttt{HCM-MCMC} & \textit{0.133 {\tiny$\pm$} 0.068} & \textit{0.095 {\tiny$\pm$} 0.063} & \textit{0.309 {\tiny$\pm$} 0.318} & \textit{0.138 {\tiny$\pm$} 0.133} & \textbf{0.120 {\tiny$\pm$} 0.168} & \textbf{0.075 {\tiny$\pm$} 0.024} \\
 & \texttt{HERGM} & 1.071 {\tiny$\pm$} 0.076 & 0.648 {\tiny$\pm$} 0.193 & 1.017 {\tiny$\pm$} 0.198 & 0.441 {\tiny$\pm$} 0.233 & 0.345 {\tiny$\pm$} 0.180 & 0.818 {\tiny$\pm$} 0.260 \\
 & \texttt{ER-HG} & 1.106 {\tiny$\pm$} 0.193 & 0.639 {\tiny$\pm$} 0.193 & 1.059 {\tiny$\pm$} 0.135 & 0.428 {\tiny$\pm$} 0.407 & 0.295 {\tiny$\pm$} 0.139 & 0.748 {\tiny$\pm$} 0.326 \\
\bottomrule
\end{tabular}

\end{table*}

\subsection{Bipartite Representations for Real-data Examples}
\label{sec:bipartite-realdata}

Figures~\ref{fig:bipartite-actor}--\ref{fig:bipartite-house-committees} provide a qualitative view of representative held-out real subhypergraphs and matched generated samples
through their \emph{bipartite representation}. A hypergraph incidence matrix $H\in\{0,1\}^{n\times m}$ may be viewed equivalently as a two-mode graph with one node set corresponding to vertices and the other corresponding to hyperedges: a left-side node $v_i$ and a right-side node $e_j$ are connected if and only if $H_{ij}=1$. Thus each line in the bipartite plot represents exactly one node--hyperedge incidence. The plots are therefore visually equivalent to the incidence-matrix view, but often easier to read in terms of participation patterns, degree heterogeneity, and overlap structure.

In each figure, rows correspond to three representative held-out test subhypergraphs chosen by incidence-count tier (low, median, and high), while the generated examples are matched to the corresponding real example by nearest incidence count. This controls for overall size, so that the visual comparison focuses on \emph{how} incidences are arranged rather than simply on \emph{how many} incidences are present. In particular, these plots make it possible to see whether a method reproduces (i) the overall sparsity level, (ii) variation in node participation and hyperedge sizes, and (iii) repeated concentration of incidences that induces nontrivial overlap between hyperedges. 

Because the vertices and hyperedges of every set of subhypergraphs is relabelled uniformly at random, the drawings carry no positional structure, and generated samples are matched to the real example by incidence count, so the panels show the scale and sparsity of the samples rather than differences between methods; those are measured in Tables~\ref{tab:realdata_full_main} and~\ref{tab:main-higher-order-uncertainty}.
These figures should therefore be read as illustrations of the scale and sparsity of the generated samples, not as evidence for the main empirical claim, which rests on the numerical evaluation.

\begin{figure}
    \centering
    \includegraphics[width=0.9\linewidth]{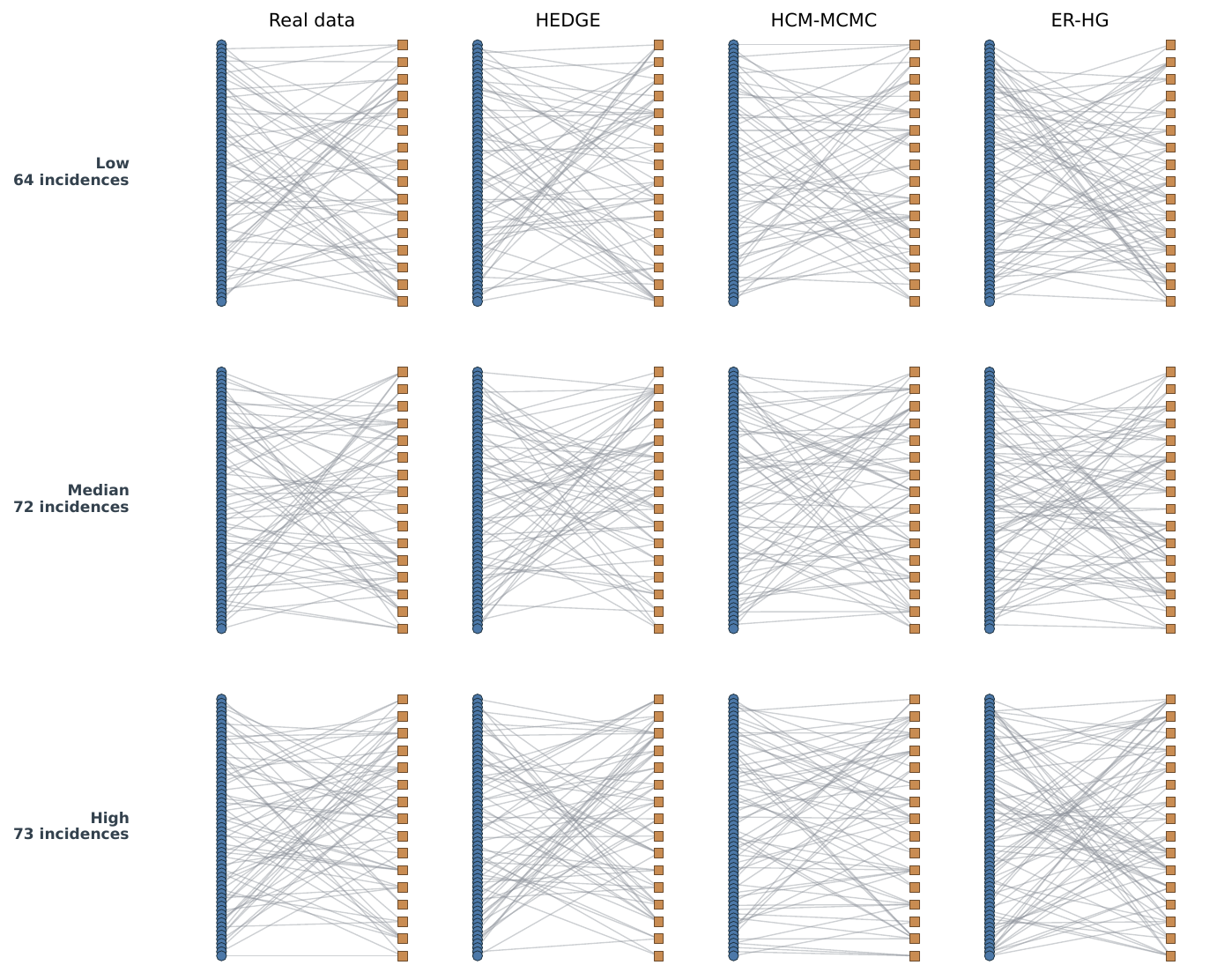}
    \caption{Actor dataset: Bipartite visualisation of held-out subhypergraph incidence matrices and matched generated samples. Each row shows a representative real test subhypergraph selected by incidence-count tier: low, median, or high, where an incidence is one node-hyperedge membership, shown as a line in the bipartite plot. Generated samples in the same row are size-matched to the real example by nearest incidence count. }
    \label{fig:bipartite-actor}
\end{figure}

\begin{figure}
    \centering
    \includegraphics[width=0.9\linewidth]{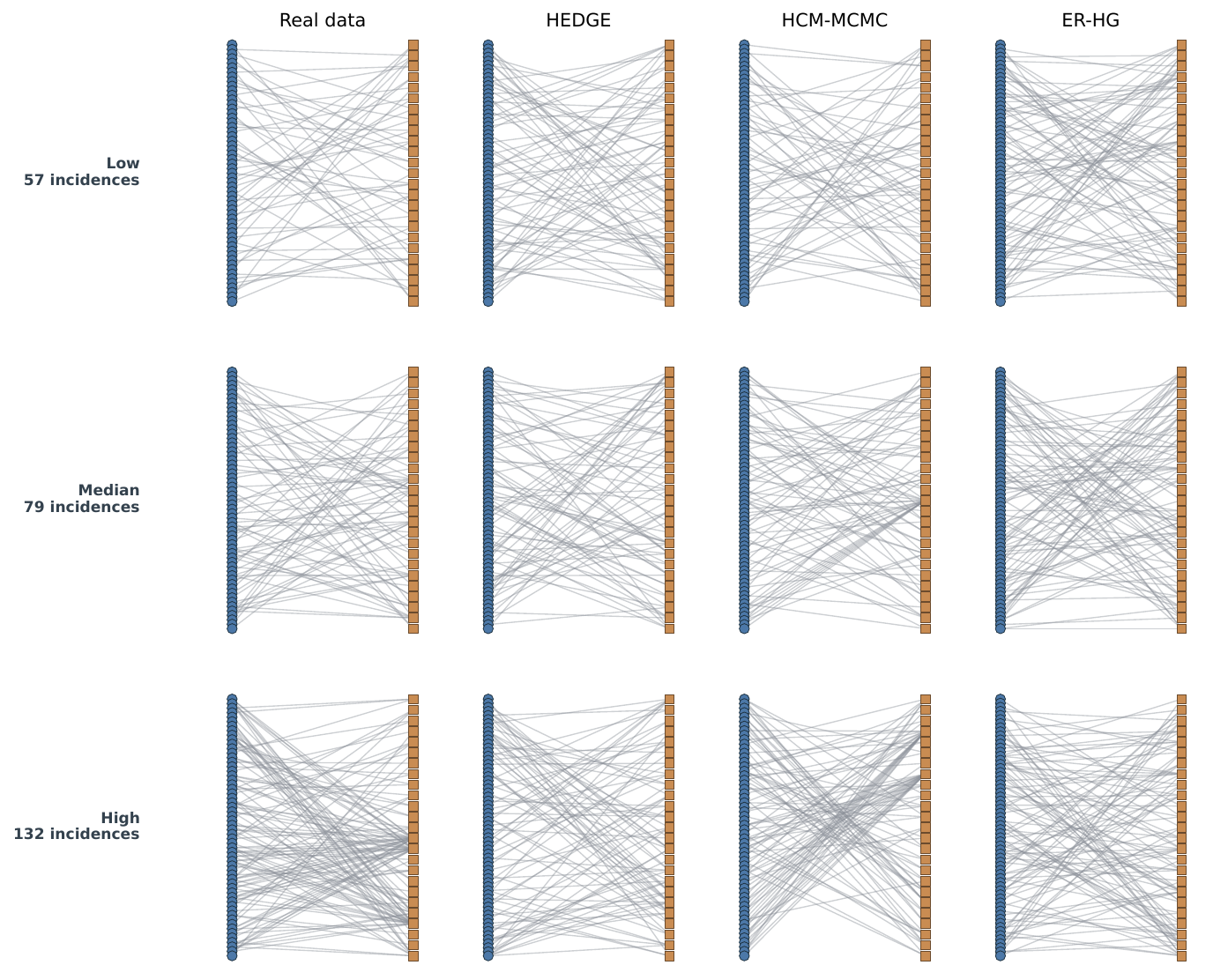}
    \caption{Citeseer dataset: Bipartite visualisation of held-out subhypergraph incidence matrices and matched generated samples. Each row shows a representative real test subhypergraph selected by incidence-count tier: low, median, or high, where an incidence is one node-hyperedge membership, shown as a line in the bipartite plot. Generated samples in the same row are size-matched to the real example by nearest incidence count.}
    \label{fig:bipartite-citeseer}
\end{figure}

\begin{figure}
    \centering
    \includegraphics[width=0.9\linewidth]{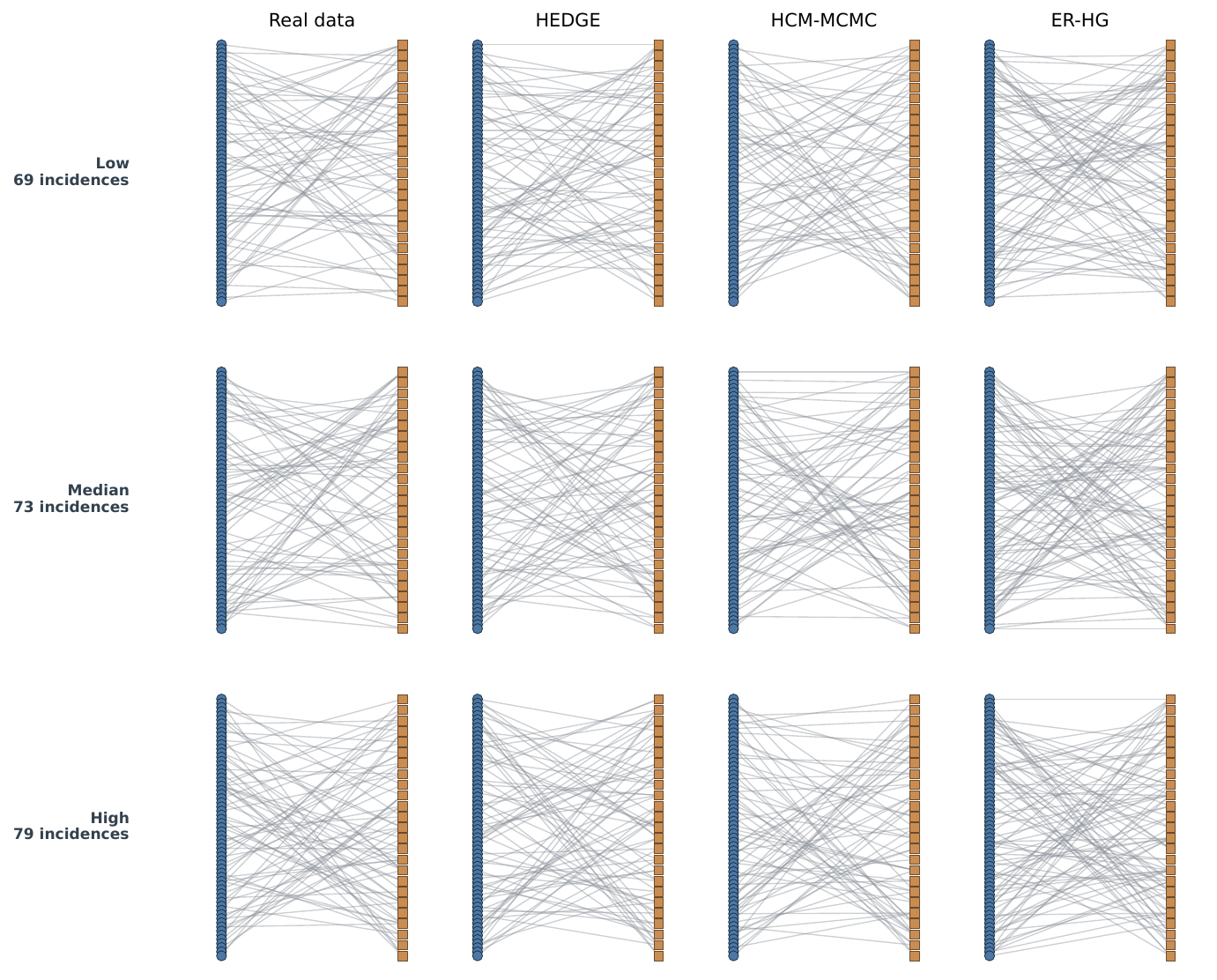}
    \caption{Cora dataset: Bipartite visualisation of held-out subhypergraph incidence matrices and matched generated samples. Each row shows a representative real test subhypergraph selected by incidence-count tier: low, median, or high, where an incidence is one node-hyperedge membership, shown as a line in the bipartite plot. Generated samples in the same row are size-matched to the real example by nearest incidence count.}
    \label{fig:bipartite-cora}
\end{figure}

\begin{figure}
    \centering
    \includegraphics[width=0.9\linewidth]{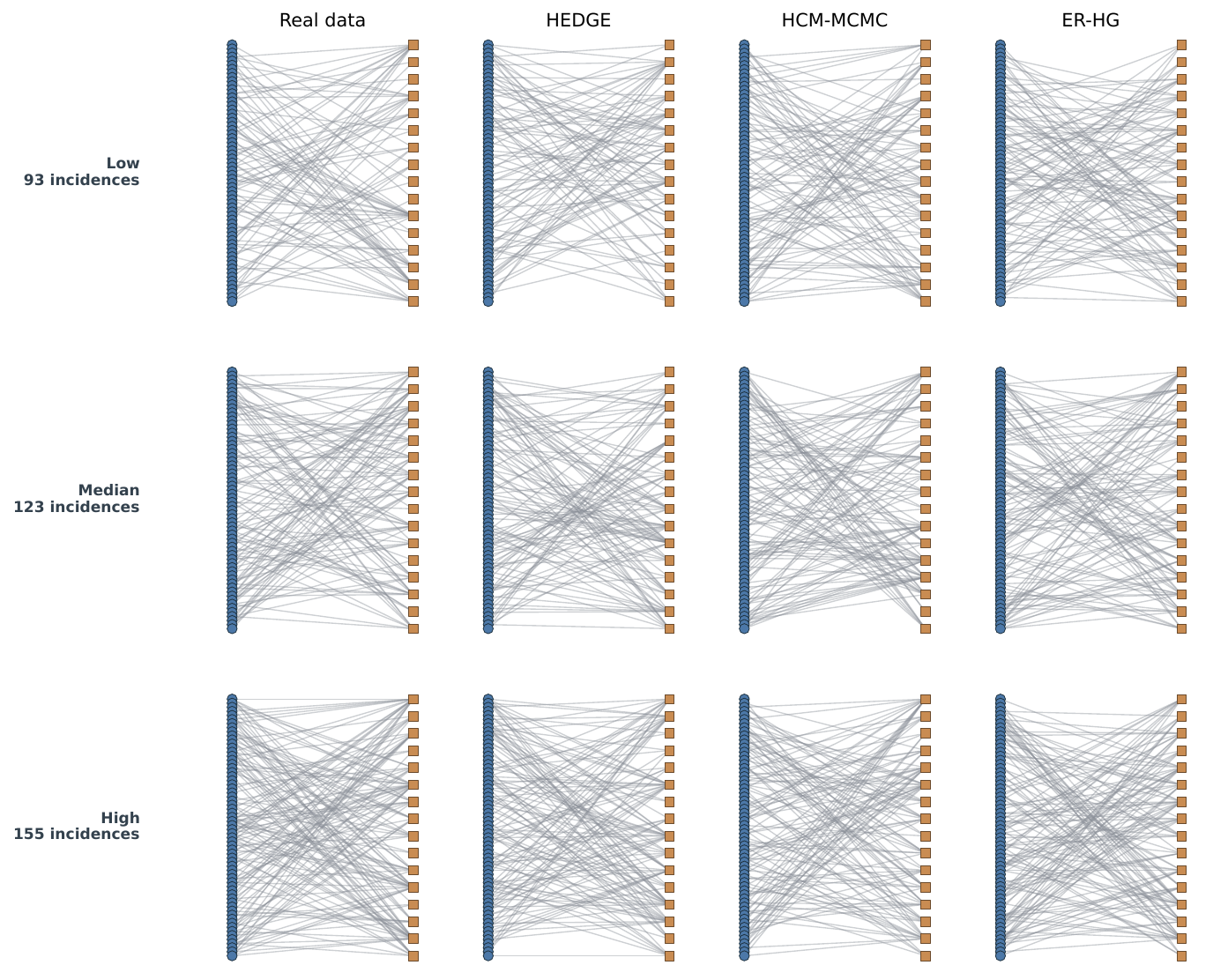}
    \caption{House committees dataset: Bipartite visualisation of held-out subhypergraph incidence matrices and matched generated samples. Each row shows a representative real test subhypergraph selected by incidence-count tier: low, median, or high, where an incidence is one node-hyperedge membership, shown as a line in the bipartite plot. Generated samples in the same row are size-matched to the real example by nearest incidence count.}
    \label{fig:bipartite-house-committees}
\end{figure}

\section{Implementation Details}
\label{app:implementation}

\subsection{Conditional moments as computed}
\label{app:moments}
All reported runs compute the conditional moments of Proposition~\ref{prop:conditional_gaussian_main} in incidence coordinates. The mean equation \eqref{eq:mean_ode_main} is integrated by fixed-step RK4 (four substeps from $s=0$ to the sampled time) using the dense $nm\times nm$ operator $I_m\otimes L_V(H)+L_E(H)\otimes I_n$, which is rebuilt at every call; the covariance equation \eqref{eq:cov_ode_main} is replaced by its diagonal in incidence coordinates, $\frac{\mathrm d}{\mathrm ds}v_s=-2\bigl(\alpha(s)\,\mathrm{diag}(K)+\beta(s)\gamma\bigr)v_s+2\tau\beta(s)$ with $K=I_m\otimes L_V+L_E\otimes I_n$, so that the conditional sample is $X_s=\matop(m_s+\sqrt{v_s}\odot\xi)$ and the score target is $-(\vecop(X)-m_s)/v_s$ entrywise. This is a diagonal approximation of the conditional covariance; the mean is unapproximated in state coordinates but numerically integrated. Variances are floored at $10^{-10}$ and capped at $10^4$, score targets are clipped to $\pm 10^2$, and non-finite values are zeroed. The cost is $O((nm)^2)$ time and memory per training example and step (Appendix~\ref{app:compute}).

The separable alternative described in Section~\ref{subsec:method_conditional_gaussian} integrates the $nm$ scalar modal equations $\frac{\mathrm d}{\mathrm ds}\tilde m_{jk}=-r_{jk}(s)\tilde m_{jk}+\beta(s)\gamma\tilde M_{0,jk}$ and $\frac{\mathrm d}{\mathrm ds}c_{jk}=-2r_{jk}(s)c_{jk}+2\tau\beta(s)$, with $r_{jk}(s)=\alpha(s)(\lambda_j+\mu_k)+\beta(s)\gamma$, by fixed-step RK4 in the Kronecker eigenbasis, and maps back with $X=U\tilde XV^\top$; it retains the full conditional covariance and yields the closed-form Gaussian score conditional on the integrated moments. We validated it against a dense full-covariance RK4 solver on random $3\times3$ and $6\times4$ hypergraphs (relative errors $3\times10^{-16}$ for the mean and $3\times10^{-15}$ for the covariance). It is used for the bounds of Appendix~\ref{app:terminal_law} and the cost comparison of Appendix~\ref{app:compute}; the reported generation results were not retrained with it.

\subsection{Post-processing rule}
\label{app:post}
\begin{algorithm}[h]
\caption{Minimum-incidence post-processing $\mathrm{Post}(H;R)$ (addition-only; deterministic given $R$)}
\label{alg:post}
\begin{algorithmic}[1]
\STATE {\bfseries Input:} binary $H\in\{0,1\}^{n\times m}$, score matrix $R\in\R^{n\times m}$, minimum hyperedge size $k$ ($=2$). For \HEDGE, $H=\mathbf 1\{Y\ge\vartheta\}$ and $R=Y$ with $\vartheta=0.55$; for ER-HG, $H=\mathbf 1\{U<p\}$ and $R=1-U$.
\FOR{each hyperedge $e$ with $\sum_v H_{ve}<k$}
    \STATE set $H_{ve}\leftarrow 1$ for the $k-\sum_v H_{ve}$ highest-scored entries of $R_{\cdot e}$ among those with $H_{ve}=0$ \hfill (size rule)
\ENDFOR
\FOR{each node $v$ with $\sum_e H_{ve}=0$}
    \STATE set $H_{ve}\leftarrow 1$ for $e=\arg\max_{e'} R_{ve'}$ \hfill (degree rule)
\ENDFOR
\STATE \textbf{return} $H$
\end{algorithmic}
\end{algorithm}
The rule never removes an incidence, uses only the generated matrix and its scores, and is independent of held-out data and of other methods. The size rule enforces a property of the benchmark data support (every hyperedge of every training subhypergraph has size at least two); the degree rule enforces a property the training subhypergraphs do not have. The ER-HG run uses the Bernoulli uniforms as scores on the identical draws (Appendix~\ref{app:projection}). HCM-MCMC preserves the degree and size sequences of its starting subhypergraph and receives no post-processing.

\subsection{Hyperparameters and baseline settings}
See Table \ref{tab:hyperparameters} for the hyperparameters used in all reported real-data runs. The table also describes the settings of the baseline methods, which are not tuned to the datasets.

\label{app:hyperparameters}
\begin{table}[h]
\centering
\scriptsize
\begin{tabular}{p{2.3cm}p{10.6cm}}
\toprule
Forward process & horizon $S=3$, OU rate $\gamma=1$, noise scale $\tau=0.25$, cosine schedule $\beta(s)=\sin^2(\pi s/2S)$, $M_0=\bar\rho\,\1_n\1_m^\top$ at the training-set density \\
Moments & dense-operator RK4 mean with 4 substeps, diagonal covariance (Appendix~\ref{app:moments}) \\
Reverse field & $S_n\times S_m$-equivariant rectangular network, hidden width 64, 4 blocks, time embedding 128 \\
Training & reverse-drift regression; time law uniform on $[0,S]$; Adam, learning rate $10^{-3}$, batch 8, $50$ epochs $\times$ $100$ steps; gradient clipping at norm 1; EMA 0.999 (used for sampling) \\
Loss weighting & time-only smooth min-SNR weighting $w(s)=\max\{\gamma_w/(\gamma_w+\overline{\mathrm{SNR}}(s)),\,0.1\}$ with $\gamma_w=3$, where $\overline{\mathrm{SNR}}(s)$ averages $\|m_s(H)-\vecop M_0\|^2/\operatorname{tr}C_s(H)$ over the training set on a 65-point time grid (linear interpolation), so weighting does not change the pointwise unrestricted minimiser for the implemented targets; the variance in the SNR denominator is floored at $\epsilon=10^{-6}$ \\
Numerical safeguards & covariance floor $10^{-10}$, cap $10^4$; score clipping $\pm10^2$; non-finite values zeroed \\
Sampling & Euler--Maruyama reverse SDE, 64 steps, base law $\pi=\mathcal N(M_0,0.25\,I)$; $n_{\mathrm{gen}}=8$ per seed (plus 64 audit draws, Appendix~\ref{app:projection}) \\
Decoding & threshold $\vartheta=0.55$, post-processing rule with $k=2$ (Algorithm~\ref{alg:post}) \\
HCM-MCMC & one pairwise-reshuffle chain per sample \citep{chodrow2020configuration}, starting subhypergraph drawn uniformly from the training set, $20\times\mathrm{nnz}$ transitions where $\mathrm{nnz}$ is its number of incidences, no burn-in or thinning; preserves degree and size sequences exactly \\
ER-HG & i.i.d.\ Bernoulli incidences at the training-set density, then Algorithm~\ref{alg:post} with $R=1-U$ \\
HERGM & reference subhypergraph drawn uniformly from the training set; node and hyperedge factors $x,y$ fitted by fixed-point iteration (residual $<10^{-10}$) so that $p_{ij}=x_iy_j/(1+x_iy_j)$ has that subhypergraph's degrees and sizes as expectations \citep{saracco2025entropy}; i.i.d.\ Bernoulli draw, then Algorithm~\ref{alg:post} with $R=P-U$ \\
Seeds & $11, 22, 33$ (the seed fixes both the set and the training/sampling randomness) \\
\bottomrule
\end{tabular}
\caption{Settings used for all reported real-data runs.}
\label{tab:hyperparameters}
\end{table}

\section{Base Law versus Finite-Horizon Mixture}
\label{app:terminal_law}

Table~\ref{tab:terminal_bounds} reports the bounds of Proposition~\ref{prop:terminal_bounds} for every dataset, computed from the per-hypergraph moments of the training set at the horizon with the separable solver (Appendix~\ref{app:moments}). The normalised $W_2$ bound, $\sqrt{\text{bound}}/\sqrt{nm\,\sigma^2}$, compares the mismatch with the root-mean-square radius of $\pi$ and lies between $0.088$ and $0.102$ across datasets; between $37\%$ and $48\%$ of the $W_2^2$ bound comes from the mean term $\|m_S(H)-M_0\|^2$, the remainder from under-dispersion of the conditional law relative to $\pi$, of which the joint null modes ($\lambda_j+\mu_k=0$, where the heat term does not act) contribute less than $0.4\%$. Sampled sliced-$W_1$ and RBF-MMD discrepancies between $64$ draws of $\hat p_S$ and of $\pi$ with a common standard normal are also shown. These values establish that $\pi$ is an approximation of the finite-horizon mixture, not that generated hypergraphs differ.

\begin{table}[h]
\centering
\scriptsize
\fitwidth{\begin{tabular}{lrrrrrrr}
\toprule
Dataset & $W_2^2$ bound & norm.\ $W_2$ bound & KL bound & mean-term share & null-mode share & sliced $W_1$ & MMD \\
\midrule
cora & 3.080 $\pm$ 0.037 & 0.088 $\pm$ 0.001 & 10.226 $\pm$ 0.080 & 0.373 $\pm$ 0.007 & 3.75e-03 $\pm$ 1.3e-04 & 0.031 $\pm$ 0.000 & 8.33e-04 $\pm$ 9.0e-06 \\
citeseer & 3.457 $\pm$ 0.098 & 0.093 $\pm$ 0.001 & 11.024 $\pm$ 0.202 & 0.435 $\pm$ 0.015 & 3.38e-03 $\pm$ 1.8e-04 & 0.031 $\pm$ 0.000 & 8.40e-04 $\pm$ 7.1e-06 \\
actor & 2.070 $\pm$ 0.036 & 0.090 $\pm$ 0.001 & 6.884 $\pm$ 0.073 & 0.370 $\pm$ 0.010 & 1.32e-03 $\pm$ 9.4e-05 & 0.033 $\pm$ 0.000 & 1.07e-03 $\pm$ 7.3e-06 \\
house\_committees & 2.651 $\pm$ 0.082 & 0.102 $\pm$ 0.002 & 8.188 $\pm$ 0.163 & 0.483 $\pm$ 0.016 & 1.56e-04 $\pm$ 2.3e-05 & 0.034 $\pm$ 0.001 & 8.95e-04 $\pm$ 1.1e-05 \\
dblp & 2.783 $\pm$ 0.085 & 0.098 $\pm$ 0.001 & 8.642 $\pm$ 0.168 & 0.474 $\pm$ 0.016 & 1.58e-03 $\pm$ 2.6e-05 & 0.034 $\pm$ 0.001 & 1.03e-03 $\pm$ 1.0e-05 \\
twitch & 2.412 $\pm$ 0.118 & 0.097 $\pm$ 0.002 & 7.589 $\pm$ 0.241 & 0.455 $\pm$ 0.026 & 1.65e-03 $\pm$ 1.5e-04 & 0.033 $\pm$ 0.001 & 1.04e-03 $\pm$ 1.3e-05 \\
\bottomrule
\end{tabular}
}
\caption{Convexity upper bounds (Proposition~\ref{prop:terminal_bounds}) for the empirical terminal mixture against the base law, with sampled discrepancies; mean $\pm$ sample standard deviation over three seeds.}
\label{tab:terminal_bounds}
\end{table}

To measure the downstream effect we ran a paired experiment on every checkpoint: $64$ standard-normal draws $\xi_k$ were mapped either to $M_0+\sigma\xi_k$ (the base law used by Algorithm~\ref{alg:sampling}) or to $m_S(H_{i_k})+C_S(H_{i_k})^{1/2}\xi_k$ with a training hypergraph $H_{i_k}$ drawn uniformly from the training set (a draw from $\hat p_S$), and both arms were integrated with identical Brownian increments and decoded identically. Table~\ref{tab:paired_initialisation} reports the mean difference over seeds (mixture minus base) relative to the across-seed standard deviation of the base arm. The difference exceeds one seed standard deviation in one of the $54$ dataset--metric cells, the Overlap Tail Gap on \textsc{Twitch} ($+0.002$, $1.24$ seed standard deviations). The largest other effects are on \textsc{DBLP}, where mixture initialisation lowers the mean hyperedge size by $0.15$ and raises Feature MMD by $0.06$ ($0.80$ and $0.85$ seed standard deviations), and on the mean-size gap of \textsc{Twitch} ($0.85$); on the other four datasets all ratios are below $0.25$. The effect of the base law is therefore within seed variability on these benchmarks, but not uniformly negligible.

\begin{table}[h]
\centering
\scriptsize
\fitwidth{\begin{tabular}{lrrrrrrrrr}
\toprule
Dataset & $\Delta e$ & $\Delta k$ & $W_1(k)$ & $W_1(e)$ & Intersect.\ WD & Tail gap & Node spec.\ WD & Edge spec.\ WD & Feature MMD \\
\midrule
Cora & -0.003 (0.06) & -0.001 (0.06) & -0.001 (0.04) & -0.003 (0.12) & -0.000 (0.05) & +0.000 (0.09) & -0.000 (0.05) & +0.000 (0.04) & +0.000 (0.00) \\
CiteSeer & -0.026 (0.22) & -0.010 (0.22) & -0.001 (0.02) & -0.010 (0.06) & -0.002 (0.07) & +0.000 (0.04) & -0.003 (0.19) & -0.003 (0.21) & +0.005 (0.17) \\
Actor & -0.010 (0.12) & -0.003 (0.12) & +0.002 (0.07) & +0.011 (0.13) & -0.001 (0.05) & +0.000 (0.12) & -0.000 (0.02) & -0.003 (0.05) & -0.004 (0.07) \\
H.-Comm. & -0.021 (0.10) & -0.005 (0.10) & +0.006 (0.19) & +0.018 (0.08) & +0.002 (0.07) & +0.001 (0.12) & +0.000 (0.02) & -0.002 (0.19) & +0.004 (0.04) \\
DBLP & -0.145 (0.80) & -0.041 (0.80) & -0.001 (0.03) & +0.058 (0.49) & -0.003 (0.14) & +0.004 (0.45) & -0.007 (0.40) & -0.007 (0.26) & +0.059 (0.85) \\
Twitch & -0.031 (0.85) & -0.008 (0.85) & +0.005 (0.26) & -0.010 (0.06) & +0.000 (0.03) & +0.002 (1.24) & -0.001 (0.08) & -0.006 (0.33) & +0.001 (0.01) \\
\bottomrule
\end{tabular}
}
\caption{Paired initialisation sensitivity: mean difference (mixture minus base) over three seeds for the full projection, with the ratio to the base arm's seed standard deviation in parentheses.}
\label{tab:paired_initialisation}
\end{table}

\section{Projection and Post-Processing Diagnostics}
\label{app:projection}

\paragraph{Threshold choice.} The threshold $\vartheta=0.55$ was introduced during synthetic-data tuning of a preliminary model variant to counter over-density and fixed before the real-data evaluation.

\paragraph{Structural properties.} Table~\ref{tab:structural} compares the real datasets with every generated set at the three decoding stages (threshold only; threshold and size rule; the full projection). Thresholding alone reproduces the isolated-vertex rate of the real datasets on every dataset but leaves hyperedges of size below two in $42$--$96\%$ of samples, which the size rule removes; the degree rule then removes all isolated vertices, which the real datasets and HCM-MCMC retain. Relaxed entries are close to binary: $91$--$93\%$ lie within $0.01$ of $\{0,1\}$ ($69\%$ on \textsc{House-Committees}, the least concentrated dataset) and $98.7$--$99.95\%$ within $0.1$, and essentially none lie in $[0.4,0.6]$, which is why the decodings at $\vartheta=0.50$ and $\vartheta=0.55$ coincide on every reported sample (Table~\ref{tab:threshold_050_055}) and on all $64$ additional audit draws per checkpoint.

\begin{table}[h]
\centering
\scriptsize
\fitwidth{\begin{tabular}{llrrrrr}
\toprule
dataset & bank & isolated vertices (node level) & samples with $\geq 1$ isolated vertex & singleton-edge rate & empty-edge rate & all edges size $\geq 2$ \\
\midrule
Cora & real train bank & 0.0300 $\pm$ 0.0021 & 0.816 $\pm$ 0.026 & 0.0000 $\pm$ 0.0000 & 0.0000 $\pm$ 0.0000 & 1.000 $\pm$ 0.000 \\
Cora & real test bank & 0.0229 $\pm$ 0.0065 & 0.800 $\pm$ 0.100 & 0.0000 $\pm$ 0.0000 & 0.0000 $\pm$ 0.0000 & 1.000 $\pm$ 0.000 \\
Cora & HEDGE threshold-only & 0.0352 $\pm$ 0.0203 & 0.833 $\pm$ 0.289 & 0.0700 $\pm$ 0.0087 & 0.0083 $\pm$ 0.0058 & 0.083 $\pm$ 0.072 \\
Cora & HEDGE + edge rule & 0.0293 $\pm$ 0.0174 & 0.792 $\pm$ 0.361 & 0.0000 $\pm$ 0.0000 & 0.0000 $\pm$ 0.0000 & 1.000 $\pm$ 0.000 \\
Cora & HEDGE submitted projection (full) & 0.0000 $\pm$ 0.0000 & 0.000 $\pm$ 0.000 & 0.0000 $\pm$ 0.0000 & 0.0000 $\pm$ 0.0000 & 1.000 $\pm$ 0.000 \\
Cora & HCM-MCMC (frozen) & 0.0339 $\pm$ 0.0144 & 0.792 $\pm$ 0.191 & 0.0000 $\pm$ 0.0000 & 0.0000 $\pm$ 0.0000 & 1.000 $\pm$ 0.000 \\
Cora & ER-HG corrected (B3) & 0.0000 $\pm$ 0.0000 & 0.000 $\pm$ 0.000 & 0.0000 $\pm$ 0.0000 & 0.0000 $\pm$ 0.0000 & 1.000 $\pm$ 0.000 \\
\midrule
CiteSeer & real train bank & 0.0488 $\pm$ 0.0019 & 0.912 $\pm$ 0.030 & 0.0000 $\pm$ 0.0000 & 0.0000 $\pm$ 0.0000 & 1.000 $\pm$ 0.000 \\
CiteSeer & real test bank & 0.0573 $\pm$ 0.0217 & 0.967 $\pm$ 0.058 & 0.0000 $\pm$ 0.0000 & 0.0000 $\pm$ 0.0000 & 1.000 $\pm$ 0.000 \\
CiteSeer & HEDGE threshold-only & 0.0488 $\pm$ 0.0000 & 0.958 $\pm$ 0.072 & 0.0983 $\pm$ 0.0189 & 0.0183 $\pm$ 0.0058 & 0.042 $\pm$ 0.072 \\
CiteSeer & HEDGE + edge rule & 0.0449 $\pm$ 0.0020 & 0.958 $\pm$ 0.072 & 0.0000 $\pm$ 0.0000 & 0.0000 $\pm$ 0.0000 & 1.000 $\pm$ 0.000 \\
CiteSeer & HEDGE submitted projection (full) & 0.0000 $\pm$ 0.0000 & 0.000 $\pm$ 0.000 & 0.0000 $\pm$ 0.0000 & 0.0000 $\pm$ 0.0000 & 1.000 $\pm$ 0.000 \\
CiteSeer & HCM-MCMC (frozen) & 0.0443 $\pm$ 0.0158 & 0.875 $\pm$ 0.000 & 0.0000 $\pm$ 0.0000 & 0.0000 $\pm$ 0.0000 & 1.000 $\pm$ 0.000 \\
CiteSeer & ER-HG corrected (B3) & 0.0000 $\pm$ 0.0000 & 0.000 $\pm$ 0.000 & 0.0000 $\pm$ 0.0000 & 0.0000 $\pm$ 0.0000 & 1.000 $\pm$ 0.000 \\
\midrule
Actor & real train bank & 0.0175 $\pm$ 0.0031 & 0.614 $\pm$ 0.061 & 0.0000 $\pm$ 0.0000 & 0.0000 $\pm$ 0.0000 & 1.000 $\pm$ 0.000 \\
Actor & real test bank & 0.0135 $\pm$ 0.0048 & 0.667 $\pm$ 0.115 & 0.0000 $\pm$ 0.0000 & 0.0000 $\pm$ 0.0000 & 1.000 $\pm$ 0.000 \\
Actor & HEDGE threshold-only & 0.0150 $\pm$ 0.0060 & 0.583 $\pm$ 0.191 & 0.0391 $\pm$ 0.0078 & 0.0026 $\pm$ 0.0045 & 0.417 $\pm$ 0.144 \\
Actor & HEDGE + edge rule & 0.0150 $\pm$ 0.0060 & 0.583 $\pm$ 0.191 & 0.0000 $\pm$ 0.0000 & 0.0000 $\pm$ 0.0000 & 1.000 $\pm$ 0.000 \\
Actor & HEDGE submitted projection (full) & 0.0000 $\pm$ 0.0000 & 0.000 $\pm$ 0.000 & 0.0000 $\pm$ 0.0000 & 0.0000 $\pm$ 0.0000 & 1.000 $\pm$ 0.000 \\
Actor & HCM-MCMC (frozen) & 0.0150 $\pm$ 0.0063 & 0.583 $\pm$ 0.072 & 0.0000 $\pm$ 0.0000 & 0.0000 $\pm$ 0.0000 & 1.000 $\pm$ 0.000 \\
Actor & ER-HG corrected (B3) & 0.0000 $\pm$ 0.0000 & 0.000 $\pm$ 0.000 & 0.0000 $\pm$ 0.0000 & 0.0000 $\pm$ 0.0000 & 1.000 $\pm$ 0.000 \\
\midrule
House-Committees & real train bank & 0.0099 $\pm$ 0.0004 & 0.474 $\pm$ 0.026 & 0.0000 $\pm$ 0.0000 & 0.0000 $\pm$ 0.0000 & 1.000 $\pm$ 0.000 \\
House-Committees & real test bank & 0.0094 $\pm$ 0.0054 & 0.367 $\pm$ 0.115 & 0.0000 $\pm$ 0.0000 & 0.0000 $\pm$ 0.0000 & 1.000 $\pm$ 0.000 \\
House-Committees & HEDGE threshold-only & 0.0143 $\pm$ 0.0108 & 0.583 $\pm$ 0.289 & 0.0391 $\pm$ 0.0135 & 0.0052 $\pm$ 0.0045 & 0.583 $\pm$ 0.144 \\
House-Committees & HEDGE + edge rule & 0.0137 $\pm$ 0.0109 & 0.542 $\pm$ 0.260 & 0.0000 $\pm$ 0.0000 & 0.0000 $\pm$ 0.0000 & 1.000 $\pm$ 0.000 \\
House-Committees & HEDGE submitted projection (full) & 0.0000 $\pm$ 0.0000 & 0.000 $\pm$ 0.000 & 0.0000 $\pm$ 0.0000 & 0.0000 $\pm$ 0.0000 & 1.000 $\pm$ 0.000 \\
House-Committees & HCM-MCMC (frozen) & 0.0098 $\pm$ 0.0039 & 0.542 $\pm$ 0.191 & 0.0000 $\pm$ 0.0000 & 0.0000 $\pm$ 0.0000 & 1.000 $\pm$ 0.000 \\
House-Committees & ER-HG corrected (B3) & 0.0000 $\pm$ 0.0000 & 0.000 $\pm$ 0.000 & 0.0000 $\pm$ 0.0000 & 0.0000 $\pm$ 0.0000 & 1.000 $\pm$ 0.000 \\
\midrule
DBLP & real train bank & 0.0548 $\pm$ 0.0045 & 0.877 $\pm$ 0.080 & 0.0000 $\pm$ 0.0000 & 0.0000 $\pm$ 0.0000 & 1.000 $\pm$ 0.000 \\
DBLP & real test bank & 0.0625 $\pm$ 0.0063 & 0.933 $\pm$ 0.115 & 0.0000 $\pm$ 0.0000 & 0.0000 $\pm$ 0.0000 & 1.000 $\pm$ 0.000 \\
DBLP & HEDGE threshold-only & 0.0651 $\pm$ 0.0236 & 1.000 $\pm$ 0.000 & 0.0995 $\pm$ 0.0080 & 0.0162 $\pm$ 0.0040 & 0.167 $\pm$ 0.072 \\
DBLP & HEDGE + edge rule & 0.0632 $\pm$ 0.0222 & 0.958 $\pm$ 0.072 & 0.0000 $\pm$ 0.0000 & 0.0000 $\pm$ 0.0000 & 1.000 $\pm$ 0.000 \\
DBLP & HEDGE submitted projection (full) & 0.0000 $\pm$ 0.0000 & 0.000 $\pm$ 0.000 & 0.0000 $\pm$ 0.0000 & 0.0000 $\pm$ 0.0000 & 1.000 $\pm$ 0.000 \\
DBLP & HCM-MCMC (frozen) & 0.0462 $\pm$ 0.0157 & 0.833 $\pm$ 0.072 & 0.0000 $\pm$ 0.0000 & 0.0000 $\pm$ 0.0000 & 1.000 $\pm$ 0.000 \\
DBLP & ER-HG corrected (B3) & 0.0000 $\pm$ 0.0000 & 0.000 $\pm$ 0.000 & 0.0000 $\pm$ 0.0000 & 0.0000 $\pm$ 0.0000 & 1.000 $\pm$ 0.000 \\
\midrule
Twitch & real train bank & 0.0247 $\pm$ 0.0038 & 0.833 $\pm$ 0.080 & 0.0000 $\pm$ 0.0000 & 0.0000 $\pm$ 0.0000 & 1.000 $\pm$ 0.000 \\
Twitch & real test bank & 0.0271 $\pm$ 0.0059 & 0.900 $\pm$ 0.000 & 0.0000 $\pm$ 0.0000 & 0.0000 $\pm$ 0.0000 & 1.000 $\pm$ 0.000 \\
Twitch & HEDGE threshold-only & 0.0208 $\pm$ 0.0126 & 0.708 $\pm$ 0.144 & 0.0365 $\pm$ 0.0045 & 0.0000 $\pm$ 0.0000 & 0.500 $\pm$ 0.125 \\
Twitch & HEDGE + edge rule & 0.0202 $\pm$ 0.0130 & 0.708 $\pm$ 0.144 & 0.0000 $\pm$ 0.0000 & 0.0000 $\pm$ 0.0000 & 1.000 $\pm$ 0.000 \\
Twitch & HEDGE submitted projection (full) & 0.0000 $\pm$ 0.0000 & 0.000 $\pm$ 0.000 & 0.0000 $\pm$ 0.0000 & 0.0000 $\pm$ 0.0000 & 1.000 $\pm$ 0.000 \\
Twitch & HCM-MCMC (frozen) & 0.0306 $\pm$ 0.0079 & 0.917 $\pm$ 0.144 & 0.0000 $\pm$ 0.0000 & 0.0000 $\pm$ 0.0000 & 1.000 $\pm$ 0.000 \\
Twitch & ER-HG corrected (B3) & 0.0000 $\pm$ 0.0000 & 0.000 $\pm$ 0.000 & 0.0000 $\pm$ 0.0000 & 0.0000 $\pm$ 0.0000 & 1.000 $\pm$ 0.000 \\
\bottomrule
\end{tabular}
}
\caption{Structural properties of the real datasets and of the generated sets (mean $\pm$ sample standard deviation over three seeds). Isolated vertices are a property of the data, not a defect.}
\label{tab:structural}
\end{table}

\paragraph{Edits.} Table~\ref{tab:edits} quantifies the rule: it touches $67$--$100\%$ of \HEDGE\ samples but changes only $0.16$--$0.56\%$ of entries ($1.7$--$6.4$ added incidences per sample, split between the two rules); on ER-HG it changes $0.8$--$2.1\%$ of entries, almost all through the degree rule. Tables~\ref{tab:threshold_sweep} and~\ref{tab:threshold_sweep_full} sweep the threshold from $0.3$ to $0.7$ for the threshold-only and the full decoding, and Table~\ref{tab:stage_metrics} reports every metric for the three stages side by side with the full-minus-threshold-only difference in units of the across-seed standard deviation. The size rule improves hyperedge-size and edge-spectral metrics, whereas the degree rule mostly lowers Feature MMD and can inflate node-spectral distances; among the distance metrics the largest effects are on \textsc{Cora} (Feature MMD, $5.0$ standard deviations) and \textsc{DBLP} (edge-spectral WD, $2.0$); the signed density, mean-size and mean-degree gaps shift by $2.2$--$3.3$ standard deviations on \textsc{Actor}, \textsc{Cora} and \textsc{CiteSeer}, because the rule only adds incidences.

\paragraph{Effect on comparisons.} Of the $42$ dataset-level comparisons between \HEDGE\ and HCM-MCMC (seven metrics, six datasets), four change sign between the threshold-only and the full decoding, all in favour of \HEDGE\ under the rule: Feature MMD on \textsc{Cora}, the node-degree and hyperedge-size distances on \textsc{Actor}, and edge-spectral WD on \textsc{DBLP}. Against ER-HG one changes (Feature MMD on \textsc{DBLP}) and against HERGM none. The main text distinguishes comparisons that hold under both decodings from those that depend on the rule.

\paragraph{Null models and the degree rule.} Table~\ref{tab:null_model_stages} reports ER-HG and HERGM without post-processing, with the size rule only, and with both rules. Without post-processing HERGM reproduces the mean hyperedge size and degree of its reference subhypergraphs almost exactly (gaps $0.06$ and $0.02$ averaged over datasets), as its fitted expectations require, but both independent-incidence models leave about $30\%$ of vertices isolated, against $3\%$ in the real datasets, because a node of expected degree one is isolated with probability near $e^{-1}$ and isolated vertices of the reference subhypergraph have probability zero of any incidence. The degree rule then attaches every isolated vertex, which inflates the density, mean size and size-distribution errors by an order of magnitude; it is therefore the post-processing rule, not the models, that produces the large calibration errors of the ranked ER-HG and HERGM rows. \HEDGE\ is better than both models at every stage on Intersection WD and both spectral distances, the higher-order comparisons the main text makes against them, so this does not affect those claims, but it shows that the degree rule is far from neutral across methods and is a further reason to report both decodings.

\begin{table}[h]
\centering
\scriptsize
\fitwidth{\begin{tabular}{llrrrrrr}
\toprule
dataset & bank & samples touched & touched by edge rule only & incidences added / sample & by edge rule & by node rule & entries changed \\
\midrule
Cora & frozen 8-sample banks (B1) & 0.958 $\pm$ 0.072 & 0.167 $\pm$ 0.289 & 4.042 $\pm$ 1.443 & 2.167 $\pm$ 0.439 & 1.875 $\pm$ 1.111 & 0.253 $\pm$ 0.090\% \\
Cora & 64-sample audit banks (B2) & 0.984 $\pm$ 0.027 & 0.135 $\pm$ 0.036 & 4.135 $\pm$ 0.715 & 2.255 $\pm$ 0.370 & 1.880 $\pm$ 0.346 & 0.258 $\pm$ 0.045\% \\
Cora & ER-HG corrected (B3) & 1.000 $\pm$ 0.000 & 0.000 $\pm$ 0.000 & 24.667 $\pm$ 1.301 & 6.542 $\pm$ 0.641 & 18.125 $\pm$ 0.661 & 1.542 $\pm$ 0.081\% \\
\midrule
CiteSeer & frozen 8-sample banks (B1) & 1.000 $\pm$ 0.000 & 0.042 $\pm$ 0.072 & 6.250 $\pm$ 0.573 & 3.375 $\pm$ 0.451 & 2.875 $\pm$ 0.125 & 0.391 $\pm$ 0.036\% \\
CiteSeer & 64-sample audit banks (B2) & 1.000 $\pm$ 0.000 & 0.052 $\pm$ 0.024 & 6.094 $\pm$ 0.597 & 2.953 $\pm$ 0.516 & 3.141 $\pm$ 0.159 & 0.381 $\pm$ 0.037\% \\
CiteSeer & ER-HG corrected (B3) & 1.000 $\pm$ 0.000 & 0.000 $\pm$ 0.000 & 22.750 $\pm$ 1.192 & 5.583 $\pm$ 0.382 & 17.167 $\pm$ 0.851 & 1.422 $\pm$ 0.075\% \\
\midrule
Actor & frozen 8-sample banks (B1) & 0.833 $\pm$ 0.191 & 0.250 $\pm$ 0.000 & 1.667 $\pm$ 0.505 & 0.708 $\pm$ 0.191 & 0.958 $\pm$ 0.382 & 0.163 $\pm$ 0.049\% \\
Actor & 64-sample audit banks (B2) & 0.875 $\pm$ 0.041 & 0.245 $\pm$ 0.145 & 1.943 $\pm$ 0.663 & 0.828 $\pm$ 0.122 & 1.115 $\pm$ 0.591 & 0.190 $\pm$ 0.065\% \\
Actor & ER-HG corrected (B3) & 1.000 $\pm$ 0.000 & 0.000 $\pm$ 0.000 & 21.875 $\pm$ 0.976 & 1.292 $\pm$ 0.473 & 20.583 $\pm$ 0.505 & 2.136 $\pm$ 0.095\% \\
\midrule
House-Committees & frozen 8-sample banks (B1) & 0.667 $\pm$ 0.260 & 0.125 $\pm$ 0.000 & 1.667 $\pm$ 0.832 & 0.792 $\pm$ 0.315 & 0.875 $\pm$ 0.696 & 0.163 $\pm$ 0.081\% \\
House-Committees & 64-sample audit banks (B2) & 0.854 $\pm$ 0.072 & 0.130 $\pm$ 0.092 & 1.953 $\pm$ 0.384 & 0.615 $\pm$ 0.258 & 1.339 $\pm$ 0.430 & 0.191 $\pm$ 0.038\% \\
House-Committees & ER-HG corrected (B3) & 1.000 $\pm$ 0.000 & 0.000 $\pm$ 0.000 & 7.833 $\pm$ 1.734 & 0.083 $\pm$ 0.072 & 7.750 $\pm$ 1.803 & 0.765 $\pm$ 0.169\% \\
\midrule
DBLP & frozen 8-sample banks (B1) & 1.000 $\pm$ 0.000 & 0.042 $\pm$ 0.072 & 6.417 $\pm$ 1.422 & 2.375 $\pm$ 0.000 & 4.042 $\pm$ 1.422 & 0.557 $\pm$ 0.123\% \\
DBLP & 64-sample audit banks (B2) & 0.995 $\pm$ 0.009 & 0.016 $\pm$ 0.016 & 7.068 $\pm$ 1.398 & 2.224 $\pm$ 0.206 & 4.844 $\pm$ 1.192 & 0.614 $\pm$ 0.121\% \\
DBLP & ER-HG corrected (B3) & 1.000 $\pm$ 0.000 & 0.000 $\pm$ 0.000 & 21.542 $\pm$ 1.018 & 1.375 $\pm$ 0.661 & 20.167 $\pm$ 0.688 & 1.870 $\pm$ 0.088\% \\
\midrule
Twitch & frozen 8-sample banks (B1) & 0.792 $\pm$ 0.191 & 0.083 $\pm$ 0.072 & 1.875 $\pm$ 0.760 & 0.583 $\pm$ 0.072 & 1.292 $\pm$ 0.832 & 0.183 $\pm$ 0.074\% \\
Twitch & 64-sample audit banks (B2) & 0.906 $\pm$ 0.095 & 0.182 $\pm$ 0.086 & 2.125 $\pm$ 0.600 & 0.724 $\pm$ 0.086 & 1.401 $\pm$ 0.524 & 0.208 $\pm$ 0.059\% \\
Twitch & ER-HG corrected (B3) & 1.000 $\pm$ 0.000 & 0.000 $\pm$ 0.000 & 19.917 $\pm$ 0.971 & 1.042 $\pm$ 0.382 & 18.875 $\pm$ 0.650 & 1.945 $\pm$ 0.095\% \\
\bottomrule
\end{tabular}
}
\caption{Post-processing edit statistics (mean $\pm$ sample standard deviation over seeds) for the eight reported samples per run (B1), the 64-sample audit batch (B2), and ER-HG with the corrected decoder (B3).}
\label{tab:edits}
\end{table}

\begin{table}[h]
\centering
\scriptsize
\fitwidth{\begin{tabular}{llrrrrrr}
\toprule
dataset & stage & entries / bank & differing entries (mean) & (max) & fraction (mean) & (max) & samples that differ \\
\midrule
Cora & threshold-only & 12800 & 0.00 & 0 & 0.0000\% & 0.0000\% & 0.000 \\
Cora & + edge rule & 12800 & 0.00 & 0 & 0.0000\% & 0.0000\% & 0.000 \\
Cora & full (submitted) & 12800 & 0.00 & 0 & 0.0000\% & 0.0000\% & 0.000 \\
\midrule
CiteSeer & threshold-only & 12800 & 0.00 & 0 & 0.0000\% & 0.0000\% & 0.000 \\
CiteSeer & + edge rule & 12800 & 0.00 & 0 & 0.0000\% & 0.0000\% & 0.000 \\
CiteSeer & full (submitted) & 12800 & 0.00 & 0 & 0.0000\% & 0.0000\% & 0.000 \\
\midrule
Actor & threshold-only & 8192 & 0.00 & 0 & 0.0000\% & 0.0000\% & 0.000 \\
Actor & + edge rule & 8192 & 0.00 & 0 & 0.0000\% & 0.0000\% & 0.000 \\
Actor & full (submitted) & 8192 & 0.00 & 0 & 0.0000\% & 0.0000\% & 0.000 \\
\midrule
House-Committees & threshold-only & 8192 & 0.00 & 0 & 0.0000\% & 0.0000\% & 0.000 \\
House-Committees & + edge rule & 8192 & 0.00 & 0 & 0.0000\% & 0.0000\% & 0.000 \\
House-Committees & full (submitted) & 8192 & 0.00 & 0 & 0.0000\% & 0.0000\% & 0.000 \\
\midrule
DBLP & threshold-only & 9216 & 0.00 & 0 & 0.0000\% & 0.0000\% & 0.000 \\
DBLP & + edge rule & 9216 & 0.00 & 0 & 0.0000\% & 0.0000\% & 0.000 \\
DBLP & full (submitted) & 9216 & 0.00 & 0 & 0.0000\% & 0.0000\% & 0.000 \\
\midrule
Twitch & threshold-only & 8192 & 0.00 & 0 & 0.0000\% & 0.0000\% & 0.000 \\
Twitch & + edge rule & 8192 & 0.00 & 0 & 0.0000\% & 0.0000\% & 0.000 \\
Twitch & full (submitted) & 8192 & 0.00 & 0 & 0.0000\% & 0.0000\% & 0.000 \\
\bottomrule
\end{tabular}
}
\caption{Entries that differ between the decodings at $\vartheta=0.50$ and $\vartheta=0.55$ on the reported samples: none, for every dataset, seed and stage.}
\label{tab:threshold_050_055}
\end{table}

\begin{table*}[p]
\centering
\scriptsize
\setlength{\tabcolsep}{3pt}
\fitwidth{\begin{tabular}{llrrrrrrrrrr}
\toprule
Dataset & $\vartheta$ & isolated vertices & all edges size $\ge2$ & $\Delta\rho$ & $W_1(k)$ & $W_1(e)$ & Intersect.\ WD & Tail gap & Node spec.\ WD & Edge spec.\ WD & Feature MMD \\
\midrule
Cora & 0.30 & 0.035 {\tiny$\pm$} 0.020 & 0.083 {\tiny$\pm$} 0.072 & -0.001 {\tiny$\pm$} 0.002 & 0.072 {\tiny$\pm$} 0.010 & 0.227 {\tiny$\pm$} 0.025 & 0.017 {\tiny$\pm$} 0.009 & 0.006 {\tiny$\pm$} 0.001 & 0.001 {\tiny$\pm$} 0.001 & 0.068 {\tiny$\pm$} 0.004 & 0.215 {\tiny$\pm$} 0.132 \\
 & 0.40 & 0.035 {\tiny$\pm$} 0.020 & 0.083 {\tiny$\pm$} 0.072 & -0.001 {\tiny$\pm$} 0.002 & 0.072 {\tiny$\pm$} 0.011 & 0.225 {\tiny$\pm$} 0.022 & 0.016 {\tiny$\pm$} 0.009 & 0.006 {\tiny$\pm$} 0.001 & 0.001 {\tiny$\pm$} 0.001 & 0.068 {\tiny$\pm$} 0.004 & 0.214 {\tiny$\pm$} 0.133 \\
 & 0.50 & 0.035 {\tiny$\pm$} 0.020 & 0.083 {\tiny$\pm$} 0.072 & -0.001 {\tiny$\pm$} 0.002 & 0.072 {\tiny$\pm$} 0.011 & 0.225 {\tiny$\pm$} 0.022 & 0.016 {\tiny$\pm$} 0.009 & 0.006 {\tiny$\pm$} 0.001 & 0.001 {\tiny$\pm$} 0.001 & 0.068 {\tiny$\pm$} 0.004 & 0.214 {\tiny$\pm$} 0.133 \\
 & 0.55 & 0.035 {\tiny$\pm$} 0.020 & 0.083 {\tiny$\pm$} 0.072 & -0.001 {\tiny$\pm$} 0.002 & 0.072 {\tiny$\pm$} 0.011 & 0.225 {\tiny$\pm$} 0.022 & 0.016 {\tiny$\pm$} 0.009 & 0.006 {\tiny$\pm$} 0.001 & 0.001 {\tiny$\pm$} 0.001 & 0.068 {\tiny$\pm$} 0.004 & 0.214 {\tiny$\pm$} 0.133 \\
 & 0.60 & 0.035 {\tiny$\pm$} 0.020 & 0.083 {\tiny$\pm$} 0.072 & -0.001 {\tiny$\pm$} 0.002 & 0.072 {\tiny$\pm$} 0.011 & 0.225 {\tiny$\pm$} 0.022 & 0.016 {\tiny$\pm$} 0.009 & 0.006 {\tiny$\pm$} 0.001 & 0.001 {\tiny$\pm$} 0.001 & 0.068 {\tiny$\pm$} 0.004 & 0.214 {\tiny$\pm$} 0.133 \\
 & 0.70 & 0.035 {\tiny$\pm$} 0.020 & 0.083 {\tiny$\pm$} 0.072 & -0.001 {\tiny$\pm$} 0.002 & 0.072 {\tiny$\pm$} 0.011 & 0.225 {\tiny$\pm$} 0.022 & 0.016 {\tiny$\pm$} 0.009 & 0.006 {\tiny$\pm$} 0.001 & 0.001 {\tiny$\pm$} 0.001 & 0.068 {\tiny$\pm$} 0.004 & 0.214 {\tiny$\pm$} 0.133 \\
\midrule
CiteSeer & 0.30 & 0.049 {\tiny$\pm$} 0.000 & 0.042 {\tiny$\pm$} 0.072 & -0.003 {\tiny$\pm$} 0.002 & 0.077 {\tiny$\pm$} 0.018 & 0.428 {\tiny$\pm$} 0.015 & 0.037 {\tiny$\pm$} 0.026 & 0.012 {\tiny$\pm$} 0.004 & 0.013 {\tiny$\pm$} 0.004 & 0.065 {\tiny$\pm$} 0.015 & 0.092 {\tiny$\pm$} 0.085 \\
 & 0.40 & 0.049 {\tiny$\pm$} 0.000 & 0.042 {\tiny$\pm$} 0.072 & -0.003 {\tiny$\pm$} 0.001 & 0.079 {\tiny$\pm$} 0.019 & 0.428 {\tiny$\pm$} 0.015 & 0.037 {\tiny$\pm$} 0.025 & 0.012 {\tiny$\pm$} 0.004 & 0.012 {\tiny$\pm$} 0.003 & 0.067 {\tiny$\pm$} 0.017 & 0.089 {\tiny$\pm$} 0.086 \\
 & 0.50 & 0.049 {\tiny$\pm$} 0.000 & 0.042 {\tiny$\pm$} 0.072 & -0.003 {\tiny$\pm$} 0.001 & 0.079 {\tiny$\pm$} 0.019 & 0.428 {\tiny$\pm$} 0.015 & 0.037 {\tiny$\pm$} 0.025 & 0.012 {\tiny$\pm$} 0.004 & 0.012 {\tiny$\pm$} 0.003 & 0.067 {\tiny$\pm$} 0.017 & 0.089 {\tiny$\pm$} 0.086 \\
 & 0.55 & 0.049 {\tiny$\pm$} 0.000 & 0.042 {\tiny$\pm$} 0.072 & -0.003 {\tiny$\pm$} 0.001 & 0.079 {\tiny$\pm$} 0.019 & 0.428 {\tiny$\pm$} 0.015 & 0.037 {\tiny$\pm$} 0.025 & 0.012 {\tiny$\pm$} 0.004 & 0.012 {\tiny$\pm$} 0.003 & 0.067 {\tiny$\pm$} 0.017 & 0.089 {\tiny$\pm$} 0.086 \\
 & 0.60 & 0.049 {\tiny$\pm$} 0.000 & 0.042 {\tiny$\pm$} 0.072 & -0.003 {\tiny$\pm$} 0.001 & 0.079 {\tiny$\pm$} 0.019 & 0.428 {\tiny$\pm$} 0.015 & 0.037 {\tiny$\pm$} 0.025 & 0.012 {\tiny$\pm$} 0.004 & 0.012 {\tiny$\pm$} 0.003 & 0.067 {\tiny$\pm$} 0.017 & 0.089 {\tiny$\pm$} 0.086 \\
 & 0.70 & 0.049 {\tiny$\pm$} 0.000 & 0.042 {\tiny$\pm$} 0.072 & -0.003 {\tiny$\pm$} 0.001 & 0.079 {\tiny$\pm$} 0.019 & 0.428 {\tiny$\pm$} 0.015 & 0.037 {\tiny$\pm$} 0.025 & 0.012 {\tiny$\pm$} 0.004 & 0.012 {\tiny$\pm$} 0.003 & 0.067 {\tiny$\pm$} 0.017 & 0.089 {\tiny$\pm$} 0.086 \\
\midrule
Actor & 0.30 & 0.015 {\tiny$\pm$} 0.006 & 0.417 {\tiny$\pm$} 0.144 & -0.002 {\tiny$\pm$} 0.001 & 0.040 {\tiny$\pm$} 0.012 & 0.358 {\tiny$\pm$} 0.073 & 0.032 {\tiny$\pm$} 0.019 & 0.003 {\tiny$\pm$} 0.002 & 0.025 {\tiny$\pm$} 0.015 & 0.034 {\tiny$\pm$} 0.011 & 0.107 {\tiny$\pm$} 0.096 \\
 & 0.40 & 0.015 {\tiny$\pm$} 0.006 & 0.417 {\tiny$\pm$} 0.144 & -0.002 {\tiny$\pm$} 0.001 & 0.040 {\tiny$\pm$} 0.012 & 0.358 {\tiny$\pm$} 0.073 & 0.032 {\tiny$\pm$} 0.019 & 0.003 {\tiny$\pm$} 0.002 & 0.025 {\tiny$\pm$} 0.015 & 0.034 {\tiny$\pm$} 0.011 & 0.107 {\tiny$\pm$} 0.096 \\
 & 0.50 & 0.015 {\tiny$\pm$} 0.006 & 0.417 {\tiny$\pm$} 0.144 & -0.002 {\tiny$\pm$} 0.001 & 0.040 {\tiny$\pm$} 0.012 & 0.358 {\tiny$\pm$} 0.073 & 0.032 {\tiny$\pm$} 0.019 & 0.003 {\tiny$\pm$} 0.002 & 0.025 {\tiny$\pm$} 0.015 & 0.034 {\tiny$\pm$} 0.011 & 0.107 {\tiny$\pm$} 0.096 \\
 & 0.55 & 0.015 {\tiny$\pm$} 0.006 & 0.417 {\tiny$\pm$} 0.144 & -0.002 {\tiny$\pm$} 0.001 & 0.040 {\tiny$\pm$} 0.012 & 0.358 {\tiny$\pm$} 0.073 & 0.032 {\tiny$\pm$} 0.019 & 0.003 {\tiny$\pm$} 0.002 & 0.025 {\tiny$\pm$} 0.015 & 0.034 {\tiny$\pm$} 0.011 & 0.107 {\tiny$\pm$} 0.096 \\
 & 0.60 & 0.015 {\tiny$\pm$} 0.006 & 0.417 {\tiny$\pm$} 0.144 & -0.002 {\tiny$\pm$} 0.001 & 0.040 {\tiny$\pm$} 0.012 & 0.358 {\tiny$\pm$} 0.073 & 0.032 {\tiny$\pm$} 0.019 & 0.003 {\tiny$\pm$} 0.002 & 0.025 {\tiny$\pm$} 0.015 & 0.034 {\tiny$\pm$} 0.011 & 0.107 {\tiny$\pm$} 0.096 \\
 & 0.70 & 0.015 {\tiny$\pm$} 0.006 & 0.417 {\tiny$\pm$} 0.144 & -0.002 {\tiny$\pm$} 0.001 & 0.041 {\tiny$\pm$} 0.011 & 0.362 {\tiny$\pm$} 0.066 & 0.031 {\tiny$\pm$} 0.019 & 0.003 {\tiny$\pm$} 0.002 & 0.024 {\tiny$\pm$} 0.016 & 0.026 {\tiny$\pm$} 0.023 & 0.109 {\tiny$\pm$} 0.095 \\
\midrule
H.-Comm. & 0.30 & 0.014 {\tiny$\pm$} 0.011 & 0.583 {\tiny$\pm$} 0.144 & -0.009 {\tiny$\pm$} 0.005 & 0.141 {\tiny$\pm$} 0.048 & 0.721 {\tiny$\pm$} 0.253 & 0.214 {\tiny$\pm$} 0.056 & 0.021 {\tiny$\pm$} 0.016 & 0.064 {\tiny$\pm$} 0.028 & 0.089 {\tiny$\pm$} 0.011 & 0.058 {\tiny$\pm$} 0.040 \\
 & 0.40 & 0.014 {\tiny$\pm$} 0.011 & 0.583 {\tiny$\pm$} 0.144 & -0.009 {\tiny$\pm$} 0.005 & 0.141 {\tiny$\pm$} 0.048 & 0.721 {\tiny$\pm$} 0.253 & 0.214 {\tiny$\pm$} 0.056 & 0.021 {\tiny$\pm$} 0.016 & 0.064 {\tiny$\pm$} 0.028 & 0.089 {\tiny$\pm$} 0.011 & 0.058 {\tiny$\pm$} 0.040 \\
 & 0.50 & 0.014 {\tiny$\pm$} 0.011 & 0.583 {\tiny$\pm$} 0.144 & -0.009 {\tiny$\pm$} 0.005 & 0.141 {\tiny$\pm$} 0.048 & 0.721 {\tiny$\pm$} 0.253 & 0.214 {\tiny$\pm$} 0.056 & 0.021 {\tiny$\pm$} 0.016 & 0.064 {\tiny$\pm$} 0.028 & 0.089 {\tiny$\pm$} 0.011 & 0.058 {\tiny$\pm$} 0.040 \\
 & 0.55 & 0.014 {\tiny$\pm$} 0.011 & 0.583 {\tiny$\pm$} 0.144 & -0.009 {\tiny$\pm$} 0.005 & 0.141 {\tiny$\pm$} 0.048 & 0.721 {\tiny$\pm$} 0.253 & 0.214 {\tiny$\pm$} 0.056 & 0.021 {\tiny$\pm$} 0.016 & 0.064 {\tiny$\pm$} 0.028 & 0.089 {\tiny$\pm$} 0.011 & 0.058 {\tiny$\pm$} 0.040 \\
 & 0.60 & 0.014 {\tiny$\pm$} 0.011 & 0.583 {\tiny$\pm$} 0.144 & -0.009 {\tiny$\pm$} 0.005 & 0.141 {\tiny$\pm$} 0.048 & 0.721 {\tiny$\pm$} 0.253 & 0.214 {\tiny$\pm$} 0.056 & 0.021 {\tiny$\pm$} 0.016 & 0.064 {\tiny$\pm$} 0.028 & 0.089 {\tiny$\pm$} 0.011 & 0.058 {\tiny$\pm$} 0.040 \\
 & 0.70 & 0.014 {\tiny$\pm$} 0.011 & 0.583 {\tiny$\pm$} 0.144 & -0.009 {\tiny$\pm$} 0.005 & 0.141 {\tiny$\pm$} 0.048 & 0.721 {\tiny$\pm$} 0.253 & 0.214 {\tiny$\pm$} 0.056 & 0.021 {\tiny$\pm$} 0.016 & 0.064 {\tiny$\pm$} 0.028 & 0.089 {\tiny$\pm$} 0.011 & 0.058 {\tiny$\pm$} 0.040 \\
\midrule
DBLP & 0.30 & 0.065 {\tiny$\pm$} 0.024 & 0.167 {\tiny$\pm$} 0.072 & -0.006 {\tiny$\pm$} 0.003 & 0.117 {\tiny$\pm$} 0.108 & 1.015 {\tiny$\pm$} 0.122 & 0.064 {\tiny$\pm$} 0.039 & 0.013 {\tiny$\pm$} 0.014 & 0.063 {\tiny$\pm$} 0.038 & 0.126 {\tiny$\pm$} 0.086 & 0.300 {\tiny$\pm$} 0.140 \\
 & 0.40 & 0.065 {\tiny$\pm$} 0.024 & 0.167 {\tiny$\pm$} 0.072 & -0.006 {\tiny$\pm$} 0.003 & 0.117 {\tiny$\pm$} 0.108 & 1.015 {\tiny$\pm$} 0.122 & 0.064 {\tiny$\pm$} 0.039 & 0.013 {\tiny$\pm$} 0.014 & 0.063 {\tiny$\pm$} 0.038 & 0.126 {\tiny$\pm$} 0.086 & 0.300 {\tiny$\pm$} 0.140 \\
 & 0.50 & 0.065 {\tiny$\pm$} 0.024 & 0.167 {\tiny$\pm$} 0.072 & -0.006 {\tiny$\pm$} 0.003 & 0.117 {\tiny$\pm$} 0.108 & 1.015 {\tiny$\pm$} 0.122 & 0.064 {\tiny$\pm$} 0.039 & 0.013 {\tiny$\pm$} 0.014 & 0.063 {\tiny$\pm$} 0.038 & 0.126 {\tiny$\pm$} 0.086 & 0.300 {\tiny$\pm$} 0.140 \\
 & 0.55 & 0.065 {\tiny$\pm$} 0.024 & 0.167 {\tiny$\pm$} 0.072 & -0.006 {\tiny$\pm$} 0.003 & 0.117 {\tiny$\pm$} 0.108 & 1.015 {\tiny$\pm$} 0.122 & 0.064 {\tiny$\pm$} 0.039 & 0.013 {\tiny$\pm$} 0.014 & 0.063 {\tiny$\pm$} 0.038 & 0.126 {\tiny$\pm$} 0.086 & 0.300 {\tiny$\pm$} 0.140 \\
 & 0.60 & 0.065 {\tiny$\pm$} 0.024 & 0.167 {\tiny$\pm$} 0.072 & -0.006 {\tiny$\pm$} 0.003 & 0.117 {\tiny$\pm$} 0.108 & 1.015 {\tiny$\pm$} 0.122 & 0.064 {\tiny$\pm$} 0.039 & 0.013 {\tiny$\pm$} 0.014 & 0.063 {\tiny$\pm$} 0.038 & 0.126 {\tiny$\pm$} 0.086 & 0.300 {\tiny$\pm$} 0.140 \\
 & 0.70 & 0.065 {\tiny$\pm$} 0.024 & 0.167 {\tiny$\pm$} 0.072 & -0.006 {\tiny$\pm$} 0.003 & 0.117 {\tiny$\pm$} 0.108 & 1.015 {\tiny$\pm$} 0.122 & 0.064 {\tiny$\pm$} 0.039 & 0.013 {\tiny$\pm$} 0.014 & 0.063 {\tiny$\pm$} 0.038 & 0.126 {\tiny$\pm$} 0.086 & 0.300 {\tiny$\pm$} 0.140 \\
\midrule
Twitch & 0.30 & 0.021 {\tiny$\pm$} 0.013 & 0.500 {\tiny$\pm$} 0.125 & -0.003 {\tiny$\pm$} 0.001 & 0.092 {\tiny$\pm$} 0.038 & 0.406 {\tiny$\pm$} 0.036 & 0.066 {\tiny$\pm$} 0.014 & 0.010 {\tiny$\pm$} 0.001 & 0.063 {\tiny$\pm$} 0.015 & 0.122 {\tiny$\pm$} 0.066 & 0.196 {\tiny$\pm$} 0.207 \\
 & 0.40 & 0.021 {\tiny$\pm$} 0.013 & 0.500 {\tiny$\pm$} 0.125 & -0.003 {\tiny$\pm$} 0.001 & 0.092 {\tiny$\pm$} 0.038 & 0.406 {\tiny$\pm$} 0.036 & 0.066 {\tiny$\pm$} 0.014 & 0.010 {\tiny$\pm$} 0.001 & 0.063 {\tiny$\pm$} 0.015 & 0.122 {\tiny$\pm$} 0.066 & 0.196 {\tiny$\pm$} 0.207 \\
 & 0.50 & 0.021 {\tiny$\pm$} 0.013 & 0.500 {\tiny$\pm$} 0.125 & -0.003 {\tiny$\pm$} 0.001 & 0.092 {\tiny$\pm$} 0.038 & 0.406 {\tiny$\pm$} 0.036 & 0.066 {\tiny$\pm$} 0.014 & 0.010 {\tiny$\pm$} 0.001 & 0.063 {\tiny$\pm$} 0.015 & 0.122 {\tiny$\pm$} 0.066 & 0.196 {\tiny$\pm$} 0.207 \\
 & 0.55 & 0.021 {\tiny$\pm$} 0.013 & 0.500 {\tiny$\pm$} 0.125 & -0.003 {\tiny$\pm$} 0.001 & 0.092 {\tiny$\pm$} 0.038 & 0.406 {\tiny$\pm$} 0.036 & 0.066 {\tiny$\pm$} 0.014 & 0.010 {\tiny$\pm$} 0.001 & 0.063 {\tiny$\pm$} 0.015 & 0.122 {\tiny$\pm$} 0.066 & 0.196 {\tiny$\pm$} 0.207 \\
 & 0.60 & 0.021 {\tiny$\pm$} 0.013 & 0.500 {\tiny$\pm$} 0.125 & -0.003 {\tiny$\pm$} 0.001 & 0.092 {\tiny$\pm$} 0.038 & 0.406 {\tiny$\pm$} 0.036 & 0.066 {\tiny$\pm$} 0.014 & 0.010 {\tiny$\pm$} 0.001 & 0.063 {\tiny$\pm$} 0.015 & 0.122 {\tiny$\pm$} 0.066 & 0.196 {\tiny$\pm$} 0.207 \\
 & 0.70 & 0.021 {\tiny$\pm$} 0.013 & 0.500 {\tiny$\pm$} 0.125 & -0.003 {\tiny$\pm$} 0.001 & 0.092 {\tiny$\pm$} 0.038 & 0.406 {\tiny$\pm$} 0.036 & 0.066 {\tiny$\pm$} 0.014 & 0.010 {\tiny$\pm$} 0.001 & 0.063 {\tiny$\pm$} 0.015 & 0.122 {\tiny$\pm$} 0.066 & 0.196 {\tiny$\pm$} 0.207 \\
\bottomrule
\end{tabular}
}
\caption{Threshold sweep on the reported retrained relaxed samples, threshold-only decoding (no post-processing rule); mean $\pm$ sample standard deviation over seeds.}
\label{tab:threshold_sweep}
\end{table*}

\begin{table*}[p]
\centering
\scriptsize
\setlength{\tabcolsep}{3pt}
\fitwidth{\begin{tabular}{llrrrrrrrrrr}
\toprule
Dataset & $\vartheta$ & isolated vertices & all edges size $\ge2$ & $\Delta\rho$ & $W_1(k)$ & $W_1(e)$ & Intersect.\ WD & Tail gap & Node spec.\ WD & Edge spec.\ WD & Feature MMD \\
\midrule
Cora & 0.30 & 0.000 {\tiny$\pm$} 0.000 & 1.000 {\tiny$\pm$} 0.000 & 0.001 {\tiny$\pm$} 0.001 & 0.058 {\tiny$\pm$} 0.024 & 0.145 {\tiny$\pm$} 0.058 & 0.020 {\tiny$\pm$} 0.013 & 0.005 {\tiny$\pm$} 0.001 & 0.004 {\tiny$\pm$} 0.004 & 0.038 {\tiny$\pm$} 0.015 & 0.105 {\tiny$\pm$} 0.024 \\
 & 0.40 & 0.000 {\tiny$\pm$} 0.000 & 1.000 {\tiny$\pm$} 0.000 & 0.001 {\tiny$\pm$} 0.001 & 0.057 {\tiny$\pm$} 0.023 & 0.143 {\tiny$\pm$} 0.055 & 0.020 {\tiny$\pm$} 0.013 & 0.005 {\tiny$\pm$} 0.001 & 0.004 {\tiny$\pm$} 0.004 & 0.038 {\tiny$\pm$} 0.015 & 0.103 {\tiny$\pm$} 0.022 \\
 & 0.50 & 0.000 {\tiny$\pm$} 0.000 & 1.000 {\tiny$\pm$} 0.000 & 0.001 {\tiny$\pm$} 0.001 & 0.057 {\tiny$\pm$} 0.023 & 0.143 {\tiny$\pm$} 0.055 & 0.020 {\tiny$\pm$} 0.013 & 0.005 {\tiny$\pm$} 0.001 & 0.004 {\tiny$\pm$} 0.004 & 0.038 {\tiny$\pm$} 0.015 & 0.103 {\tiny$\pm$} 0.022 \\
 & 0.55 & 0.000 {\tiny$\pm$} 0.000 & 1.000 {\tiny$\pm$} 0.000 & 0.001 {\tiny$\pm$} 0.001 & 0.057 {\tiny$\pm$} 0.023 & 0.143 {\tiny$\pm$} 0.055 & 0.020 {\tiny$\pm$} 0.013 & 0.005 {\tiny$\pm$} 0.001 & 0.004 {\tiny$\pm$} 0.004 & 0.038 {\tiny$\pm$} 0.015 & 0.103 {\tiny$\pm$} 0.022 \\
 & 0.60 & 0.000 {\tiny$\pm$} 0.000 & 1.000 {\tiny$\pm$} 0.000 & 0.001 {\tiny$\pm$} 0.001 & 0.057 {\tiny$\pm$} 0.023 & 0.143 {\tiny$\pm$} 0.055 & 0.020 {\tiny$\pm$} 0.013 & 0.005 {\tiny$\pm$} 0.001 & 0.004 {\tiny$\pm$} 0.004 & 0.038 {\tiny$\pm$} 0.015 & 0.103 {\tiny$\pm$} 0.022 \\
 & 0.70 & 0.000 {\tiny$\pm$} 0.000 & 1.000 {\tiny$\pm$} 0.000 & 0.001 {\tiny$\pm$} 0.001 & 0.057 {\tiny$\pm$} 0.023 & 0.143 {\tiny$\pm$} 0.055 & 0.020 {\tiny$\pm$} 0.013 & 0.005 {\tiny$\pm$} 0.001 & 0.004 {\tiny$\pm$} 0.004 & 0.038 {\tiny$\pm$} 0.015 & 0.103 {\tiny$\pm$} 0.022 \\
\midrule
CiteSeer & 0.30 & 0.000 {\tiny$\pm$} 0.000 & 1.000 {\tiny$\pm$} 0.000 & 0.001 {\tiny$\pm$} 0.001 & 0.096 {\tiny$\pm$} 0.033 & 0.352 {\tiny$\pm$} 0.042 & 0.051 {\tiny$\pm$} 0.023 & 0.012 {\tiny$\pm$} 0.004 & 0.022 {\tiny$\pm$} 0.012 & 0.058 {\tiny$\pm$} 0.008 & 0.041 {\tiny$\pm$} 0.028 \\
 & 0.40 & 0.000 {\tiny$\pm$} 0.000 & 1.000 {\tiny$\pm$} 0.000 & 0.001 {\tiny$\pm$} 0.001 & 0.094 {\tiny$\pm$} 0.031 & 0.352 {\tiny$\pm$} 0.042 & 0.050 {\tiny$\pm$} 0.023 & 0.012 {\tiny$\pm$} 0.004 & 0.021 {\tiny$\pm$} 0.010 & 0.058 {\tiny$\pm$} 0.008 & 0.039 {\tiny$\pm$} 0.027 \\
 & 0.50 & 0.000 {\tiny$\pm$} 0.000 & 1.000 {\tiny$\pm$} 0.000 & 0.001 {\tiny$\pm$} 0.001 & 0.094 {\tiny$\pm$} 0.031 & 0.352 {\tiny$\pm$} 0.042 & 0.050 {\tiny$\pm$} 0.023 & 0.012 {\tiny$\pm$} 0.004 & 0.021 {\tiny$\pm$} 0.010 & 0.058 {\tiny$\pm$} 0.008 & 0.039 {\tiny$\pm$} 0.027 \\
 & 0.55 & 0.000 {\tiny$\pm$} 0.000 & 1.000 {\tiny$\pm$} 0.000 & 0.001 {\tiny$\pm$} 0.001 & 0.094 {\tiny$\pm$} 0.031 & 0.352 {\tiny$\pm$} 0.042 & 0.050 {\tiny$\pm$} 0.023 & 0.012 {\tiny$\pm$} 0.004 & 0.021 {\tiny$\pm$} 0.010 & 0.058 {\tiny$\pm$} 0.008 & 0.039 {\tiny$\pm$} 0.027 \\
 & 0.60 & 0.000 {\tiny$\pm$} 0.000 & 1.000 {\tiny$\pm$} 0.000 & 0.001 {\tiny$\pm$} 0.001 & 0.094 {\tiny$\pm$} 0.031 & 0.352 {\tiny$\pm$} 0.042 & 0.050 {\tiny$\pm$} 0.023 & 0.012 {\tiny$\pm$} 0.004 & 0.021 {\tiny$\pm$} 0.010 & 0.058 {\tiny$\pm$} 0.008 & 0.039 {\tiny$\pm$} 0.027 \\
 & 0.70 & 0.000 {\tiny$\pm$} 0.000 & 1.000 {\tiny$\pm$} 0.000 & 0.001 {\tiny$\pm$} 0.001 & 0.094 {\tiny$\pm$} 0.031 & 0.352 {\tiny$\pm$} 0.042 & 0.050 {\tiny$\pm$} 0.023 & 0.012 {\tiny$\pm$} 0.004 & 0.021 {\tiny$\pm$} 0.010 & 0.058 {\tiny$\pm$} 0.008 & 0.039 {\tiny$\pm$} 0.027 \\
\midrule
Actor & 0.30 & 0.000 {\tiny$\pm$} 0.000 & 1.000 {\tiny$\pm$} 0.000 & -0.000 {\tiny$\pm$} 0.001 & 0.036 {\tiny$\pm$} 0.022 & 0.304 {\tiny$\pm$} 0.038 & 0.038 {\tiny$\pm$} 0.016 & 0.003 {\tiny$\pm$} 0.002 & 0.039 {\tiny$\pm$} 0.011 & 0.068 {\tiny$\pm$} 0.050 & 0.113 {\tiny$\pm$} 0.095 \\
 & 0.40 & 0.000 {\tiny$\pm$} 0.000 & 1.000 {\tiny$\pm$} 0.000 & -0.000 {\tiny$\pm$} 0.001 & 0.036 {\tiny$\pm$} 0.022 & 0.304 {\tiny$\pm$} 0.038 & 0.038 {\tiny$\pm$} 0.016 & 0.003 {\tiny$\pm$} 0.002 & 0.039 {\tiny$\pm$} 0.011 & 0.068 {\tiny$\pm$} 0.050 & 0.113 {\tiny$\pm$} 0.095 \\
 & 0.50 & 0.000 {\tiny$\pm$} 0.000 & 1.000 {\tiny$\pm$} 0.000 & -0.000 {\tiny$\pm$} 0.001 & 0.036 {\tiny$\pm$} 0.022 & 0.304 {\tiny$\pm$} 0.038 & 0.038 {\tiny$\pm$} 0.016 & 0.003 {\tiny$\pm$} 0.002 & 0.039 {\tiny$\pm$} 0.011 & 0.068 {\tiny$\pm$} 0.050 & 0.113 {\tiny$\pm$} 0.095 \\
 & 0.55 & 0.000 {\tiny$\pm$} 0.000 & 1.000 {\tiny$\pm$} 0.000 & -0.000 {\tiny$\pm$} 0.001 & 0.036 {\tiny$\pm$} 0.022 & 0.304 {\tiny$\pm$} 0.038 & 0.038 {\tiny$\pm$} 0.016 & 0.003 {\tiny$\pm$} 0.002 & 0.039 {\tiny$\pm$} 0.011 & 0.068 {\tiny$\pm$} 0.050 & 0.113 {\tiny$\pm$} 0.095 \\
 & 0.60 & 0.000 {\tiny$\pm$} 0.000 & 1.000 {\tiny$\pm$} 0.000 & -0.000 {\tiny$\pm$} 0.001 & 0.036 {\tiny$\pm$} 0.022 & 0.304 {\tiny$\pm$} 0.038 & 0.038 {\tiny$\pm$} 0.016 & 0.003 {\tiny$\pm$} 0.002 & 0.039 {\tiny$\pm$} 0.011 & 0.068 {\tiny$\pm$} 0.050 & 0.113 {\tiny$\pm$} 0.095 \\
 & 0.70 & 0.000 {\tiny$\pm$} 0.000 & 1.000 {\tiny$\pm$} 0.000 & -0.000 {\tiny$\pm$} 0.001 & 0.037 {\tiny$\pm$} 0.020 & 0.308 {\tiny$\pm$} 0.031 & 0.037 {\tiny$\pm$} 0.017 & 0.003 {\tiny$\pm$} 0.002 & 0.039 {\tiny$\pm$} 0.012 & 0.060 {\tiny$\pm$} 0.036 & 0.111 {\tiny$\pm$} 0.095 \\
\midrule
H.-Comm. & 0.30 & 0.000 {\tiny$\pm$} 0.000 & 1.000 {\tiny$\pm$} 0.000 & -0.007 {\tiny$\pm$} 0.005 & 0.128 {\tiny$\pm$} 0.033 & 0.638 {\tiny$\pm$} 0.222 & 0.225 {\tiny$\pm$} 0.056 & 0.020 {\tiny$\pm$} 0.017 & 0.068 {\tiny$\pm$} 0.027 & 0.088 {\tiny$\pm$} 0.018 & 0.069 {\tiny$\pm$} 0.055 \\
 & 0.40 & 0.000 {\tiny$\pm$} 0.000 & 1.000 {\tiny$\pm$} 0.000 & -0.007 {\tiny$\pm$} 0.005 & 0.128 {\tiny$\pm$} 0.033 & 0.638 {\tiny$\pm$} 0.222 & 0.225 {\tiny$\pm$} 0.056 & 0.020 {\tiny$\pm$} 0.017 & 0.068 {\tiny$\pm$} 0.027 & 0.088 {\tiny$\pm$} 0.018 & 0.069 {\tiny$\pm$} 0.055 \\
 & 0.50 & 0.000 {\tiny$\pm$} 0.000 & 1.000 {\tiny$\pm$} 0.000 & -0.007 {\tiny$\pm$} 0.005 & 0.128 {\tiny$\pm$} 0.033 & 0.638 {\tiny$\pm$} 0.222 & 0.225 {\tiny$\pm$} 0.056 & 0.020 {\tiny$\pm$} 0.017 & 0.068 {\tiny$\pm$} 0.027 & 0.088 {\tiny$\pm$} 0.018 & 0.069 {\tiny$\pm$} 0.055 \\
 & 0.55 & 0.000 {\tiny$\pm$} 0.000 & 1.000 {\tiny$\pm$} 0.000 & -0.007 {\tiny$\pm$} 0.005 & 0.128 {\tiny$\pm$} 0.033 & 0.638 {\tiny$\pm$} 0.222 & 0.225 {\tiny$\pm$} 0.056 & 0.020 {\tiny$\pm$} 0.017 & 0.068 {\tiny$\pm$} 0.027 & 0.088 {\tiny$\pm$} 0.018 & 0.069 {\tiny$\pm$} 0.055 \\
 & 0.60 & 0.000 {\tiny$\pm$} 0.000 & 1.000 {\tiny$\pm$} 0.000 & -0.007 {\tiny$\pm$} 0.005 & 0.128 {\tiny$\pm$} 0.033 & 0.638 {\tiny$\pm$} 0.222 & 0.225 {\tiny$\pm$} 0.056 & 0.020 {\tiny$\pm$} 0.017 & 0.068 {\tiny$\pm$} 0.027 & 0.088 {\tiny$\pm$} 0.018 & 0.069 {\tiny$\pm$} 0.055 \\
 & 0.70 & 0.000 {\tiny$\pm$} 0.000 & 1.000 {\tiny$\pm$} 0.000 & -0.007 {\tiny$\pm$} 0.005 & 0.128 {\tiny$\pm$} 0.033 & 0.638 {\tiny$\pm$} 0.222 & 0.225 {\tiny$\pm$} 0.056 & 0.020 {\tiny$\pm$} 0.017 & 0.068 {\tiny$\pm$} 0.027 & 0.088 {\tiny$\pm$} 0.018 & 0.069 {\tiny$\pm$} 0.055 \\
\midrule
DBLP & 0.30 & 0.000 {\tiny$\pm$} 0.000 & 1.000 {\tiny$\pm$} 0.000 & -0.001 {\tiny$\pm$} 0.003 & 0.143 {\tiny$\pm$} 0.096 & 0.967 {\tiny$\pm$} 0.161 & 0.072 {\tiny$\pm$} 0.029 & 0.013 {\tiny$\pm$} 0.014 & 0.071 {\tiny$\pm$} 0.029 & 0.065 {\tiny$\pm$} 0.030 & 0.251 {\tiny$\pm$} 0.101 \\
 & 0.40 & 0.000 {\tiny$\pm$} 0.000 & 1.000 {\tiny$\pm$} 0.000 & -0.001 {\tiny$\pm$} 0.003 & 0.143 {\tiny$\pm$} 0.096 & 0.967 {\tiny$\pm$} 0.161 & 0.072 {\tiny$\pm$} 0.029 & 0.013 {\tiny$\pm$} 0.014 & 0.071 {\tiny$\pm$} 0.029 & 0.065 {\tiny$\pm$} 0.030 & 0.251 {\tiny$\pm$} 0.101 \\
 & 0.50 & 0.000 {\tiny$\pm$} 0.000 & 1.000 {\tiny$\pm$} 0.000 & -0.001 {\tiny$\pm$} 0.003 & 0.143 {\tiny$\pm$} 0.096 & 0.967 {\tiny$\pm$} 0.161 & 0.072 {\tiny$\pm$} 0.029 & 0.013 {\tiny$\pm$} 0.014 & 0.071 {\tiny$\pm$} 0.029 & 0.065 {\tiny$\pm$} 0.030 & 0.251 {\tiny$\pm$} 0.101 \\
 & 0.55 & 0.000 {\tiny$\pm$} 0.000 & 1.000 {\tiny$\pm$} 0.000 & -0.001 {\tiny$\pm$} 0.003 & 0.143 {\tiny$\pm$} 0.096 & 0.967 {\tiny$\pm$} 0.161 & 0.072 {\tiny$\pm$} 0.029 & 0.013 {\tiny$\pm$} 0.014 & 0.071 {\tiny$\pm$} 0.029 & 0.065 {\tiny$\pm$} 0.030 & 0.251 {\tiny$\pm$} 0.101 \\
 & 0.60 & 0.000 {\tiny$\pm$} 0.000 & 1.000 {\tiny$\pm$} 0.000 & -0.001 {\tiny$\pm$} 0.003 & 0.143 {\tiny$\pm$} 0.096 & 0.967 {\tiny$\pm$} 0.161 & 0.072 {\tiny$\pm$} 0.029 & 0.013 {\tiny$\pm$} 0.014 & 0.071 {\tiny$\pm$} 0.029 & 0.065 {\tiny$\pm$} 0.030 & 0.251 {\tiny$\pm$} 0.101 \\
 & 0.70 & 0.000 {\tiny$\pm$} 0.000 & 1.000 {\tiny$\pm$} 0.000 & -0.001 {\tiny$\pm$} 0.003 & 0.143 {\tiny$\pm$} 0.096 & 0.967 {\tiny$\pm$} 0.161 & 0.072 {\tiny$\pm$} 0.029 & 0.013 {\tiny$\pm$} 0.014 & 0.071 {\tiny$\pm$} 0.029 & 0.065 {\tiny$\pm$} 0.030 & 0.251 {\tiny$\pm$} 0.101 \\
\midrule
Twitch & 0.30 & 0.000 {\tiny$\pm$} 0.000 & 1.000 {\tiny$\pm$} 0.000 & -0.001 {\tiny$\pm$} 0.001 & 0.102 {\tiny$\pm$} 0.034 & 0.362 {\tiny$\pm$} 0.051 & 0.070 {\tiny$\pm$} 0.013 & 0.010 {\tiny$\pm$} 0.001 & 0.068 {\tiny$\pm$} 0.009 & 0.085 {\tiny$\pm$} 0.054 & 0.108 {\tiny$\pm$} 0.064 \\
 & 0.40 & 0.000 {\tiny$\pm$} 0.000 & 1.000 {\tiny$\pm$} 0.000 & -0.001 {\tiny$\pm$} 0.001 & 0.102 {\tiny$\pm$} 0.034 & 0.362 {\tiny$\pm$} 0.051 & 0.070 {\tiny$\pm$} 0.013 & 0.010 {\tiny$\pm$} 0.001 & 0.068 {\tiny$\pm$} 0.009 & 0.085 {\tiny$\pm$} 0.054 & 0.108 {\tiny$\pm$} 0.064 \\
 & 0.50 & 0.000 {\tiny$\pm$} 0.000 & 1.000 {\tiny$\pm$} 0.000 & -0.001 {\tiny$\pm$} 0.001 & 0.102 {\tiny$\pm$} 0.034 & 0.362 {\tiny$\pm$} 0.051 & 0.070 {\tiny$\pm$} 0.013 & 0.010 {\tiny$\pm$} 0.001 & 0.068 {\tiny$\pm$} 0.009 & 0.085 {\tiny$\pm$} 0.054 & 0.108 {\tiny$\pm$} 0.064 \\
 & 0.55 & 0.000 {\tiny$\pm$} 0.000 & 1.000 {\tiny$\pm$} 0.000 & -0.001 {\tiny$\pm$} 0.001 & 0.102 {\tiny$\pm$} 0.034 & 0.362 {\tiny$\pm$} 0.051 & 0.070 {\tiny$\pm$} 0.013 & 0.010 {\tiny$\pm$} 0.001 & 0.068 {\tiny$\pm$} 0.009 & 0.085 {\tiny$\pm$} 0.054 & 0.108 {\tiny$\pm$} 0.064 \\
 & 0.60 & 0.000 {\tiny$\pm$} 0.000 & 1.000 {\tiny$\pm$} 0.000 & -0.001 {\tiny$\pm$} 0.001 & 0.102 {\tiny$\pm$} 0.034 & 0.362 {\tiny$\pm$} 0.051 & 0.070 {\tiny$\pm$} 0.013 & 0.010 {\tiny$\pm$} 0.001 & 0.068 {\tiny$\pm$} 0.009 & 0.085 {\tiny$\pm$} 0.054 & 0.108 {\tiny$\pm$} 0.064 \\
 & 0.70 & 0.000 {\tiny$\pm$} 0.000 & 1.000 {\tiny$\pm$} 0.000 & -0.001 {\tiny$\pm$} 0.001 & 0.102 {\tiny$\pm$} 0.034 & 0.362 {\tiny$\pm$} 0.051 & 0.070 {\tiny$\pm$} 0.013 & 0.010 {\tiny$\pm$} 0.001 & 0.068 {\tiny$\pm$} 0.009 & 0.085 {\tiny$\pm$} 0.054 & 0.108 {\tiny$\pm$} 0.064 \\
\bottomrule
\end{tabular}
}
\caption{Threshold sweep on the reported retrained relaxed samples, full decoding (threshold followed by both rules); mean $\pm$ sample standard deviation over seeds.}
\label{tab:threshold_sweep_full}
\end{table*}

\begin{table*}[p]
\centering
\scriptsize
\fitwidth{\begin{tabular}{llrrrr}
\toprule
Dataset & Metric & threshold only & + size rule & full & $\Delta/\mathrm{sd}$ \\
\midrule
Cora & $\Delta\rho$ & -0.001 {\tiny$\pm$} 0.002 & -0.000 {\tiny$\pm$} 0.002 & 0.001 {\tiny$\pm$} 0.001 & +2.76 \\
 & $\Delta e$ & -0.088 {\tiny$\pm$} 0.116 & -0.002 {\tiny$\pm$} 0.102 & 0.073 {\tiny$\pm$} 0.059 & +2.76 \\
 & $\Delta k$ & -0.035 {\tiny$\pm$} 0.045 & -0.001 {\tiny$\pm$} 0.040 & 0.029 {\tiny$\pm$} 0.023 & +2.76 \\
 & $W_1(e)$ & 0.225 {\tiny$\pm$} 0.022 & 0.138 {\tiny$\pm$} 0.038 & 0.143 {\tiny$\pm$} 0.055 & -1.49 \\
 & $W_1(k)$ & 0.072 {\tiny$\pm$} 0.011 & 0.055 {\tiny$\pm$} 0.015 & 0.057 {\tiny$\pm$} 0.023 & -0.64 \\
 & Tail gap & 0.006 {\tiny$\pm$} 0.001 & 0.005 {\tiny$\pm$} 0.001 & 0.005 {\tiny$\pm$} 0.001 & -0.10 \\
 & Intersect.\ WD & 0.016 {\tiny$\pm$} 0.009 & 0.020 {\tiny$\pm$} 0.013 & 0.020 {\tiny$\pm$} 0.013 & +0.25 \\
 & Node spec.\ WD & 0.001 {\tiny$\pm$} 0.001 & 0.004 {\tiny$\pm$} 0.004 & 0.004 {\tiny$\pm$} 0.004 & +0.61 \\
 & Edge spec.\ WD & 0.068 {\tiny$\pm$} 0.004 & 0.038 {\tiny$\pm$} 0.015 & 0.038 {\tiny$\pm$} 0.015 & -2.00 \\
 & Feature MMD & 0.214 {\tiny$\pm$} 0.133 & 0.212 {\tiny$\pm$} 0.134 & 0.103 {\tiny$\pm$} 0.022 & -4.98 \\
\midrule
CiteSeer & $\Delta\rho$ & -0.003 {\tiny$\pm$} 0.001 & -0.001 {\tiny$\pm$} 0.001 & 0.001 {\tiny$\pm$} 0.001 & +3.26 \\
 & $\Delta e$ & -0.197 {\tiny$\pm$} 0.092 & -0.062 {\tiny$\pm$} 0.079 & 0.053 {\tiny$\pm$} 0.077 & +3.26 \\
 & $\Delta k$ & -0.077 {\tiny$\pm$} 0.036 & -0.024 {\tiny$\pm$} 0.031 & 0.021 {\tiny$\pm$} 0.030 & +3.26 \\
 & $W_1(e)$ & 0.428 {\tiny$\pm$} 0.015 & 0.293 {\tiny$\pm$} 0.033 & 0.352 {\tiny$\pm$} 0.042 & -1.83 \\
 & $W_1(k)$ & 0.079 {\tiny$\pm$} 0.019 & 0.055 {\tiny$\pm$} 0.030 & 0.094 {\tiny$\pm$} 0.031 & +0.51 \\
 & Tail gap & 0.012 {\tiny$\pm$} 0.004 & 0.012 {\tiny$\pm$} 0.004 & 0.012 {\tiny$\pm$} 0.004 & +0.00 \\
 & Intersect.\ WD & 0.037 {\tiny$\pm$} 0.025 & 0.050 {\tiny$\pm$} 0.023 & 0.050 {\tiny$\pm$} 0.023 & +0.57 \\
 & Node spec.\ WD & 0.012 {\tiny$\pm$} 0.003 & 0.022 {\tiny$\pm$} 0.010 & 0.021 {\tiny$\pm$} 0.010 & +0.85 \\
 & Edge spec.\ WD & 0.067 {\tiny$\pm$} 0.017 & 0.058 {\tiny$\pm$} 0.008 & 0.058 {\tiny$\pm$} 0.008 & -1.11 \\
 & Feature MMD & 0.089 {\tiny$\pm$} 0.086 & 0.088 {\tiny$\pm$} 0.085 & 0.039 {\tiny$\pm$} 0.027 & -1.83 \\
\midrule
Actor & $\Delta\rho$ & -0.002 {\tiny$\pm$} 0.001 & -0.001 {\tiny$\pm$} 0.001 & -0.000 {\tiny$\pm$} 0.001 & +2.15 \\
 & $\Delta e$ & -0.126 {\tiny$\pm$} 0.075 & -0.082 {\tiny$\pm$} 0.064 & -0.022 {\tiny$\pm$} 0.048 & +2.15 \\
 & $\Delta k$ & -0.031 {\tiny$\pm$} 0.019 & -0.020 {\tiny$\pm$} 0.016 & -0.005 {\tiny$\pm$} 0.012 & +2.15 \\
 & $W_1(e)$ & 0.358 {\tiny$\pm$} 0.073 & 0.312 {\tiny$\pm$} 0.087 & 0.304 {\tiny$\pm$} 0.038 & -1.42 \\
 & $W_1(k)$ & 0.040 {\tiny$\pm$} 0.012 & 0.028 {\tiny$\pm$} 0.017 & 0.036 {\tiny$\pm$} 0.022 & -0.14 \\
 & Tail gap & 0.003 {\tiny$\pm$} 0.002 & 0.003 {\tiny$\pm$} 0.002 & 0.003 {\tiny$\pm$} 0.002 & +0.00 \\
 & Intersect.\ WD & 0.032 {\tiny$\pm$} 0.019 & 0.038 {\tiny$\pm$} 0.016 & 0.038 {\tiny$\pm$} 0.016 & +0.38 \\
 & Node spec.\ WD & 0.025 {\tiny$\pm$} 0.015 & 0.039 {\tiny$\pm$} 0.010 & 0.039 {\tiny$\pm$} 0.011 & +1.33 \\
 & Edge spec.\ WD & 0.034 {\tiny$\pm$} 0.011 & 0.069 {\tiny$\pm$} 0.050 & 0.068 {\tiny$\pm$} 0.050 & +0.68 \\
 & Feature MMD & 0.107 {\tiny$\pm$} 0.096 & 0.107 {\tiny$\pm$} 0.096 & 0.113 {\tiny$\pm$} 0.095 & +0.06 \\
\bottomrule
\end{tabular}
}
\caption{All metrics for the three decoding stages at $\vartheta=0.55$ (mean $\pm$ sample standard deviation over seeds) and the full-minus-threshold-only difference in units of the across-seed standard deviation of the full stage. Cora, CiteSeer and Actor; continued below.}
\label{tab:stage_metrics}
\end{table*}

\begin{table*}[p]
\ContinuedFloat
\centering
\scriptsize
\fitwidth{\begin{tabular}{llrrrr}
\toprule
Dataset & Metric & threshold only & + size rule & full & $\Delta/\mathrm{sd}$ \\
\midrule
H.-Comm. & $\Delta\rho$ & -0.009 {\tiny$\pm$} 0.005 & -0.008 {\tiny$\pm$} 0.005 & -0.007 {\tiny$\pm$} 0.005 & +0.35 \\
 & $\Delta e$ & -0.563 {\tiny$\pm$} 0.319 & -0.514 {\tiny$\pm$} 0.328 & -0.459 {\tiny$\pm$} 0.297 & +0.35 \\
 & $\Delta k$ & -0.141 {\tiny$\pm$} 0.080 & -0.128 {\tiny$\pm$} 0.082 & -0.115 {\tiny$\pm$} 0.074 & +0.35 \\
 & $W_1(e)$ & 0.721 {\tiny$\pm$} 0.253 & 0.672 {\tiny$\pm$} 0.238 & 0.638 {\tiny$\pm$} 0.222 & -0.37 \\
 & $W_1(k)$ & 0.141 {\tiny$\pm$} 0.048 & 0.132 {\tiny$\pm$} 0.045 & 0.128 {\tiny$\pm$} 0.033 & -0.39 \\
 & Tail gap & 0.021 {\tiny$\pm$} 0.016 & 0.020 {\tiny$\pm$} 0.017 & 0.020 {\tiny$\pm$} 0.017 & -0.06 \\
 & Intersect.\ WD & 0.214 {\tiny$\pm$} 0.056 & 0.225 {\tiny$\pm$} 0.056 & 0.225 {\tiny$\pm$} 0.056 & +0.20 \\
 & Node spec.\ WD & 0.064 {\tiny$\pm$} 0.028 & 0.068 {\tiny$\pm$} 0.028 & 0.068 {\tiny$\pm$} 0.027 & +0.15 \\
 & Edge spec.\ WD & 0.089 {\tiny$\pm$} 0.011 & 0.088 {\tiny$\pm$} 0.018 & 0.088 {\tiny$\pm$} 0.018 & -0.05 \\
 & Feature MMD & 0.058 {\tiny$\pm$} 0.040 & 0.058 {\tiny$\pm$} 0.040 & 0.069 {\tiny$\pm$} 0.055 & +0.21 \\
\midrule
DBLP & $\Delta\rho$ & -0.006 {\tiny$\pm$} 0.003 & -0.004 {\tiny$\pm$} 0.003 & -0.001 {\tiny$\pm$} 0.003 & +1.91 \\
 & $\Delta e$ & -0.400 {\tiny$\pm$} 0.220 & -0.268 {\tiny$\pm$} 0.220 & -0.044 {\tiny$\pm$} 0.187 & +1.91 \\
 & $\Delta k$ & -0.113 {\tiny$\pm$} 0.062 & -0.075 {\tiny$\pm$} 0.062 & -0.012 {\tiny$\pm$} 0.052 & +1.91 \\
 & $W_1(e)$ & 1.015 {\tiny$\pm$} 0.122 & 0.900 {\tiny$\pm$} 0.154 & 0.967 {\tiny$\pm$} 0.161 & -0.30 \\
 & $W_1(k)$ & 0.117 {\tiny$\pm$} 0.108 & 0.097 {\tiny$\pm$} 0.089 & 0.143 {\tiny$\pm$} 0.096 & +0.27 \\
 & Tail gap & 0.013 {\tiny$\pm$} 0.014 & 0.013 {\tiny$\pm$} 0.014 & 0.013 {\tiny$\pm$} 0.014 & +0.00 \\
 & Intersect.\ WD & 0.064 {\tiny$\pm$} 0.039 & 0.072 {\tiny$\pm$} 0.029 & 0.072 {\tiny$\pm$} 0.029 & +0.29 \\
 & Node spec.\ WD & 0.063 {\tiny$\pm$} 0.038 & 0.068 {\tiny$\pm$} 0.028 & 0.071 {\tiny$\pm$} 0.029 & +0.26 \\
 & Edge spec.\ WD & 0.126 {\tiny$\pm$} 0.086 & 0.065 {\tiny$\pm$} 0.030 & 0.065 {\tiny$\pm$} 0.030 & -2.04 \\
 & Feature MMD & 0.300 {\tiny$\pm$} 0.140 & 0.299 {\tiny$\pm$} 0.140 & 0.251 {\tiny$\pm$} 0.101 & -0.49 \\
\midrule
Twitch & $\Delta\rho$ & -0.003 {\tiny$\pm$} 0.001 & -0.002 {\tiny$\pm$} 0.001 & -0.001 {\tiny$\pm$} 0.001 & +1.45 \\
 & $\Delta e$ & -0.171 {\tiny$\pm$} 0.036 & -0.134 {\tiny$\pm$} 0.032 & -0.054 {\tiny$\pm$} 0.081 & +1.45 \\
 & $\Delta k$ & -0.043 {\tiny$\pm$} 0.009 & -0.034 {\tiny$\pm$} 0.008 & -0.013 {\tiny$\pm$} 0.020 & +1.45 \\
 & $W_1(e)$ & 0.406 {\tiny$\pm$} 0.036 & 0.370 {\tiny$\pm$} 0.040 & 0.362 {\tiny$\pm$} 0.051 & -0.87 \\
 & $W_1(k)$ & 0.092 {\tiny$\pm$} 0.038 & 0.085 {\tiny$\pm$} 0.039 & 0.102 {\tiny$\pm$} 0.034 & +0.29 \\
 & Tail gap & 0.010 {\tiny$\pm$} 0.001 & 0.010 {\tiny$\pm$} 0.001 & 0.010 {\tiny$\pm$} 0.001 & +0.00 \\
 & Intersect.\ WD & 0.066 {\tiny$\pm$} 0.014 & 0.070 {\tiny$\pm$} 0.013 & 0.070 {\tiny$\pm$} 0.013 & +0.34 \\
 & Node spec.\ WD & 0.063 {\tiny$\pm$} 0.015 & 0.066 {\tiny$\pm$} 0.008 & 0.068 {\tiny$\pm$} 0.009 & +0.55 \\
 & Edge spec.\ WD & 0.122 {\tiny$\pm$} 0.066 & 0.085 {\tiny$\pm$} 0.054 & 0.085 {\tiny$\pm$} 0.054 & -0.69 \\
 & Feature MMD & 0.196 {\tiny$\pm$} 0.207 & 0.196 {\tiny$\pm$} 0.207 & 0.108 {\tiny$\pm$} 0.064 & -1.36 \\
\bottomrule
\end{tabular}
}
\caption{Continued: House-Committees, DBLP and Twitch. The decoding stages and uncertainty convention are as above.}
\end{table*}

\begin{table*}[p]
\centering
\scriptsize
\setlength{\tabcolsep}{3pt}
\fitwidth{\begin{tabular}{lllrrrrrrrrrr}
\toprule
Dataset & Model & Post-proc. & isolated & $\Delta e$ & $\Delta k$ & $W_1(e)$ & $W_1(k)$ & Tail gap & Int.\ WD & Node sp. & Edge sp. & MMD \\
\midrule
Cora & ER-HG & none & 0.315 & 0.073 & 0.029 & 0.620 & 0.620 & 0.003 & 0.109 & 0.229 & 0.167 & 0.751 \\
 & ER-HG & size rule & 0.283 & 0.335 & 0.131 & 0.358 & 0.657 & 0.005 & 0.133 & 0.251 & 0.212 & 0.760 \\
 & ER-HG & both rules & 0.000 & 1.060 & 0.414 & 1.057 & 0.417 & 0.005 & 0.133 & 0.202 & 0.210 & 1.106 \\
 & HERGM & none & 0.362 & 0.045 & 0.018 & 0.738 & 0.697 & 0.009 & 0.135 & 0.278 & 0.156 & 0.866 \\
 & HERGM & size rule & 0.316 & 0.352 & 0.137 & 0.432 & 0.725 & 0.010 & 0.163 & 0.277 & 0.199 & 0.874 \\
 & HERGM & both rules & 0.000 & 1.160 & 0.453 & 1.157 & 0.452 & 0.010 & 0.163 & 0.215 & 0.196 & 1.071 \\
\midrule
CiteSeer & ER-HG & none & 0.295 & 0.086 & 0.034 & 0.728 & 0.551 & 0.004 & 0.141 & 0.253 & 0.170 & 0.169 \\
 & ER-HG & size rule & 0.268 & 0.309 & 0.121 & 0.505 & 0.584 & 0.004 & 0.162 & 0.268 & 0.202 & 0.174 \\
 & ER-HG & both rules & 0.000 & 0.996 & 0.389 & 1.135 & 0.430 & 0.004 & 0.162 & 0.219 & 0.201 & 0.639 \\
 & HERGM & none & 0.329 & 0.084 & 0.033 & 0.820 & 0.615 & 0.007 & 0.166 & 0.301 & 0.144 & 0.294 \\
 & HERGM & size rule & 0.290 & 0.404 & 0.158 & 0.500 & 0.661 & 0.008 & 0.199 & 0.302 & 0.192 & 0.297 \\
 & HERGM & both rules & 0.000 & 1.148 & 0.448 & 1.243 & 0.486 & 0.008 & 0.199 & 0.241 & 0.192 & 0.648 \\
\midrule
Actor & ER-HG & none & 0.332 & 0.075 & 0.019 & 0.637 & 0.647 & 0.031 & 0.302 & 0.347 & 0.318 & 0.293 \\
 & ER-HG & size rule & 0.322 & 0.155 & 0.039 & 0.537 & 0.649 & 0.032 & 0.312 & 0.340 & 0.332 & 0.293 \\
 & ER-HG & both rules & 0.000 & 1.442 & 0.360 & 1.421 & 0.355 & 0.032 & 0.312 & 0.267 & 0.329 & 1.059 \\
 & HERGM & none & 0.354 & 0.088 & 0.022 & 0.808 & 0.696 & 0.044 & 0.326 & 0.382 & 0.295 & 0.593 \\
 & HERGM & size rule & 0.338 & 0.273 & 0.068 & 0.621 & 0.718 & 0.046 & 0.353 & 0.380 & 0.322 & 0.593 \\
 & HERGM & both rules & 0.000 & 1.624 & 0.406 & 1.637 & 0.409 & 0.046 & 0.353 & 0.297 & 0.315 & 1.017 \\
\midrule
H.-Comm. & ER-HG & none & 0.121 & 0.286 & 0.071 & 1.565 & 0.330 & 0.109 & 0.607 & 0.169 & 0.184 & 0.274 \\
 & ER-HG & size rule & 0.121 & 0.291 & 0.073 & 1.560 & 0.331 & 0.109 & 0.607 & 0.169 & 0.184 & 0.274 \\
 & ER-HG & both rules & 0.000 & 0.775 & 0.194 & 1.638 & 0.231 & 0.109 & 0.607 & 0.158 & 0.185 & 0.428 \\
 & HERGM & none & 0.204 & -0.029 & -0.007 & 0.911 & 0.417 & 0.156 & 0.597 & 0.174 & 0.180 & 0.321 \\
 & HERGM & size rule & 0.199 & 0.059 & 0.015 & 0.823 & 0.424 & 0.160 & 0.623 & 0.176 & 0.179 & 0.321 \\
 & HERGM & both rules & 0.000 & 0.856 & 0.214 & 1.031 & 0.246 & 0.160 & 0.623 & 0.145 & 0.183 & 0.441 \\
\midrule
DBLP & ER-HG & none & 0.320 & 0.116 & 0.033 & 1.281 & 0.583 & 0.012 & 0.278 & 0.328 & 0.237 & 0.270 \\
 & ER-HG & size rule & 0.315 & 0.192 & 0.054 & 1.211 & 0.595 & 0.012 & 0.288 & 0.331 & 0.241 & 0.270 \\
 & ER-HG & both rules & 0.000 & 1.313 & 0.369 & 1.870 & 0.401 & 0.012 & 0.288 & 0.279 & 0.239 & 0.295 \\
 & HERGM & none & 0.339 & 0.125 & 0.035 & 0.974 & 0.655 & 0.043 & 0.297 & 0.357 & 0.207 & 0.168 \\
 & HERGM & size rule & 0.318 & 0.410 & 0.115 & 0.700 & 0.699 & 0.047 & 0.340 & 0.363 & 0.236 & 0.172 \\
 & HERGM & both rules & 0.000 & 1.542 & 0.434 & 1.759 & 0.495 & 0.047 & 0.340 & 0.292 & 0.234 & 0.345 \\
\midrule
Twitch & ER-HG & none & 0.303 & 0.123 & 0.031 & 0.539 & 0.550 & 0.023 & 0.324 & 0.309 & 0.268 & 0.056 \\
 & ER-HG & size rule & 0.295 & 0.188 & 0.047 & 0.474 & 0.551 & 0.024 & 0.330 & 0.309 & 0.273 & 0.057 \\
 & ER-HG & both rules & 0.000 & 1.368 & 0.342 & 1.294 & 0.311 & 0.024 & 0.330 & 0.265 & 0.271 & 0.748 \\
 & HERGM & none & 0.343 & 0.069 & 0.017 & 0.516 & 0.617 & 0.042 & 0.341 & 0.334 & 0.228 & 0.364 \\
 & HERGM & size rule & 0.322 & 0.293 & 0.073 & 0.292 & 0.630 & 0.043 & 0.369 & 0.334 & 0.261 & 0.365 \\
 & HERGM & both rules & 0.000 & 1.579 & 0.395 & 1.453 & 0.363 & 0.043 & 0.369 & 0.280 & 0.258 & 0.818 \\
\bottomrule
\end{tabular}
}
\caption{The two independent-incidence null models at the three post-processing stages (B3 reruns on the frozen batches): isolated-vertex fraction and the main metrics, means over seeds (standard deviations in the released CSV). The real datasets have $0.9$--$6.3\%$ isolated vertices.}
\label{tab:null_model_stages}
\end{table*}

\section{Interpreting the HCM-MCMC Comparison}
\label{app:hcm}

Each HCM-MCMC sample is produced by a pairwise-reshuffle chain started from a training subhypergraph, and preserves that subhypergraph's node-degree vector and hyperedge-size vector exactly; we verified this for all $144$ reported samples (each matches the degree and size sequences of its starting subhypergraph exactly, so the $W_1$ distances to that subhypergraph's degree and size distributions are zero). The non-zero $W_1(k)$ and $W_1(e)$ values in Table~\ref{tab:realdata_full_main} arise because the population of samples conditioned on training subhypergraphs is compared with a separate held-out set whose empirical marginals differ. The reference row of that table reports the same metrics between the first $n_{\mathrm{eval}}$ training and held-out samples; it is a contextual reference discrepancy, not a lower bound, since the finite choice of starting subhypergraphs, reweighting and stochastic sampling can give smaller or larger held-out error. Table~\ref{tab:win_counts} summarises, per metric, the number of datasets on which \HEDGE\ with the full projection beats HCM-MCMC, and on how many of those the threshold-only decoding also does.

\begin{table}[h]
\centering
\scriptsize
\begin{tabular}{lrrrr}
\toprule
Metric & HEDGE wins & ties & HCM wins & wins also under thr.-only \\
\midrule
$\Delta\rho$ & 1 & 3 & 2 & 0 \\
$\Delta e$ & 2 & 0 & 4 & 0 \\
$\Delta k$ & 2 & 0 & 4 & 0 \\
$W_1(e)$ & 2 & 0 & 4 & 1 \\
$W_1(k)$ & 4 & 0 & 2 & 3 \\
Overlap Tail Gap & 6 & 0 & 0 & 6 \\
Intersection WD & 5 & 1 & 0 & 5 \\
Node spec.\ WD & 6 & 0 & 0 & 6 \\
Edge spec.\ WD & 6 & 0 & 0 & 5 \\
Feature MMD & 4 & 0 & 2 & 3 \\
\bottomrule
\end{tabular}

\caption{Datasets (of six) won by \HEDGE\ against HCM-MCMC per metric, and the number of those wins that also hold with thresholding only.}
\label{tab:win_counts}
\end{table}

\section{Computational Cost}
\label{app:compute}

All measurements use a MacBook Pro with an Apple M4 Pro (10 performance and 4 efficiency cores, 24\,GB) running PyTorch 2.10 with the MPS backend, with one warm-up call and the median of five repetitions. Tables~\ref{tab:costs} and~\ref{tab:costs_sampling} separate the costs on the submitted checkpoints, whose architecture and moment path the reported runs share: the one-time eigendecomposition per training hypergraph, the conditional-moment computation at the horizon for one training hypergraph under the dense path and under the separable path, a full training step (batch of eight: moments and targets, then network forward and backward), reverse-SDE sampling of $64$ hypergraphs with $64$ steps, and post-processing. Sampling takes $5.4$--$8.3$\,s per batch of $64$ ($0.08$--$0.13$\,s per hypergraph) with a peak accelerator memory of $1.1$\,GiB; a training step takes $0.16$--$0.34$\,s, of which the dense moment path accounts for $0.11$--$0.28$\,s and the network for $0.04$--$0.06$\,s. The time-only loss weights add a one-time pass of $65$ moment computations per training hypergraph, under two minutes per run. HCM-MCMC generates a sample in $6$--$13$\,ms and ER-HG in $0.1$\,ms on the CPU.

\begin{table*}[p]
\centering
\scriptsize
\setlength{\tabcolsep}{4pt}
\fitwidth{\begin{tabular}{llrrrrrrr}
\toprule
Dataset & $n\times m$ & Eig.\ ($L_V,L_E$) & Moments, submitted & Dense op.\ mem.\ & Moments, separable & Train step (batch 8) & \quad moments + targets & \quad net fwd + bwd \\
\midrule
Cora & $64\times25$ & 0.39 ms & 29.85 ms & 19.5 MiB & 0.56 ms & 332 ms & 274 ms & 59.22 ms \\
CiteSeer & $64\times25$ & 0.36 ms & 33.03 ms & 19.5 MiB & 0.59 ms & 343 ms & 279 ms & 58.10 ms \\
Actor & $64\times16$ & 0.38 ms & 14.88 ms & 8.0 MiB & 0.57 ms & 175 ms & 122 ms & 43.91 ms \\
H.-Comm. & $64\times16$ & 0.29 ms & 13.04 ms & 8.0 MiB & 0.58 ms & 163 ms & 114 ms & 41.89 ms \\
DBLP & $64\times18$ & 0.37 ms & 16.89 ms & 10.1 MiB & 0.58 ms & 192 ms & 136 ms & 45.08 ms \\
Twitch & $64\times16$ & 0.32 ms & 16.31 ms & 8.0 MiB & 0.58 ms & 174 ms & 117 ms & 41.79 ms \\
\bottomrule
\end{tabular}
}
\caption{Training-side costs on the submitted checkpoints (seed 11): one-time eigendecomposition per training hypergraph, the conditional-moment computation at the horizon for one hypergraph under the dense diagonal-covariance path (with the dense operator's memory) and under the separable full-covariance path, and a training step of batch eight split into its moment-and-target and network parts. Medians over five repetitions.}
\label{tab:costs}
\end{table*}

\begin{table*}[p]
\centering
\scriptsize
\setlength{\tabcolsep}{4pt}
\fitwidth{\begin{tabular}{llrrrrrrr}
\toprule
Dataset & $n\times m$ & Sampling (64, MPS) & per hypergraph (MPS) & per hypergraph (CPU) & Post-proc.\ (64) & Peak MPS mem.\ & HCM-MCMC / sample & ER-HG / sample \\
\midrule
Cora & $64\times25$ & 8.26 s & 129 ms & 604 ms & 1.47 ms & 1.1 GiB & 6.1 ms & 0.1 ms \\
CiteSeer & $64\times25$ & 8.18 s & 128 ms & 601 ms & 1.64 ms & 1.1 GiB & 6.6 ms & 0.1 ms \\
Actor & $64\times16$ & 5.41 s & 84.6 ms & 359 ms & 0.99 ms & 1.1 GiB & 6.2 ms & 0.1 ms \\
H.-Comm. & $64\times16$ & 5.43 s & 84.8 ms & 371 ms & 0.84 ms & 1.1 GiB & 12.6 ms & 0.1 ms \\
DBLP & $64\times18$ & 5.90 s & 92.2 ms & 476 ms & 1.35 ms & 1.1 GiB & 6.5 ms & 0.1 ms \\
Twitch & $64\times16$ & 5.36 s & 83.8 ms & 367 ms & 0.93 ms & 1.1 GiB & 6.3 ms & 0.1 ms \\
\bottomrule
\end{tabular}
}
\caption{Sampling-side costs: reverse-SDE generation of 64 hypergraphs with 64 steps (MPS, and per hypergraph on MPS and CPU), post-processing of the 64 samples, peak MPS memory, and per-sample generation cost of the baselines on the CPU.}
\label{tab:costs_sampling}
\end{table*}

\clearpage
\section*{NeurIPS Paper Checklist}

\begin{enumerate}

\item {\bf Claims}
    \item[] Question: Do the main claims made in the abstract and introduction accurately reflect the paper's contributions and scope?
    \item[] Answer: \answerYes{}.
    \item[] Justification: The abstract and introduction state the main methodological contributions: incidence-space structured heat--OU diffusion, analytically specified conditional reverse-drift targets, approximated in the implementation, equivariant reverse-drift learning, stability analysis, and empirical comparison. The experimental claims are scoped to fixed-size incidence generation, distinguish comparisons that hold with both decodings from those that depend on the post-processing rule, and are supported by the ablations and real-data comparisons in Section~\ref{sec:experiments} and Appendix~\ref{sec:additional_numerics}.

\item {\bf Limitations}
    \item[] Question: Does the paper discuss the limitations of the work performed by the authors?
    \item[] Answer: \answerYes{}.
    \item[] Justification: Section~\ref{sec:limitations} discusses the main limitations of the work: the method as implemented is fixed-size; incidence matrices are unweighted and unattributed; generation starts from the OU base law rather than the finite-horizon mixture (Appendix~\ref{app:terminal_law}); the post-processing rule enforces a node-degree constraint the data do not have, and we report results with and without it (Appendix~\ref{app:projection}); the implementation used for the reported results has $O((nm)^2)$ cost per training example (Appendix~\ref{app:compute}); and configuration-style baselines remain stronger on marginal degree and size statistics.

\item {\bf Theory assumptions and proofs}
    \item[] Question: For each theoretical result, does the paper provide the full set of assumptions and a complete (and correct) proof?
    \item[] Answer: \answerYes{}.
    \item[] Justification: The main theoretical results are stated in Section~\ref{sec:method}, and complete proofs are provided in Appendix~\ref{app:theory_proofs}. The appendix states the assumptions used for the heat operator, conditional Gaussian law, reverse-time drift, stability bounds, Euler--Maruyama discretisation, total generation error, and equivariance.

\item {\bf Experimental result reproducibility}
    \item[] Question: Does the paper fully disclose all the information needed to reproduce the main experimental results of the paper to the extent that it affects the main claims and/or conclusions of the paper (regardless of whether the code and data are provided or not)?
    \item[] Answer: \answerYes{}.
    \item[] Justification: The paper specifies the fixed-size subhypergraph protocol, datasets, baselines, metrics, number of seeds, and the main evaluation pipeline in Section~\ref{sec:experiments} and Appendix~\ref{sec:additional_numerics}. Reproducibility is further supported by the release of the codebase used in the paper.
    
\item {\bf Open access to data and code}
    \item[] Question: Does the paper provide open access to the data and code, with sufficient instructions to faithfully reproduce the main experimental results, as described in supplemental material?
    \item[] Answer: \answerYes{}.
    \item[] Justification: The experiments use public benchmark datasets and standard baselines cited in the paper. The code, configuration files, scripts and the scripts that regenerate every table of this paper are released at \url{https://github.com/chris-nemeth/HEDGE}.

\item {\bf Experimental setting/details}
    \item[] Question: Does the paper specify all the training and test details (e.g., data splits, hyperparameters, how they were chosen, type of optimizer) necessary to understand the results?
    \item[] Answer: \answerYes{}.
    \item[] Justification: The main paper specifies the task, representation, fixed-size train/test protocol, datasets, comparators, metrics, and number of seeds. Full implementation-level details, including the conditional-moment computation actually used, the post-processing rule, hyperparameters and baseline settings, are given in Appendix~\ref{app:implementation}; run configurations are in the released code.

\item {\bf Experiment statistical significance}
    \item[] Question: Does the paper report error bars suitably and correctly defined or other appropriate information about the statistical significance of the experiments?
    \item[] Answer: \answerYes{}.
    \item[] Justification: The per-regime simulated ablations report mean $\pm$ standard error across seeds; the summary ablation tables report mean $\pm$ standard error across the four regime means after averaging over seeds, and the real-data tables report mean $\pm$ sample standard deviation across three independent seeds (the seed fixes the bank and the training randomness). The captions identify the uncertainty convention and the source of variability.

\item {\bf Experiments compute resources}
    \item[] Question: For each experiment, does the paper provide sufficient information on the computer resources (type of compute workers, memory, time of execution) needed to reproduce the experiments?
    \item[] Answer: \answerYes{}.
    \item[] Justification: All reported experiments were run on a MacBook Pro with an Apple M4 Pro chip and 24GB memory. Appendix~\ref{app:compute} reports measured training-step, sampling and post-processing times, peak memory, baseline costs and a scaling sweep; the released code contains the scripts and configurations used.

\item {\bf Code of ethics}
    \item[] Question: Does the research conducted in the paper conform, in every respect, with the NeurIPS Code of Ethics \url{https://neurips.cc/public/EthicsGuidelines}?
    \item[] Answer: \answerYes{}.
    \item[] Justification: The work is methodological and uses public benchmark datasets; it does not involve human-subject experiments, private data collection, or deployment. To the best of the authors' knowledge, the research conforms to the NeurIPS Code of Ethics.

\item {\bf Broader impacts}
    \item[] Question: Does the paper discuss both potential positive societal impacts and negative societal impacts of the work performed?
    \item[] Answer: \answerNA{}.
    \item[] Justification: This paper is primarily methodological and theoretical, and does not introduce a deployed system, user-facing model, or new dataset involving sensitive personal data. We do not expect direct societal impacts in the short term; any broader impacts would depend on downstream applications of hypergraph generation to specific domains.
    
\item {\bf Safeguards}
    \item[] Question: Does the paper describe safeguards that have been put in place for responsible release of data or models that have a high risk for misuse (e.g., pre-trained language models, image generators, or scraped datasets)?
    \item[] Answer: \answerNA{}.
    \item[] Justification: The paper does not release a high-risk pretrained model, scraped dataset, or generative system aimed at producing sensitive media or personal data. The datasets used are public benchmarks, and the released assets are intended to reproduce research experiments.

\item {\bf Licenses for existing assets}
    \item[] Question: Are the creators or original owners of assets (e.g., code, data, models), used in the paper, properly credited and are the license and terms of use explicitly mentioned and properly respected?
    \item[] Answer: \answerYes{}.
    \item[] Justification: Existing assets are credited through the corresponding papers and benchmark sources. The real-data benchmarks are drawn from public hypergraph benchmark collections, primarily DHG-Bench/AllSet and related cited sources, and the baseline methods are credited through their original publications. The supplementary code and documentation record the provenance of the datasets and baselines used in the experiments; no proprietary or restricted assets are introduced by this work.

\item {\bf New assets}
    \item[] Question: Are new assets introduced in the paper well documented and is the documentation provided alongside the assets?
    \item[] Answer: \answerYes{}.
    \item[] Justification: The new asset is the \HEDGE\ implementation and associated experimental scripts/configurations. These will be distributed with documentation describing installation, dataset preparation, training, sampling and evaluation.

\item {\bf Crowdsourcing and research with human subjects}
    \item[] Question: For crowdsourcing experiments and research with human subjects, does the paper include the full text of instructions given to participants and screenshots, if applicable, as well as details about compensation (if any)? 
    \item[] Answer: \answerNA{}.
    \item[] Justification: The paper does not involve crowdsourcing, participant studies, or newly collected human-subject data.

\item {\bf Institutional review board (IRB) approvals or equivalent for research with human subjects}
    \item[] Question: Does the paper describe potential risks incurred by study participants, whether such risks were disclosed to the subjects, and whether Institutional Review Board (IRB) approvals (or an equivalent approval/review based on the requirements of your country or institution) were obtained?
    \item[] Answer: \answerNA{}.
    \item[] Justification: The paper does not involve crowdsourcing or human-subject research, so IRB approval or equivalent review is not applicable.

\item {\bf Declaration of LLM usage}
    \item[] Question: Does the paper describe the usage of LLMs if it is an important, original, or non-standard component of the core methods in this research? Note that if the LLM is used only for writing, editing, or formatting purposes and does \emph{not} impact the core methodology, scientific rigor, or originality of the research, declaration is not required.
    \item[] Answer: \answerNA{}.
    \item[] Justification: LLMs are not an important, original, or non-standard component of the core \HEDGE\ methodology. 

\end{enumerate}

\end{document}